\documentclass{article}
\PassOptionsToPackage{table}{xcolor}
\usepackage{colm2024_conference}

\usepackage{amsmath,amsfonts,bm}

\def\eqref#1{equation~\ref{#1}}

\def\1{\bm{1}}

\DeclareMathAlphabet{\mathsfit}{\encodingdefault}{\sfdefault}{m}{sl}
\SetMathAlphabet{\mathsfit}{bold}{\encodingdefault}{\sfdefault}{bx}{n}

\usepackage[utf8]{inputenc} 
\usepackage[T1]{fontenc}    
\usepackage{url}            
\usepackage{booktabs}       
\usepackage{amsfonts}       
\usepackage{nicefrac}       
\usepackage{microtype}      
\usepackage{amsmath}
\usepackage{amssymb}
\usepackage{amsthm}
\newtheorem{lemma}{Lemma}[section]
\newtheorem{theorem}{Theorem}[section]
\newtheorem{proposition}[theorem]{Proposition}
\newtheorem{corollary}[theorem]{Corollary}
\newtheorem{assumption}{Assumption}[section]
\usepackage{tabularx}

\usepackage{subcaption}
\usepackage{arydshln}

\definecolor{lightblue}{RGB}{229, 235,253}
\definecolor{lightgray}{RGB}{235,235,235}
\usepackage{float}
\usepackage{wrapfig}
\usepackage{algorithm}
\usepackage{algpseudocode}
\usepackage{algorithmicx}   
\usepackage{mathtools}      
\usepackage{bm}             
\usepackage{dsfont}         
\usepackage{titletoc}

\makeatletter
\newcommand{\enableappendixcontents}{%
  \let\addcontentsline\colm@orig@addcontentsline
}
\makeatother

\title{STARE: Surprisal-Guided Token-Level Advantage Reweighting for Policy Entropy Stability}

{

\setcounter{footnote}{1}  
\author{\textbf{Haipeng Luo \quad
Qingfeng Sun \quad
Songli Wu \quad
Can Xu\thanks{\quad Corresponding authors. This work was done during Luo's internship at Tencent and was supported by the CIE-Tencent Ph.D. Student Research Incentive Program (Tencent Hunyuan Special Fund).}  \quad
Wenfeng Deng} \\
\textbf{Han Hu \quad
Yansong Tang}$^{\dagger}$ \\
Shenzhen International Graduate School, Tsinghua University\\
Tencent Hunyuan \\
\texttt{\{luohp24@mails., tang.yansong@sz.\}tsinghua.edu.cn}\\
\texttt{\{victorqsun,leocaxu\}@tencent.com}
}
}

\begin{document}

\maketitle

\begingroup

\begin{abstract}
Reinforcement Learning with Verifiable Rewards algorithms like GRPO have emerged as the dominant post-training paradigm for complex reasoning in LLMs, yet commonly suffer from policy entropy collapse during training. We conduct a first-order gradient analysis of token-level entropy dynamics under GRPO and identify a token-level credit assignment mismatch: the per-token entropy variation decomposes into the product of the trajectory-level advantage and an entropy sensitivity function over the next-token distribution, yielding an advantage-surprisal four-quadrant structure and a near-criticality property. Motivated by it, we propose STARE (Surprisal-guided Token-level Advantage Reweighting for policy Entropy stability), which identifies entropy-critical token subsets via batch-internal surprisal quantiles, selectively reweights their effective advantages, and incorporates a target-entropy closed-loop gate for stable entropy regulation. Across model scales from 1.5B to 32B and three task families (Short CoT, Long CoT, and Multi-Turn Tool Use), STARE sustains stable RL training over thousands of steps while maintaining policy entropy within the target band. On AIME24 and AIME25, STARE outperforms DAPO and other competitive baselines by 4\%–8\% in average accuracy, with reflection tokens and response length growing in tandem, indicating sustained exploration–exploitation balance that further unlocks RL training potential. Code is available at \url{https://github.com/hp-luo/STARE}

\end{abstract}

\begin{figure}[H]
    \centering
    \begin{subfigure}[t]{0.32\textwidth}\centering
        \includegraphics[width=\linewidth]{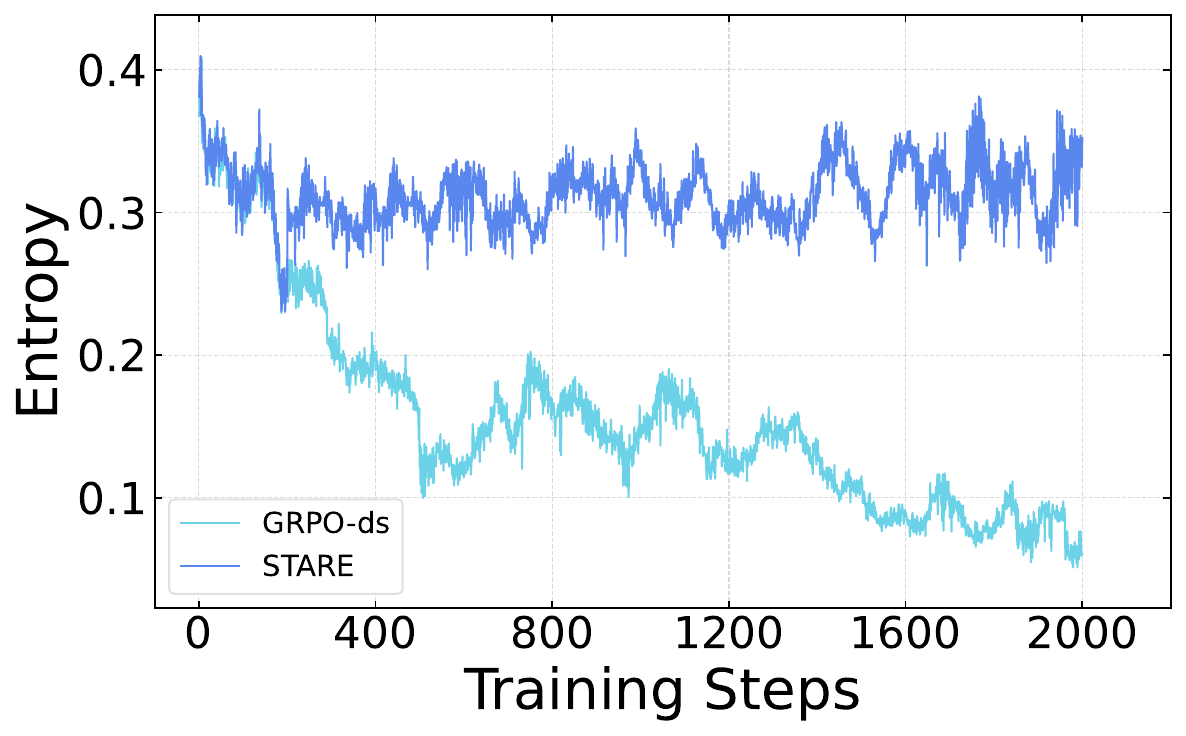}\caption{Training Entropy}\label{fig:7b_agent_sub1_first}
    \end{subfigure}\hfill
    \begin{subfigure}[t]{0.32\textwidth}\centering
        \includegraphics[width=\linewidth]{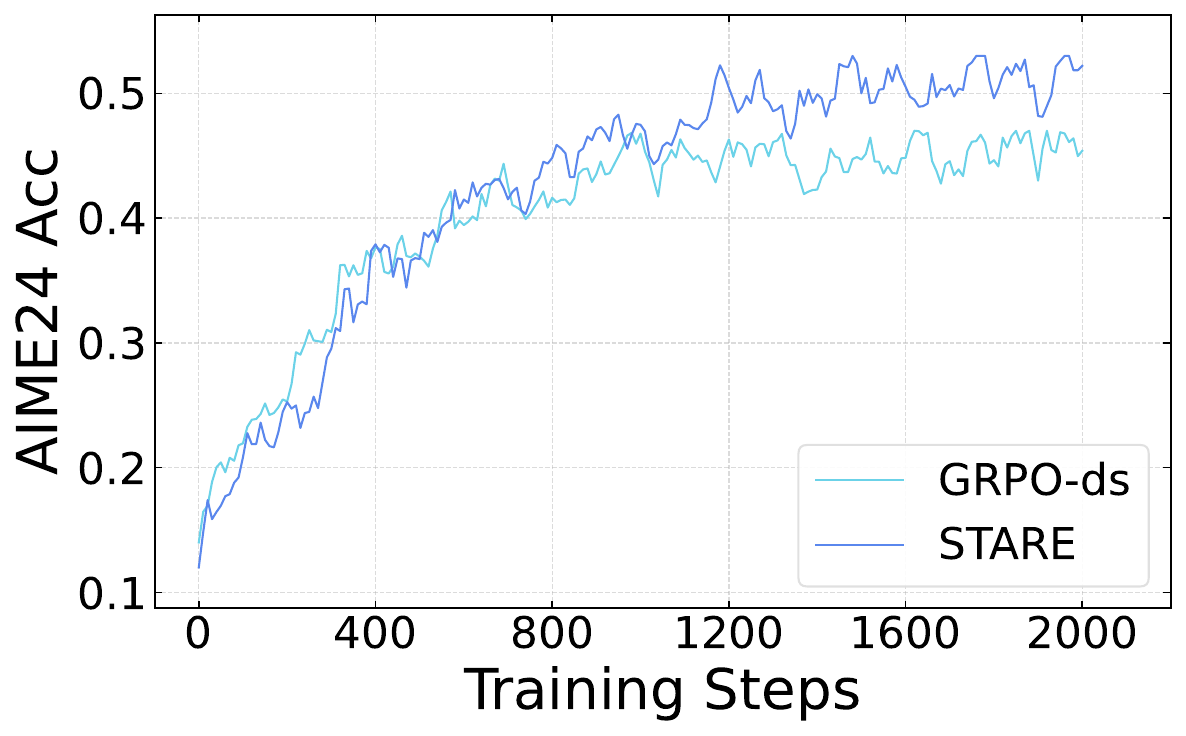}\caption{AIME24 Acc}\label{fig:7b_agent_sub2_first}
    \end{subfigure}\hfill
    \begin{subfigure}[t]{0.32\textwidth}\centering
        \includegraphics[width=\linewidth]{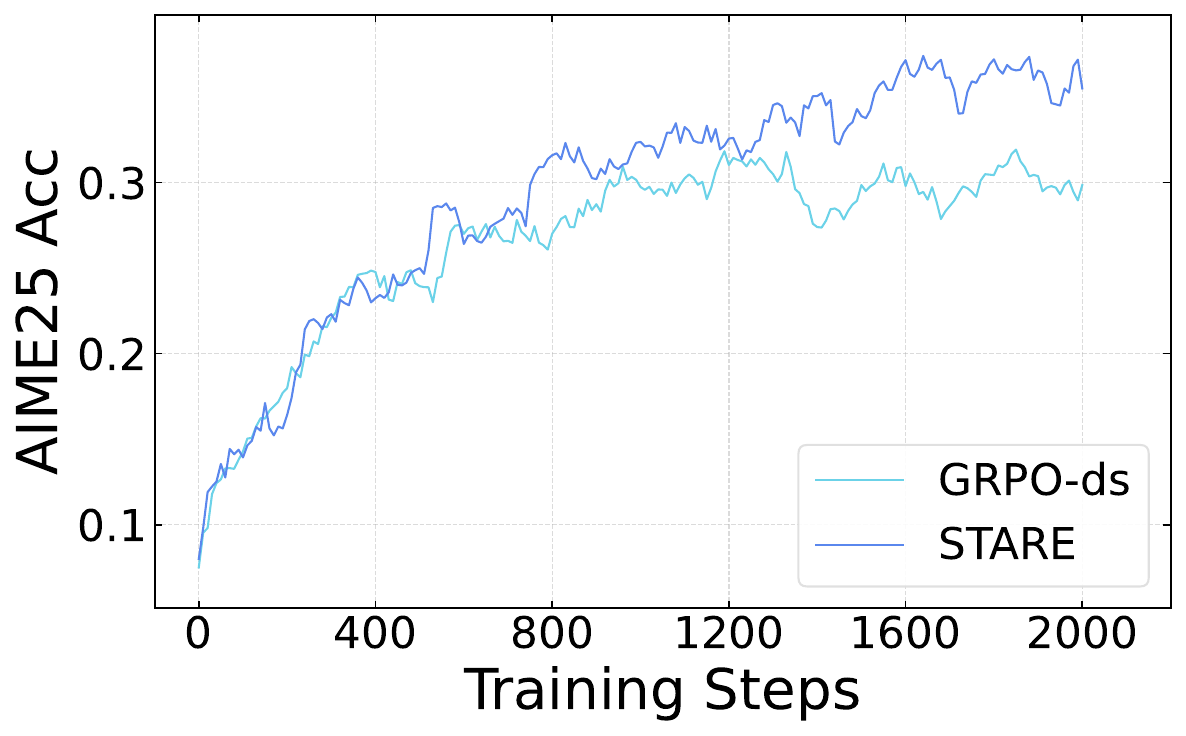}\caption{AIME25 Acc}\label{fig:7b_agent_sub3_first}
    \end{subfigure}
    \caption{Training dynamics of STARE vs. GRPO-ds on Qwen2.5-7B-Base (cold-start SFT from Retool 2K) in the tool-use agent scenario: policy entropy, AIME24 accuracy, and AIME25 accuracy.}\label{fig:qwen7b_agent_first}
\end{figure}

\section{Introduction}

Reinforcement Learning with Verifiable Rewards (RLVR) has emerged as the dominant post-training paradigm for eliciting complex reasoning in LLMs, as exemplified by DeepSeek-R1, Qwen3, and Kimi K1.5\citep{deepseek_r1,openai2024o1,kimi1.5,qwen3technicalreport,claude_example}. Among RLVR algorithms, Group Relative Policy Optimization (GRPO) dispenses with the value network and instead employs group-normalized rewards as the baseline for advantage estimation; and it has been widely adopted in mathematical reasoning, code generation, and is effective at inducing emergent behaviors like long-cot reasoning and self-reflection\citep{shao2024deepseekmathpushinglimitsmathematical,PPO}.

As RL training extends over more optimization steps, however, GRPO-style algorithms commonly suffer from policy entropy collapse: entropy decays rapidly, output diversity vanishes, the policy converges prematurely, and within-group rollouts homogenize, degrading relative advantage estimation and ultimately capping trainable steps, a critical bottleneck for long-horizon post-training\citep{Leap_Does_rl,semantic_entropy}. Existing mitigations fall into three directions. (i) Adjusting the clipping thresholds for the importance-sampling ratio (i.e., DAPO's clip-higher) protects low-probability exploratory tokens, but exerts an asymmetric and uncontrollable effect on entropy and is largely inactive in the on-policy regime where the ratios stay near one\citep{DAPO,chen2026flexible_DID,BAPO,DGPO}. (ii) Asymmetric trajectory-level weighting between positive and negative rollouts: Upweighting rare correct rollouts or biasing updates toward negative samples, provides coarse-grained control\citep{W_Reinforce_Negative_Zhu2025TheSE,A3PO,80_20_Rule,Clip_beta_entropy_dynamic,Lopti,reward_unlikely_grpo,EAPO}. (iii) Entropy-aware advantage reshaping or entropy regularization couples token-level entropy into the advantage, but tends to overamplify high-entropy tokens, induce oscillations, and remain hyperparameter-sensitive\citep{EntropyAdv,Clip_cov,EPO_entropy,skywork-or1,EntropyPIC,Lp_Reg}. These approaches slow entropy decay to varying degrees, yet operate at trajectory or sample granularity, lacking a principled account of the collapse mechanism. Two questions remain open: which tokens drive entropy decay under GRPO, and how strong an intervention suffices to reverse it.

We conduct a first-order gradient analysis of token-level entropy dynamics under GRPO and identify a token-level credit assignment mismatch: although the trajectory-level advantage $\hat{A}_i$ is shared across all tokens within a rollout, the per-token entropy contribution decomposes into the product of $\hat{A}_i$ and a local entropy sensitivity function $\Phi$ determined by the next-token distribution (Section~\ref{sec:theoretical-analysis}). This decomposition yields an advantage--surprisal four-quadrant view: within positive-advantage trajectories, low-surprisal tokens dominate the sampling frequency and drive most entropy-decreasing updates, whereas the rare high-surprisal tokens that could raise entropy are diluted (a mirror-image asymmetry holds for negative-advantage trajectories). We establish a near-criticality property: a mild token-level weight perturbation suffices to flip the sign of entropy evolution and is robust to the specific weight value and beyond the critical threshold, the weight modulates the magnitude rather than the direction of the entropy shift. 

Motivated by these, we propose STARE (Surprisal-guided Token-level Advantage Reweighting for policy Entropy stability), a minimally invasive token-level credit-rebalancing mechanism that operates inside the clipped surrogate of GRPO. STARE selects an entropy-critical token subset via batch-internal surprisal quantiles, thereby selectively amplifying the effective advantage of positive-advantage high-surprisal tokens and attenuating that of negative-advantage high-surprisal tokens. A target-entropy closed-loop gate further governs the intervention: the reweighting is activated to restore exploration when the batch-averaged entropy $\bar{H}_k$ drops below a target level $H_{\text{tgt}}$, and reverts to GRPO once entropy recovers, yielding closed-loop, stable, and low-intrusion entropy regulation.

We validate STARE across multiple model scales and task regimes. On 7B models, STARE stably sustains over 5k RL training steps; on 14B and 32B models, it sustains over 1.5k steps; throughout training, the policy entropy is held within the target band. Across three task families spanning Short CoT, Long CoT, and multi-turn tool-use agents, STARE consistently outperforms DAPO on AIME24 and AIME25 by $4\%$-$8\%$ in average accuracy, with reflection-related tokens and response length growing in tandem, indicating an improved exploration--exploitation balance.

The main contributions of this work are as follows: (i) From a first-order entropy-dynamics analysis, we expose the token-level credit assignment mismatch in GRPO and establish a near-criticality property: a mild weight perturbation suffices to flip the direction of entropy evolution while remaining robust to the weight value. (ii) We propose a surprisal-based advantage reweighting mechanism coupled with a target-entropy closed-loop constraint, achieving stable policy-entropy regulation through a minimal modification to the GRPO objective and sustaining RL training over thousands of steps. (iii) We validate STARE across model scales from 1.5B to 32B and across Short CoT, Long CoT, or multi-turn tool-use regimes, where it maintains stable policy entropy and substantially outperforms DAPO and other baselines by a consistent margin. We defer the discussion of related work to Appendix~\ref{app:related_work}.

\section{Preliminaries}\label{sec:preliminaries}

\textbf{GRPO.}
Given prompt $x$, the old policy $\pi_{\theta_{\mathrm{old}}}$ samples $G$ responses with rewards $\{r_i\}_{i=1}^G$; the group-normalized advantage is $\hat A_i = \bigl(r_i - \operatorname{mean}(\{r_j\})\bigr) / \operatorname{std}(\{r_j\})$.The clipped surrogate objective is:
\[
\mathcal J_{\mathrm{GRPO}}(\theta)
=
\frac{1}{N}
\sum_{i=1}^{B}\sum_{t=1}^{T_i}
\min\!\Big(
\rho_{i,t}(\theta)\,\hat A_i,\;
\operatorname{clip}\!\big(\rho_{i,t}(\theta),\,1{-}\epsilon,\,1{+}\epsilon\big)\,\hat A_i
\Big),
\]
where $\rho_{i,t}(\theta)\triangleq\pi_\theta(o_{i,t}\mid x_i,o_{i,<t})/\pi_{\theta_{\mathrm{old}}}(o_{i,t}\mid x_i,o_{i,<t})$ is the per-token importance ratio, $B$ is the number of responses, $N=\sum_{i=1}^{B}T_i$ is the total token count, and $\beta=0$ throughout (no KL penalty).
Let $c=(x,o_{<t})$ denote context and $\mathcal V$ the vocabulary.
The next-token distribution is parameterized as $\pi_v\triangleq\pi_\theta(v\mid c)=\exp(z_v)/\sum_{v'}\exp(z_{v'})$, with softmax Jacobian $\partial\pi_{v'}/\partial z_v=\pi_{v'}(\delta_{v'v}-\pi_v)$. \textbf{Token Surprisal, Entropy, and Logit updates.}
The \emph{token surprisal} is $\mathfrak{s}_v\triangleq-\ln\pi_v$ and the position-level \emph{policy entropy} is $H\triangleq-\sum_{v}\pi_v\ln\pi_v=\mathbb{E}_{\pi}[\mathfrak{s}]$\cite{shannon1948mathematical,surprisal_oh2024frequencyexplainsinversecorrelation,surprisal_zeng2026pruningunsurprisingefficientllm,oh2023transformer_surprisal,surprisal_smith2013effect}; the batch mean $\bar H\triangleq N^{-1}\sum_{i,t}H_{i,t}$ typically decreases during RL fine-tuning (\emph{entropy collapse}).
In the unclipped regime, the GRPO gradient at a position where token $a$ was sampled yields the logit update $\Delta z_v=\eta\,\hat A\,(\delta_{va}-\pi_v)$ for all $v\in\mathcal V$, where $\eta>0$ is an infinitesimal step size (Appendix~\ref{app:logit-update}).
\begin{lemma}[Entropy gradient w.r.t.\ logits: surprisal-deviation form]\label{lem:entropy-gradient-logits}
For any $v\in\mathcal V$, $\;\frac{\partial H}{\partial z_v}=\pi_v\!\big(\mathfrak{s}_v-H\big)$. Appendix~\ref{app:entropy-gradient-logits} provides the derivation.
\end{lemma}

Raising $z_v$ increases entropy when token $v$ is rarer than average ($\mathfrak{s}_v>H$) and decreases it otherwise. It relates entropy to individual logits and underlies the token-level entropy dynamics derived below.

\section{Theoretical Analysis}\label{sec:theoretical-analysis}

We develop a theoretical framework for analyzing entropy evolution during GRPO training, proceeding from token-level gradient analysis (Section~\ref{sec:first-order-gradient}) through an advantage--surprisal decomposition (Section~\ref{sec:four-quadrant}) and a batch-level near-criticality result (Section~\ref{sec:batch-level}) to cross-step feedback dynamics (Section~\ref{sec:cross-step}). Complete proofs are in Appendices~\ref{app:basic-derivations}--\ref{app:cross-step}.

\subsection{First-Order Gradient Analysis of Token-Level Policy Entropy}\label{sec:first-order-gradient}

Consider the next-token distribution $\pi(\cdot\mid c)$ with entropy $H$.
Let $a$ denote the sampled token, with probability $p = \pi(a\mid c)$ and surprisal $\mathfrak{s}_a = -\ln p$.
Define
\(S_2 \triangleq \sum_{v \in \mathcal{V}} \pi_v^2 \!\big(\ln \pi_v + H\big)\)
and
\(\Phi(p) \triangleq p\!\big(\ln p + H\big) - S_2\).
We call $\Phi$ the \textbf{entropy sensitivity function}: it measures the signed excess of the sampled token's probability-weighted surprisal deviation over the distributional baseline $S_2$.

\begin{theorem}[Token-level entropy variation]\label{thm:token-entropy-variation}
In the unclipped regime of GRPO, let $\hat{A}$ denote the advantage at this position, and let $\eta$ be the step size along the GRPO policy-gradient direction. Then
\(\left.\frac{dH}{d\eta}\right|_{\eta=0}
=
-\hat{A}\,\Phi(p)\).
\end{theorem}

The proof (Appendix~\ref{app:token-entropy-variation}) follows by taking the inner product of $\partial H/\partial z_v$ (Lemma~\ref{lem:entropy-gradient-logits}) with the GRPO logit update.
The result decomposes the instantaneous entropy effect into the advantage $\hat{A}$ and $\Phi(p)$, governing whether probability redistribution concentrates or disperses the distribution. And their product determines the sign and magnitude of entropy variation at each position.

\subsection{Advantage--Surprisal Four-Quadrant Decomposition}\label{sec:four-quadrant}

Determining $\operatorname{sign}(dH/d\eta)$ reduces to analyzing $\operatorname{sign}(\Phi(p))$.
Two properties, proved in Appendix~\ref{app:theory-proofs}, underpin the analysis: $S_2 > 0$ for any non-uniform distribution (Lemma~\ref{lem:app-s2-positive}), and $H > S_2$ for any non-degenerate distribution (Lemma~\ref{lem:app-h-greater-s2}).
These imply $\Phi(0^+) = -S_2 < 0$ and $\Phi(1) = H - S_2 > 0$.
Since $\Phi'(p) = \ln p + H + 1$, the function is strictly increasing on $(e^{-(H+1)}, 1]$, yielding:

\begin{proposition}[Uniqueness of the Critical Surprisal Threshold]\label{prop:critical-surprisal-threshold} For any non-uniform, non-degenerate $\pi$, there exists a unique $p^*\in(e^{-H},1)$ with $\mathfrak{s}^*\triangleq-\ln p^*\in(0,H)$ such that $\Phi(p^*)=0$ and when $\Phi(p)>0 \;\Longleftrightarrow\; p>p^* \;\Longleftrightarrow\; \mathfrak{s}_a<\mathfrak{s}^*$. The proof is in Appendix~\ref{app:critical-surprisal-threshold}. \end{proposition}

The threshold $\mathfrak{s}^*$ partitions the vocabulary into a \textbf{low-surprisal region} ($\mathfrak{s}_a<\mathfrak{s}^*$) and a \textbf{high-surprisal region} ($\mathfrak{s}_a>\mathfrak{s}^*$).
Substituting into Theorem~\ref{thm:token-entropy-variation} yields the four-quadrant structure.

\begin{corollary}[Four-Quadrant Decomposition]\label{cor:four-quadrant}
The sign of $dH/d\eta$ is determined by $(\operatorname{sign}\hat{A},\;\mathbf{1}[\mathfrak{s}_a<\mathfrak{s}^*])$:
(i)~reinforcing low-surprisal tokens ($\hat{A}>0,\mathfrak{s}_a<\mathfrak{s}^*$) reduces entropy;
(ii)~reinforcing high-surprisal tokens ($\hat{A}>0,\mathfrak{s}_a>\mathfrak{s}^*$) increases it;
(iii)~suppressing low-surprisal tokens ($\hat{A}<0,\mathfrak{s}_a<\mathfrak{s}^*$) increases it;
(iv)~suppressing high-surprisal tokens ($\hat{A}<0,\mathfrak{s}_a>\mathfrak{s}^*$) reduces it.
Each GRPO step thus propagates four token-level entropy signals with opposing signs. The proof is in Appendix~\ref{app:four-quadrant}.
\end{corollary}

\textbf{Asymmetric entropy contributions.}
Since rollouts are sampled from $\pi_\theta$, low-surprisal tokens are drawn more frequently than high-surprisal ones at each decoding position.
Within positive-advantage trajectories, entropy-decreasing tokens (low surprisal, $\Phi > 0$) therefore constitute the statistical majority.
Because GRPO assigns a single trajectory-level $\hat{A}_i$ to all tokens, it cannot distinguish these opposing entropy effects:
the reinforced low-surprisal majority systematically drives the distribution toward concentration, while the high-surprisal minority that could preserve diversity contributes limited entropy-increasing effects.
An analogous asymmetry governs the negative-advantage subset, revealing a fundamental gradient-level mechanism underlying entropy collapse in GRPO.

\begin{figure}[t]
    \centering
    \includegraphics[width=1.0\textwidth]{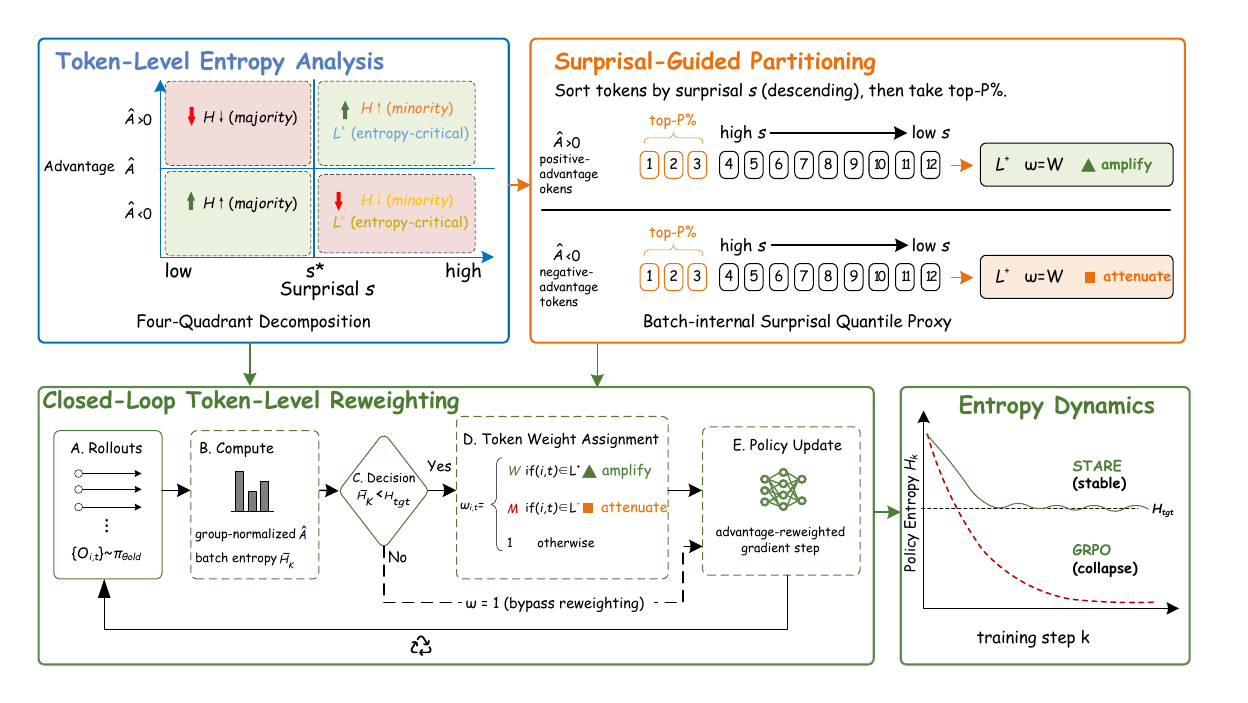}
    \caption{Overview of STARE. Guided by a four-quadrant decomposition of token-level entropy dynamics (top-left) and a batch-internal surprisal-quantile proxy that identifies entropy-critical tokens (top-right), STARE applies target-entropy-gated advantage reweighting in GRPO (bottom-left), stabilizing policy entropy where vanilla GRPO collapses (bottom-right).}
    \label{fig:STARE_figures}
\end{figure}

\subsection{Batch-Level Entropy Decomposition and Near-Criticality}\label{sec:batch-level}

The preceding analysis establishes that the entropy-contribution asymmetry persists within both advantage subsets.
A natural quantitative question arises: what reweighting of the entropy-increasing minority suffices to reverse the batch-level net entropy gradient?

\begin{theorem}[Entropy Neutrality Identity]\label{thm:entropy-neutrality}
For any conditional distribution $\pi$,
$\;\mathbb{E}_{a \sim \pi}\!\big[\Phi(a)\big]
=
\sum_{v \in \mathcal{V}} \pi_v\,\Phi(\pi_v)
=
0$. The proof is in Appendix~\ref{app:entropy-neutrality}.
\end{theorem}

\textbf{Token-level advantage reweighting.}
Define $\mathcal{L}^{+}=\{(i,t):\hat{A}_i>0,\;\mathfrak{s}_{i,t}>\mathfrak{s}^*_{i,t}\}$.
Scaling the effective advantage of each token in $\mathcal{L}^+$ by a multiplicative factor $W\ge 1$, while retaining unit weight at all remaining positions, yields:

\begin{proposition}[Entropy Gradient under Token-Level Reweighting]\label{prop:entropy-gradient-reweighting}
\(\left.\frac{d\bar{H}}{d\eta}\right|_{W}
=
-\frac{1}{N}\big[\Lambda - (W{-}1)\Gamma\big]\),
where
\(\Lambda\triangleq\sum_{i,t}\hat{A}_i\,\Phi_{i,t}\),
\(\Gamma\triangleq\sum_{(i,t)\in\mathcal{L}^{+}}\hat{A}_i\,|\Phi_{i,t}|>0\),
and the critical weight is \(W^*=1+\Lambda/\Gamma\).
\end{proposition}

The proof is in Appendix~\ref{app:entropy-gradient-reweighting}.
Two structural properties jointly ensure that $W^*$ remains near unity (\textbf{near-critical regime}).
First, the entropy neutrality identity guarantees that $\Lambda$ is a \emph{weak residual} arising solely from the statistical dependence between trajectory-level advantages and token-level entropy sensitivities; under the assumptions formalized in Appendix~\ref{app:near-criticality-assumptions},
$|\Lambda|/\Sigma_{\mathrm{abs}} = O(T^{-1})$
where $\Sigma_{\mathrm{abs}} = \sum_{i,t}|\hat{A}_i|\,|\Phi_{i,t}|$.
Second, \emph{high-surprisal tokens carry amplified entropy sensitivity} (Appendix~\ref{app:auxiliary-bounds}):
$\mathfrak{s}\ge H \Rightarrow |\Phi(p)|\ge S_2$,
whereas
$\mathfrak{s}<\mathfrak{s}^* \Rightarrow |\Phi(p)|\le H-S_2$.
Although $\mathcal{L}^+$ is a sampling minority, the amplified per-token $|\Phi|$ values ensure that $\Gamma$ remains appreciable.

\begin{corollary}[Near-Criticality]\label{thm:near-criticality}
When the sequence length $T$ and the batch size are both sufficiently large,
$\;W^* - 1 = \Lambda/\Gamma = O(T^{-1})$ (Appendix~\ref{app:near-criticality}). So beyond the critical threshold, the specific value of $W$ principally controls the magnitude rather than the sign of the per-step entropy shift.
\end{corollary}

\subsection{Cross-Step Entropy Dynamics}\label{sec:cross-step}

Let $\Delta\bar{H}_k = -N^{-1}[\Lambda_k - (W{-}1)\,\Gamma_k]$ denote the batch-level entropy variation at step $k$.
Under $W=1$, the sampling asymmetry implies $\Lambda_k>0$ in expectation, so $\Delta\bar{H}_k<0$.
The resulting entropy reduction further concentrates $\pi_\theta$, lowering the sampling frequency of high-surprisal tokens in subsequent batches, shrinking $\Gamma_{k+1}$, and raising $W^*_{k+1}$, forming a \textbf{self-reinforcing loop of entropy collapse}.
Conversely, when $W>W^*_k$, entropy increases disperse the distribution, enlarge $\Gamma_{k+1}$, and lower $W^*_{k+1}$, forming a symmetric loop of \textbf{entropy recovery} (formalized in Appendix~\ref{app:cross-step}).
The near-criticality result (corollary~\ref{thm:near-criticality}) establishes that a modest token-level weight adjustment suffices to alter the macroscopic entropy trajectory.
This theoretical insight directly motivates the algorithmic design proposed in the next section: by selectively modulating the effective advantage weights of a targeted subset of tokens within the GRPO policy gradient, one can restore a sustainable dynamic equilibrium between the entropy-increasing and entropy-decreasing gradient forces.

\section{Method}\label{sec:method}

The theoretical analysis in Section~\ref{sec:theoretical-analysis} reveals that GRPO's shared trajectory-level advantages induce a token-level credit assignment mismatch: high-frequency low-surprisal tokens dominate gradient aggregation while sparse high-surprisal tokens with critical entropy effects are under-represented. Motivated by it, we propose \textbf{S}urprisal-Guided \textbf{T}oken-Level \textbf{A}dvantage \textbf{R}eweighting For Policy \textbf{E}ntropy Stability \textbf{(STARE)}: a surprisal-based credit rebalancing mechanism that operates within the clipped surrogate of GRPO, assigns differentiated weights to entropy-critical tokens, and incorporates target-entropy closed-loop feedback for stable entropy regulation on training,  as shown in Figure~\ref{fig:STARE_figures}.

\subsection{Entropy-Critical Token Partitioning via High-Surprisal Quantiles}\label{sec:entropy-critical-partition}

Let $\mathcal{T}^{+}=\{(i,t):\hat{A}_i>0\}$ and $\mathcal{T}^{-}=\{(i,t):\hat{A}_i<0\}$ denote the positive- and negative-advantage token sets.  Computing the exact critical threshold $\mathfrak{s}^*_{i,t}$ (Proposition~\ref{prop:critical-surprisal-threshold}) at every position requires the full conditional distribution, incurring prohibitive overhead. Instead, STARE employs a simple, stable, and theoretically motivated
\textbf{batch-internal surprisal-quantile proxy}. Concretely, within $\mathcal{T}^+$
and $\mathcal{T}^-$ separately, tokens are ranked in descending order of surprisal 
$\mathfrak{s}_{i,t}=-\ln\pi_\theta(o_{i,t}\mid x_i,o_{i,<t})$, and the
top $P\%$ tokens are selected to form the entropy-critical sets below:
\[
\mathcal{L}^{\pm}
=
\bigl\{(i,t)\in\mathcal{T}^{\pm}:\mathfrak{s}_{i,t}\ge Q_{P}\!\big(\{\mathfrak{s}_{j,s}\}_{(j,s)\in\mathcal{T}^{\pm}}\big)\bigr\}.
\]
By Corollary~\ref{cor:four-quadrant}, $\mathcal{L}^+$ approximately denotes the entropy-increasing tokens among positive-advantage responses, and $\mathcal{L}^-$ indicates the entropy-decreasing tokens in negative-advantage responses. The fixed proportion $P$ directly controls the intervention scale, obviating per-position threshold computation.

\subsection{Advantage-Conditioned Token-Level Credit Rebalancing}\label{sec:token-credit-rebalancing}

We augment the GRPO objective with positive token-level weights $\omega_{i,t}>0$:
\[
\mathcal{J}_{\mathrm{STARE}}(\theta)
=
\frac{1}{N}
\sum_{i,t}
\textcolor{red}{\omega_{i,t}}\;
\min\!\Big(
\rho_{i,t}(\theta)\,\hat{A}_i,\;
\operatorname{clip}\!\big(\rho_{i,t}(\theta),\,1{-}\epsilon,\,1{+}\epsilon\big)\,\hat{A}_i
\Big).
\]
Setting $\omega_{i,t}\equiv 1$ recovers STARE to standard GRPO. Because $\omega_{i,t}>0$, STARE preserves all token-level gradient directions: tokens with positive advantage remain reinforced, while those with negative advantage remain suppressed. STARE therefore acts as an advantage-conditioned credit-rebalancing mechanism, selectively rescaling relative magnitudes along the surprisal dimension.

{\bfseries\boldmath Variant~I: One-Sided Entropy Amplification ($\hat{A}_i>0$, High-Surprisal tokens, denoted as O1).}
\[
\omega_{i,t}^{(\mathrm{V1})}
=
\begin{cases}
W, & (i,t)\in\mathcal{L}^{+},\\
1, & \text{otherwise},
\end{cases}
\qquad W>1.
\]
Tokens in $\mathcal{L}^{+}$ simultaneously carry positive advantage and entropy-increasing effect; amplifying their weights directly strengthens the minority that GRPO systematically underweights. The resulting batch-level net entropy shift is $\Lambda_{\mathrm{V1}}=\Lambda-(W{-}1)\,\Gamma^{+}$, where $\Gamma^{+}=\sum_{(i,t)\in\mathcal{L}^{+}}|\hat{A}_i\,\Phi_{i,t}|>0$. And under the near-criticality condition (Corollary~\ref{thm:near-criticality}), a moderate $W>1$ usually suffices to reverse the sign of the batch net entropy shift. We provide the Algorithmic~\ref{alg:stare-o1} pseudocode in the Appendix~\ref{app:Algorithm_stare}.

\textbf{Variant~II: Dual-Sided Entropy Regulation.} Extending Variant I, this variant additionally attenuates the weights of tokens in $\mathcal{L}^{-}$ ($\hat{A}_i<0$, High-Surprisal tokens, denoted as C2):
\[
\omega_{i,t}^{(\mathrm{V2})}
=
\begin{cases}
W, & (i,t)\in\mathcal{L}^{+},\\
M, & (i,t)\in\mathcal{L}^{-},\\
1, & \text{otherwise},
\end{cases}
\qquad W>1,\;0<M<1.
\]
Tokens in $\mathcal{L}^-$ reside in the high-surprisal tail of negative-advantage responses; and large negative-advantage updates on these tokens redistribute mass from the tail toward the peak, exacerbating concentration. Attenuating their weights alleviates this entropy-decreasing pressure. Letting $\Gamma^{-}=\sum_{(i,t)\in\mathcal{L}^{-}}|\hat{A}_i\,\Phi_{i,t}|$, the batch net entropy shift becomes $\Lambda_{\mathrm{V2}}=\Lambda-(W{-}1)\,\Gamma^{+}-(1{-}M)\,\Gamma^{-}$. 

Two-sided regulation simultaneously amplifies the entropy-increasing signal and attenuates the entropy-decreasing signal to adjust policy entropy. Definitions and analysis for all four single-polarity (O1, O2, O3, O4) and four combined operations (C1, C2, C3, C4) are deferred to Appendices~\ref{app:single-polarity} and~\ref{app:combined-operations}.

\subsection{Closed-Loop Regulation via Target-Entropy Gating}\label{sec:target-entropy-gating}
A purely open-loop reweighting strategy risks overshooting from entropy collapse into uncontrolled divergence. To achieve stable regulation, we introduce a batch-level target entropy $H_{\mathrm{tgt}}$ and employ the current batch mean entropy $\bar{H}_k$ as a closed-loop feedback signal via a binary gate $g_k=\mathbf{1}[\bar{H}_k<H_{\mathrm{tgt}}]$ and express the weights in unified form below. When $\bar{H}_k<H_{\mathrm{tgt}}$, the gate activates  and 
\[
\omega_{i,t}
=
\begin{cases}
1+g_k\,(W{-}1), & (i,t)\in\mathcal{L}^{+},\\
1-g_k\,(1{-}M), & (i,t)\in\mathcal{L}^{-}\;\text{(two-sided only)},\\
1, & \text{otherwise}.
\end{cases}
\]
 STARE  strengthens entropy-increasing signals; when $\bar{H}_k\ge H_{\mathrm{tgt}}$, all weights revert to unity, automatically recovering standard GRPO. This drives entropy toward $H_{\mathrm{tgt}}$ via bounded oscillation. Finer-grained sample-level and token-level closed-loop variants are presented in Appendix~\ref{app:closed-loop-granularities}.

\subsection{Static and Adaptive Weighting Schedules}\label{sec:weighting-schedules}

\textbf{Fixed weights.} Near-criticality (Corollary~\ref{thm:near-criticality}) implies that the required reweight perturbation is typically modest: beyond the critical point, the specific value principally controls the magnitude rather than the direction of the per-step entropy shift, reducing sensitivity to hyperparameter choices. Fixed $W$ and $M$ therefore suffice in most settings.

\textbf{Adaptive weights.} As training progresses, intervention strength may vary across phases with distinct distributional. STARE also supports adaptive weight updates driven by the target-entropy signal:
\begin{align*}
W_{k+1} &= \operatorname{clip}\!\Big(W_k + \alpha\,\operatorname{sgn}\big(H_{\mathrm{tgt}} - \bar{H}_k\big),\;[1,\,W_{\max}]\Big),\\
M_{k+1} &= \operatorname{clip}\!\Big(M_k - \alpha\,\operatorname{sgn}\big(H_{\mathrm{tgt}} - \bar{H}_k\big),\;[M_{\min},\,1]\Big).
\end{align*}
When $\bar{H}_k<H_{\mathrm{tgt}}$, $W$ increases and $M$ decreases, intensifying the intervention; otherwise both relax toward GRPO. The constraints $W\ge 1$ and $M\le 1$ ensure graceful degradation. Setting $\alpha=0$ recovers the fixed-weight. Default values $\alpha{=}0.01$, $W_{\max}{=}1.5$, $M_{\min}{=}0.5$ yield robust performance.

\textbf{Default configuration.} All main experiments adopt Variant~I (O1) with batch-level target-entropy gating and fixed weights. This minimal configuration suffices to stabilize entropy and improve performance. Ablations on two-sided regulation and adaptive weights are provided in Appendix~\ref{app:combined-operations}.


\section{Experiments}
\subsection{Experimental Setup}

\textbf{Models and scenarios.} We systematically evaluate STARE across three scenarios. In the Short CoT scenario, we use Qwen2.5-Math-7B-Base with a maximum decoding length of 4k, Qwen2.5-14B-Instruct with 8k, and Qwen2.5-32B-Base with 8k\citep{qwen25math,qwen2025qwen25technicalreport}. In the Long CoT scenario, we employ DeepSeek-R1-Distill-Qwen-1.5B and Qwen3-8B-Base with 16k to elicit deep reasoning and self-reflection\citep{deepseek_r1,qwen3technicalreport}. In the tool-use scenario, we first perform cold-start SFT on Qwen2.5-7B-Base using Retool 2K data\citep{retool}, then conduct RL training with 8k length.

\textbf{Training.} We use a learning rate of $1\times 10^{-6}$ and a batch size of $64$ samples with $8$ rollouts per sample, and on-policy updates with a single gradient step per batch. Decoding adopts Top-$p\!=\!1.0$ and temperature $T\!=\!1.0$. The STARE hyperparameters are set to $W\!=\!1.1$, $M\!=\!0.9$, $H_{\text{tgt}}\!=\!0.3$, and $P\%\!=\!10\%$, with batch-level gating and fixed weights as the default configuration. The training corpus consists of 100k samples deduplicated and sampled from open-source RL datasets including DeepScaler, Skywork-o1, Polaris, and DAPO\citep{deepscaler2025,skywork-or1,Polaris2025,DAPO}. We construct an 
\begin{wraptable}[43]{r}{0.54\textwidth}
\centering
\caption{Performance comparison of STARE and competitive RL algorithms on six math benchmarks across 1.5B–32B scales and three scenarios (Acc avg@N); $^{\dagger}$ denotes results from prior works.}
\label{tab:main_results}
\scalebox{0.50}{
\begin{tabular}{lccccccc}
\toprule
\textbf{Method} & \textbf{AIME24} & \textbf{AIME25} & \textbf{AMC} & \textbf{MATH} & \textbf{Minerva} & \textbf{Olympiad} & \textbf{Avg.} \\
\midrule
\rowcolor{lightblue}
\multicolumn{8}{c}{\textbf{In the Short CoT Scenario}} \\
\midrule
\multicolumn{8}{c}{\textbf{\textit{RL from the Qwen2.5-Math-7B-Base}}} \\
\hdashline
W-REINF$^\dagger$\citep{W_Reinforce_Negative_Zhu2025TheSE} & 31.2 & 10.1 & 58.1 & 76.2 & 34.2 & 38.9 & 41.6 \\
GRPO$^\dagger$\citep{shao2024deepseekmathpushinglimitsmathematical} & 34.0 & 9.3 & 58.6 & 79.9 & 38.1 & 42.8 & 43.8 \\
EntroReg$^\dagger$\citep{EntropyAdv} & 34.6 & 12.6 & 61.1 & 82.7 & 43.1 & 42.8 & 46.2 \\
DAPO$^\dagger$\citep{DAPO} & 35.7 & 17.1 & 61.0 & 82.6 & 42.8 & 43.9 & 47.2 \\
80/20 Rule$^\dagger$\citep{80_20_Rule} & 35.8 & 15.1 & 63.7 & 84.1 & 43.0 & 45.5 & 47.9 \\
EntroAdv$^\dagger$\citep{EntropyAdv} & 36.8 & 16.3 & 63.8 & 83.5 & 42.7 & 44.5 & 47.9 \\
STEER$^\dagger$\citep{STEER} & 36.9 & 16.2 & 72.2 & 82.4 & 41.7 & 43.3 & 49.1 \\
GRPO-ds & 37.1 & 17.7 & 75.3 & 82.7 & 39.4 & 42.6 & 49.1 \\
KL-Cov$^\dagger$\citep{Clip_cov} & 38.9 & 13.8 & 59.2 & 81.5 & 40.9 & 44.2 & 46.4 \\
EAPO$^\dagger$\citep{EAPO} & 39.8 & 17.2 & 62.1 & 83.7 & 43.9 & 45.1 & 48.6 \\
\rowcolor{lightgray}
STARE-O1 & 44.2 & 23.8 & 83.4 & 86.1 & 44.2 & 44.7 & 54.4 \\
\rowcolor{lightgray}
STARE-C2 & 42.9 & 24.2 & 84.1 & 85.8 & 44.7 & 45.3 & 54.5 \\
\midrule
\multicolumn{8}{c}{\textbf{\textit{RL from the Qwen2.5-14B-Instruct}}} \\
\hdashline
Base-14B$^\dagger$\citep{qwen2025qwen25technicalreport} & 12.1 & 11.7 & - & - & - & - & - \\
GRPO$^\dagger$\citep{shao2024deepseekmathpushinglimitsmathematical} & 22.5 & 17.6 & - & - & - & - & - \\
GRPO-$\mathrm{Clip}_{\nu}$$^\dagger$\citep{Clip_beta_entropy_dynamic} & 23.4 & 21.4 & - &  & - & - & - \\
GRPO-ds & 24.2 & 21.9 & 69.1 & 80.6 & 37.2 & 43.4 & 46.1 \\
\rowcolor{lightgray}
STARE-O1 & 30.8 & 27.1 & 77.5 & 85.4 & 40.5 & 50.9 & 52.0 \\
\rowcolor{lightgray}
STARE-C2 & 31.5 & 28.3 & 76.3 & 86.1 & 40.2 & 51.4 & 52.3 \\
\midrule
\multicolumn{8}{c}{\textbf{\textit{RL from the Qwen2.5-32B-Base}}} \\
\hdashline
GRPO$^\dagger$\citep{shao2024deepseekmathpushinglimitsmathematical} & 28.5 & 22.5 & - & 86.6 & 44.9 & 60.3 & - \\
80/20 Rule$^\dagger$\citep{80_20_Rule} & 32.5 & 28.5 & - & 89.4 & 45.6 & 57.6 & - \\
GSPO$^\dagger$\citep{gspo} & 33.3 & 22.3 & - & 87.6 & 48.5 & 55.6 & - \\
Lp-Reg$^\dagger$\citep{Lp_Reg} & 38.1 & 27.1 & - & 90.0 & 46.32 & 61.2 & - \\
DAPO$^\dagger$\citep{DAPO} & 38.3 & 29.8 & - & 87.6 & 48.5 & 55.6 & - \\
GRPO-ds & 38.5 & 28.8 & 85.3 & 85.6 & 44.6 & 54.0 & 56.1 \\
\rowcolor{lightgray}
STARE-O1 & 43.3 & 34.1 & 87.3 & 90.4 & 48.8 & 60.1 & 60.7 \\
\rowcolor{lightgray}
STARE-C2 & 42.9 & 35.7 & 88.8 & 90.6 & 49.3 & 60.9 & 61.4 \\
\midrule
\rowcolor{lightblue}
\multicolumn{8}{c}{\textbf{In the Long CoT Scenario}} \\
\midrule
\multicolumn{8}{c}{\textbf{\textit{RL from the DeepSeek-R1-Distill-Qwen-1.5B}}} \\
\hdashline
CE-GPPO$^\dagger$\citep{GPPO} & 29.6 & 23.5 & 73.5 & 76.3 & 26.6 & 43.9 & 45.6 \\
Base-1.5B$^\dagger$\citep{deepseek_r1} & 32.5 & 24.3 & 69.4 & 85.7 & 37.5 & 55.2 & 50.7 \\
EntroReg$^\dagger$\citep{EntropyAdv} & 34.6 & 26.4 & 70.2 & 86.8 & 37.6 & 56.2 & 51.9 \\
CISPO$^\dagger$\citep{CISPO} & 34.8 & 25.8 & 76.9 & 76.8 & 26.5 & 45.8 & 48.4 \\
W-REINF$^\dagger$\citep{W_Reinforce_Negative_Zhu2025TheSE} & 35.0 & 25.0 & 69.8 & 87.5 & 37.8 & 55.9 & 51.8 \\
CE-GPPO$^\dagger$\citep{GPPO} & 35.1 & 27.7 & 82.5 & 76.7 & 27.8 & 45.6 & 49.2 \\
ASPO$^\dagger$\citep{ASPO} & 36.4 & 28.3 & 83.1 & 74.6 & 26.0 & 44.9 & 48.9 \\
GRPO$^\dagger$\citep{shao2024deepseekmathpushinglimitsmathematical} & 38.1 & 27.1 & 71.4 & 88.9 & 39.4 & 58.7 & 53.9 \\
DAPO$^\dagger$\citep{DAPO} & 40.9 & 28.6 & 73.6 & 89.9 & 39.1 & 58.4 & 55.1 \\
DGPO$^\dagger$\citep{DGPO} & 43.3 & 32.8 & 86.0 & 77.9 & 28.2 & 48.0 & 52.7 \\
KL-Cov$^\dagger$\citep{Clip_cov} & 43.9 & 30.1 & 75.0 & 90.0 & 40.5 & 59.9 & 56.5 \\
EntroAdv$^\dagger$\citep{EntropyAdv} & 44.0 & 30.4 & 73.9 & 90.2 & 40.4 & 59.1 & 56.3 \\
EAPO$^\dagger$\citep{EAPO} & 45.1 & 30.1 & 75.5 & 91.1 & 39.7 & 60.8 & 57.0 \\
GRPO-ds & 50.4 & 37.4 & 85.9 & 88.3 & 49.2 & 64.1 & 62.5 \\
JustRL$^\dagger$\citep{he2025justrlscaling15bllm} & 52.6 & 38.8 & 91.0 & 91.7 & 51.5 & 68.0 & 65.6 \\
\rowcolor{lightgray}
STARE-O1 & 53.8 & 41.5 & 89.5 & 90.4 & 52.6 & 67.8 & 65.9 \\
\rowcolor{lightgray}
STARE-C2 & 53.1 & 40.5 & 91.3 & 92.1 & 51.8 & 68.8 & 66.3 \\
\midrule
\multicolumn{8}{c}{\textbf{\textit{RL from the Qwen3-8B-Base}}} \\
\hdashline
GRPO$^\dagger$\citep{shao2024deepseekmathpushinglimitsmathematical} & 31.3 & 24.7 & 75.2 & 88.9 & 55.9 & 61.5 & 56.2 \\
80/20 Rule$^\dagger$\citep{80_20_Rule} & 31.3 & 27.5 & 79.9 & 89.9 & 54.8 & 62.5 & 57.6 \\
STAPO$^\dagger$\citep{STAPO} & 33.4 & 28.7 & 79.9 & 90.4 & 57.2 & 63.0 & 58.8 \\
DAPO$^\dagger$\citep{DAPO} & 34.2 & 26.1 & - & 84.5 & - & - & - \\
Lp-Reg$^\dagger$\citep{Lp_Reg} & 35.9 & 25.8 & - & 87.4 & - & - & - \\
A3PO$^\dagger$\citep{A3PO} & 37.8 & 30.4 & - & 91.3 & - & - & - \\
GRPO-ds & 39.5 & 30.8 & 80.8 & 88.6 & 52.3 & 59.0 & 58.5 \\
\rowcolor{lightgray}
STARE-O1 & 43.9 & 34.7 & 85.3 & 90.6 & 55.6 & 61.8 & 62.0 \\
\rowcolor{lightgray}
STARE-C2 & 44.3 & 32.6 & 86.1 & 91.2 & 56.9 & 62.3 & 62.2 \\
\midrule
\rowcolor{lightblue}
\multicolumn{8}{c}{\textbf{In the Tool-Use Agent Scenario}} \\
\midrule
\multicolumn{8}{c}{\textbf{\textit{RL from the Qwen2.5-7B-Base}}} \\
\hdashline
Base-7B-TIR$^\dagger$\citep{qwen2025qwen25technicalreport} & 1.7 & 0.6 & 10.8 & 18.0 & - & 6.2 & - \\
ToRL$^\dagger$\citep{wang2023torl} & 40.2 & 27.9 & 75.0 & 82.2 & - & 49.9 & - \\
Effective TIR$^\dagger$\citep{effective_tir} & 42.3 & 29.2 & 74.2 & 86.4 & - & - & - \\
ZeroTIR$^\dagger$\citep{zero_tir} & 46.7 & 30.0 & - & 85.2 &  & - & - \\
GRPO-ds & 46.8 & 32.4 & 75.9 & 81.4 & 38.3 & 48.8 & 53.9 \\
SimpleTIR$^\dagger$\citep{xue2025simpletir} & 50.5 & 30.9 & 79.1 & 88.4 & - & 54.8 & - \\
\rowcolor{lightgray}
STARE-O1 & 53.2 & 37.5 & 84.9 & 86.8 & 41.9 & 52.3 & 59.4 \\
\rowcolor{lightgray}
STARE-C2 & 52.8 & 38.1 & 86.9 & 87.2 & 43.7 & 53.6 & 60.4 \\

\bottomrule
\end{tabular}
}
\end{wraptable}
enhanced GRPO baseline, denoted GRPO-ds, which removes the KL penalty and incorporates dynamic sampling with token-level loss\citep{DAPO}. We report two variants: STARE-O1 amplifies only $\mathcal{L}_q^{+}$, whereas STARE-C2 additionally attenuates $\mathcal{L}_q^{-}$ over O1. Ablation studies adopt STARE-O1 as the default configuration. We append the instruction ``Please reason step by step, and put your final answer within \texttt{\textbackslash boxed\{\}}'' to each question, and then extract the content enclosed in \texttt{\textbackslash boxed\{\}} as the final answer for correctness evaluation.

\textbf{Evaluation.} We evaluate on six mathematical benchmarks: AIME24, AIME25, AMC23, MATH-500, Minerva Math, and OlympiadBench\citep{yang2023aimo,MATH500_bench,Minerva_bench,OlympiadBench}, with Top-$p\!=\!0.95$ and $T\!=\!0.7$ at inference. AIME24/25 and AMC23 are evaluated $32$ times and the other benchmarks $4$ times, and all results report average accuracy.

\subsection{Main Results}

Table~\ref{tab:main_results} presents the full performance comparison, where results marked with $\dagger$ are cited from prior work such as EAPO, STEER, GRPO-$\mathrm{Clip}_{\nu}$, Lp-Reg, DGPO, JustRL, STAPO, A3PO, and SimpleTIR\citep{EAPO,STEER,Clip_beta_entropy_dynamic,Lp_Reg,DGPO,he2025justrlscaling15bllm,STAPO,A3PO,xue2025simpletir}. STARE consistently delivers substantial gains across the six math benchmarks at scales ranging from 1.5B to 32B and three scenarios.

\textbf{Short-CoT Scenario.} At the 7B scale, STARE-O1 attains an average accuracy of $54.4\%$, outperforming STEER ($49.1\%$, $+5.3\%$) and GRPO-ds ($49.1\%$, $+5.3\%$), while reaching $44.2\%$ and $23.8\%$ on AIME24 and AIME25, corresponding to improvements of roughly $10\%$ and $7\%$ over DAPO. At the 14B scale, STARE-O1 ($52.0\%$) surpasses GRPO-ds ($46.1\%$) by $5.9\%$. At the 32B scale, STARE-O1 ($60.7\%$) exceeds GRPO-ds ($56.1\%$) by $4.6\%$. 

\textbf{Long-CoT Scenario.} At the 1.5B scale, STARE-O1 reaches an average of $65.9\%$, substantially exceeding EAPO ($57.0\%$, $+8.9\%$) and DAPO ($55.1\%$, $+10.8\%$), while also outperforming JustRL on AIME24 and AIME25. At the 8B scale, STARE-O1 ($62.0\%$) surpasses STAPO ($58.8\%$) and GRPO-ds ($59.0\%$), and STARE-C2 further raises the average to $62.2\%$.

\textbf{Tool-Use Scenario.} STARE-O1 achieves an average of $59.4\%$, improving over GRPO-ds ($53.9\%$) by $5.5\%$ and outperforms SimpleTIR, while reaching $53.2\%$ and $37.5\%$ on AIME24 and AIME25. STARE-C2 further lifts the average to $60.4\%$.

\textbf{Key findings.} Three principal observations emerge. \textbf{(i)} On the challenging AIME24 and AIME25 benchmarks, STARE improves over GRPO-ds by $4\%$--$8\%$ in average accuracy, with gains of $3\%$--$6\%$ across all six benchmarks. \textbf{(ii)} Across different thinking scenarios and model scales from 1.5B to 32B, STARE consistently outperforms the majority of competitive RL improvement methods, confirming its effectiveness and robustness. \textbf{(iii)} The additional gains of STARE-C2 further indicate that dual-sided regulation, which simultaneously strengthens entropy-increasing signals and attenuates entropy-decreasing ones, can yield an even more favorable exploration-exploitation balance.

\begin{figure}[ht]
    \centering

    \begin{subfigure}[t]{0.24\textwidth}
        \centering
        \includegraphics[width=\linewidth]{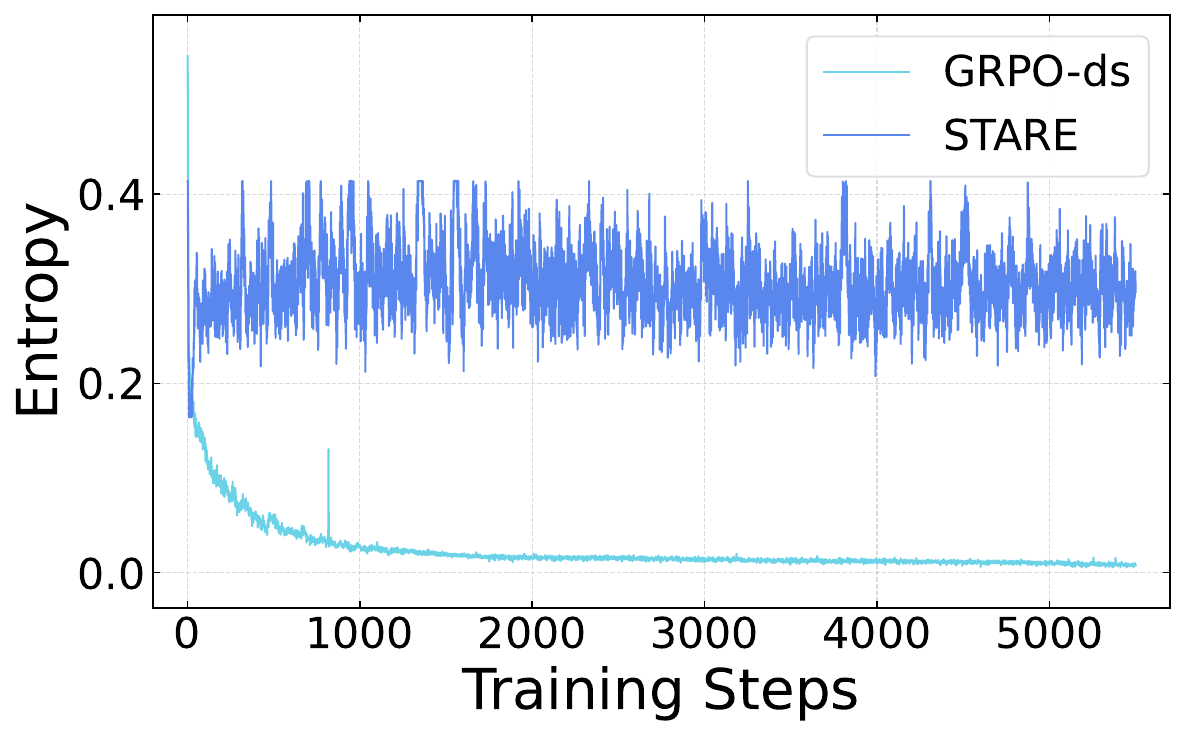}
        \caption{Training Entropy}
        \label{fig:qwen7b_Entropy}
    \end{subfigure}
    \hfill
    \begin{subfigure}[t]{0.24\textwidth}
        \centering
        \includegraphics[width=\linewidth]{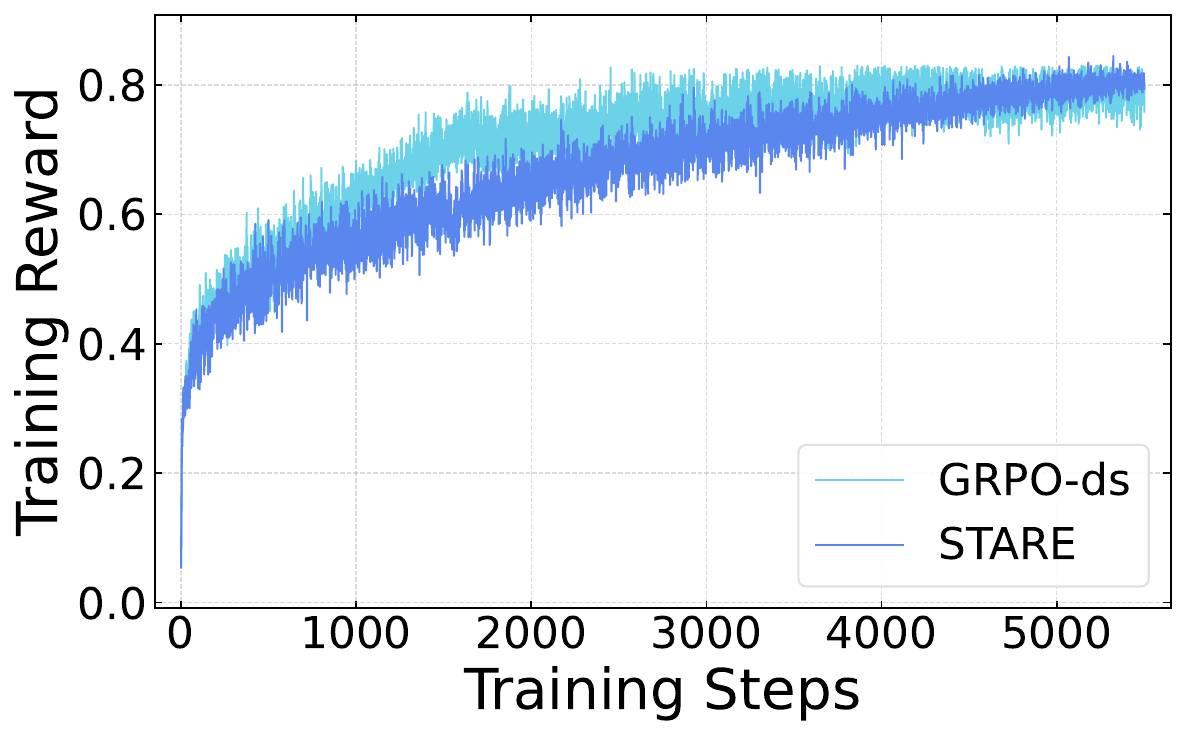}
        \caption{Training Reward}
        \label{fig:qwen7b_Reward}
    \end{subfigure}
    \hfill
    \begin{subfigure}[t]{0.24\textwidth}
        \centering
        \includegraphics[width=\linewidth]{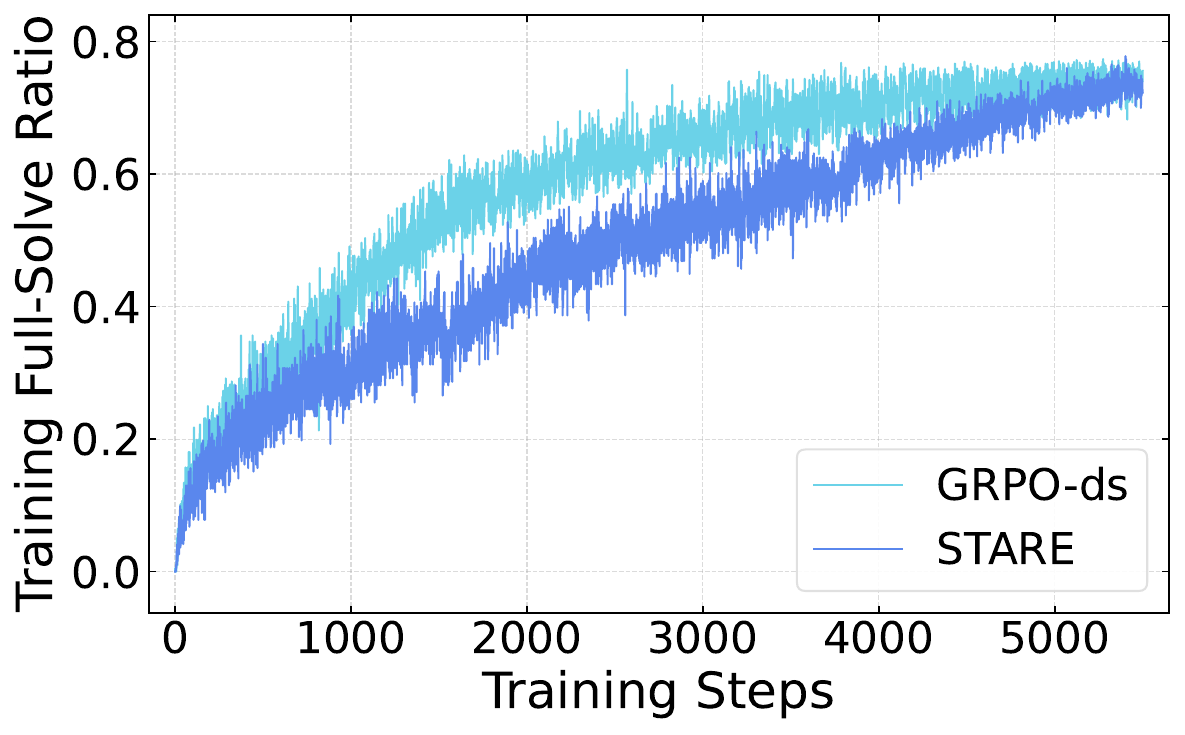}
        \caption{Train Full-solve Ratio}
        \label{fig:qwen7b_Full_solve_Ratio}
    \end{subfigure}
    \hfill
    \begin{subfigure}[t]{0.24\textwidth}
        \centering
        \includegraphics[width=\linewidth]{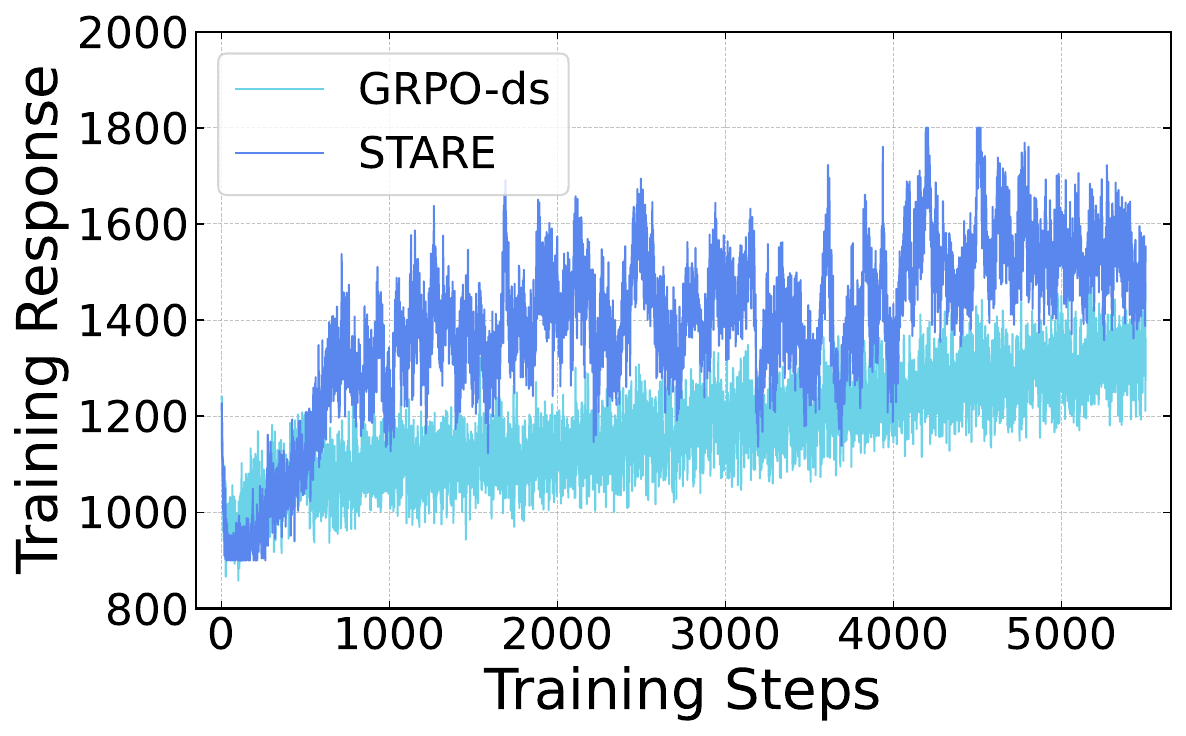}
        \caption{Train Response Length}
        \label{fig:qwen7b_Response}
    \end{subfigure}

    \vspace{0.4em}

    \begin{subfigure}[t]{0.24\textwidth}
        \centering
        \includegraphics[width=\linewidth]{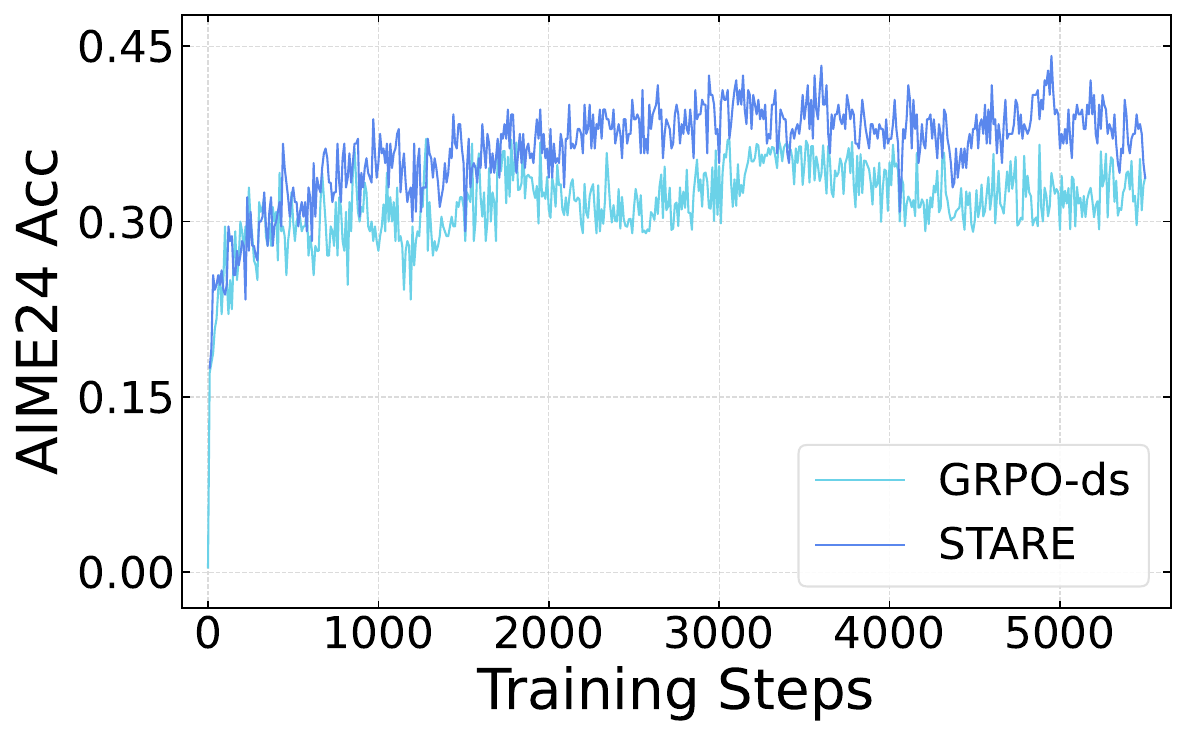}
        \caption{AIME24 Acc}
        \label{fig:qwen7b_AIME24_Acc}
    \end{subfigure}
    \hfill
    \begin{subfigure}[t]{0.24\textwidth}
        \centering
        \includegraphics[width=\linewidth]{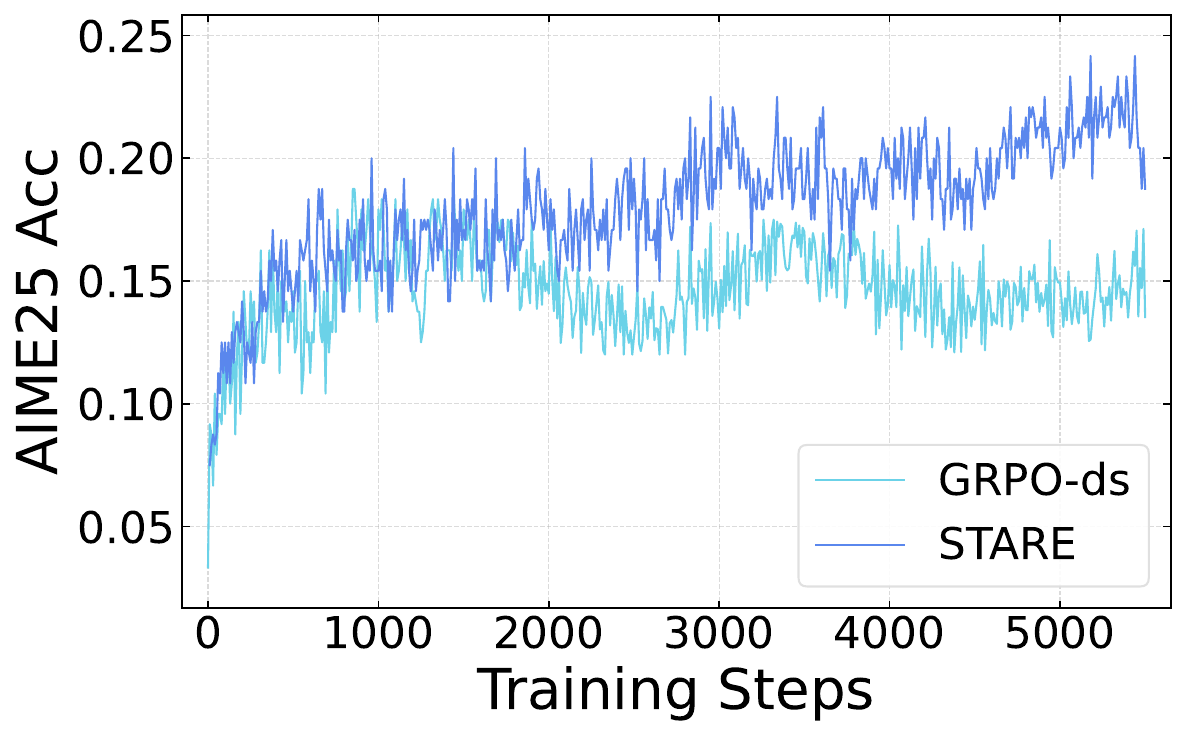}
        \caption{AIME25 Acc}
        \label{fig:qwen7b_AIME25_Acc}
    \end{subfigure}
    \hfill
    \begin{subfigure}[t]{0.24\textwidth}
        \centering
        \includegraphics[width=\linewidth]{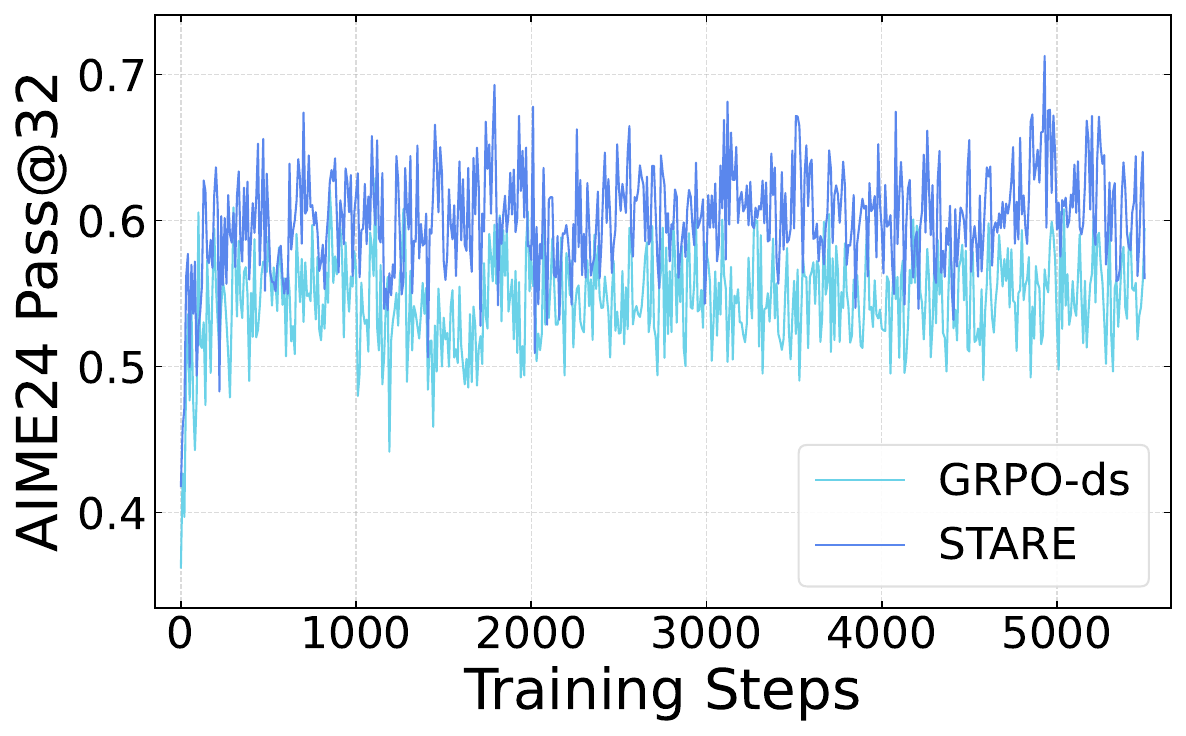}
        \caption{AIME24 Pass@32}
        \label{fig:qwen7b_AIME24_PassN}
    \end{subfigure}
    \hfill
    \begin{subfigure}[t]{0.24\textwidth}
        \centering
        \includegraphics[width=\linewidth]{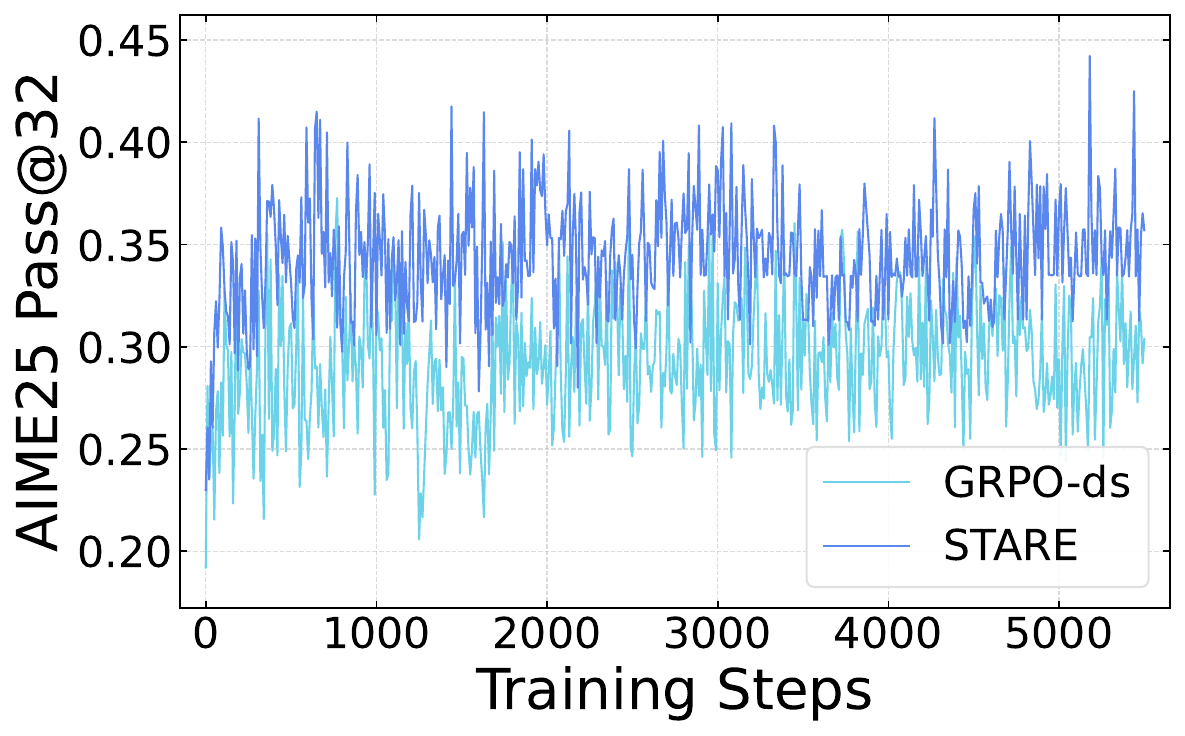}
        \caption{AIME25 Pass@32}
        \label{fig:qwen7b_AIME25_PassN}
    \end{subfigure}

    \caption{Comparison of key training metrics between STARE and GRPO-ds on Qwen2.5-Math-7B-Base model in the Short CoT scenario over 5k RL steps.}
    \label{fig:qwen7b_key_metrics}
\end{figure}

\begin{figure}[ht]
    \centering
    \begin{subfigure}[t]{0.32\textwidth}\centering
        \includegraphics[width=\linewidth]{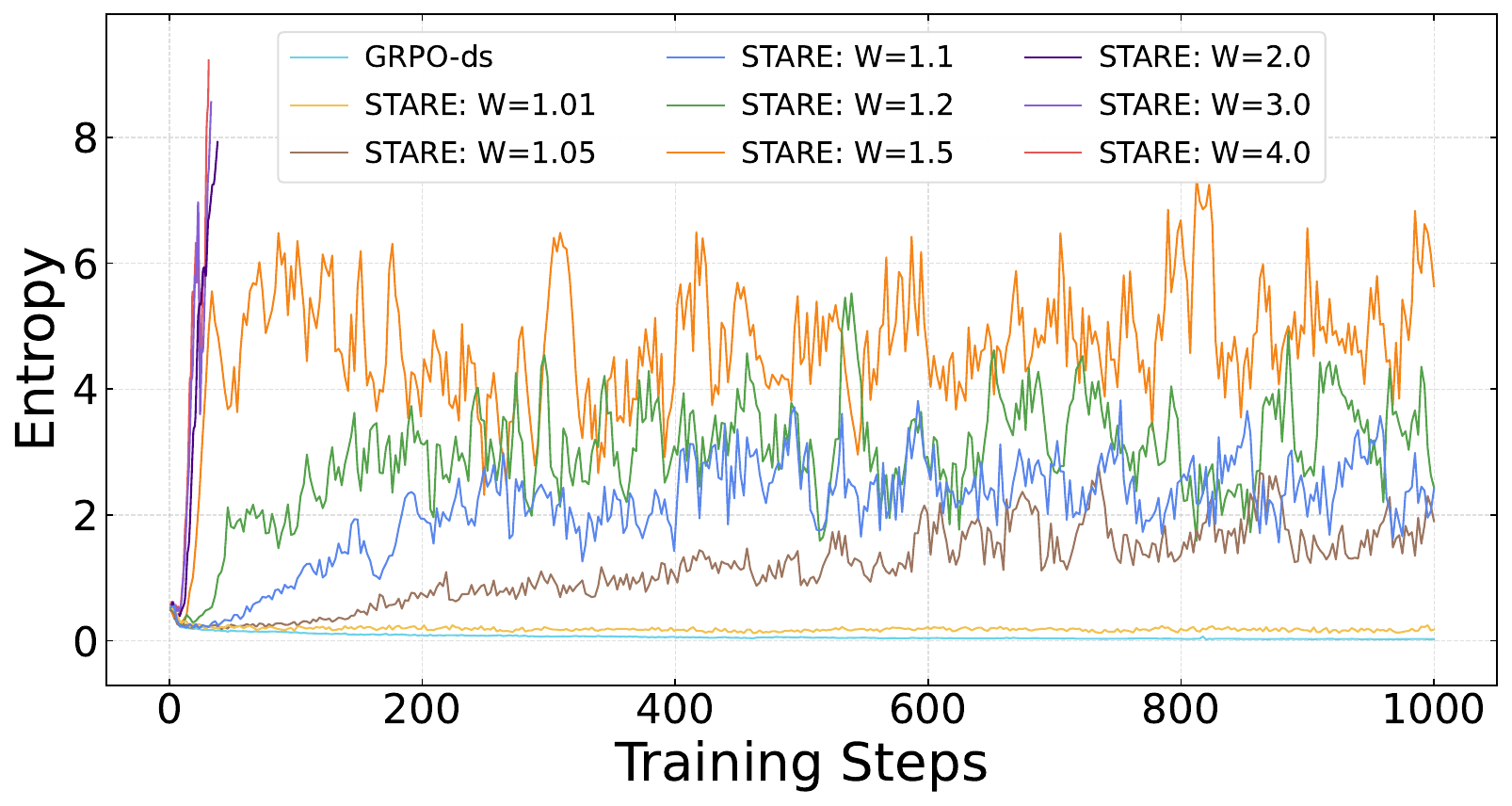}\caption{ Entropy evolution under varying $W$ without target-entropy gating}\label{fig:7b_ablation_W_values}
    \end{subfigure}\hfill
    \begin{subfigure}[t]{0.32\textwidth}\centering
        \includegraphics[width=\linewidth]{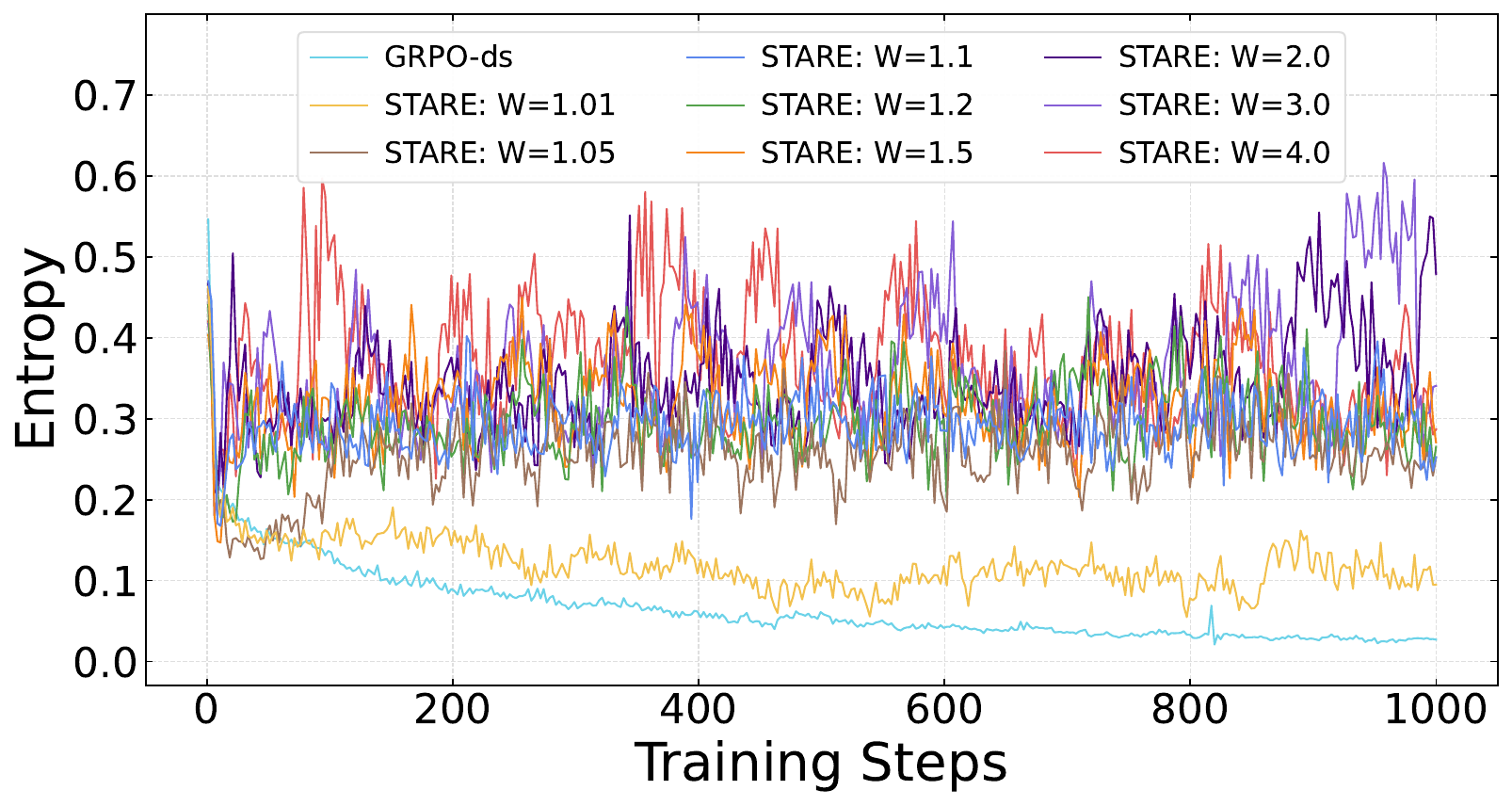}\caption{Entropy evolution under varying $W$ with target-entropy gating}\label{fig:7b_ablation_W_with_target_entropy}
    \end{subfigure}\hfill
    \begin{subfigure}[t]{0.32\textwidth}\centering
        \includegraphics[width=\linewidth]{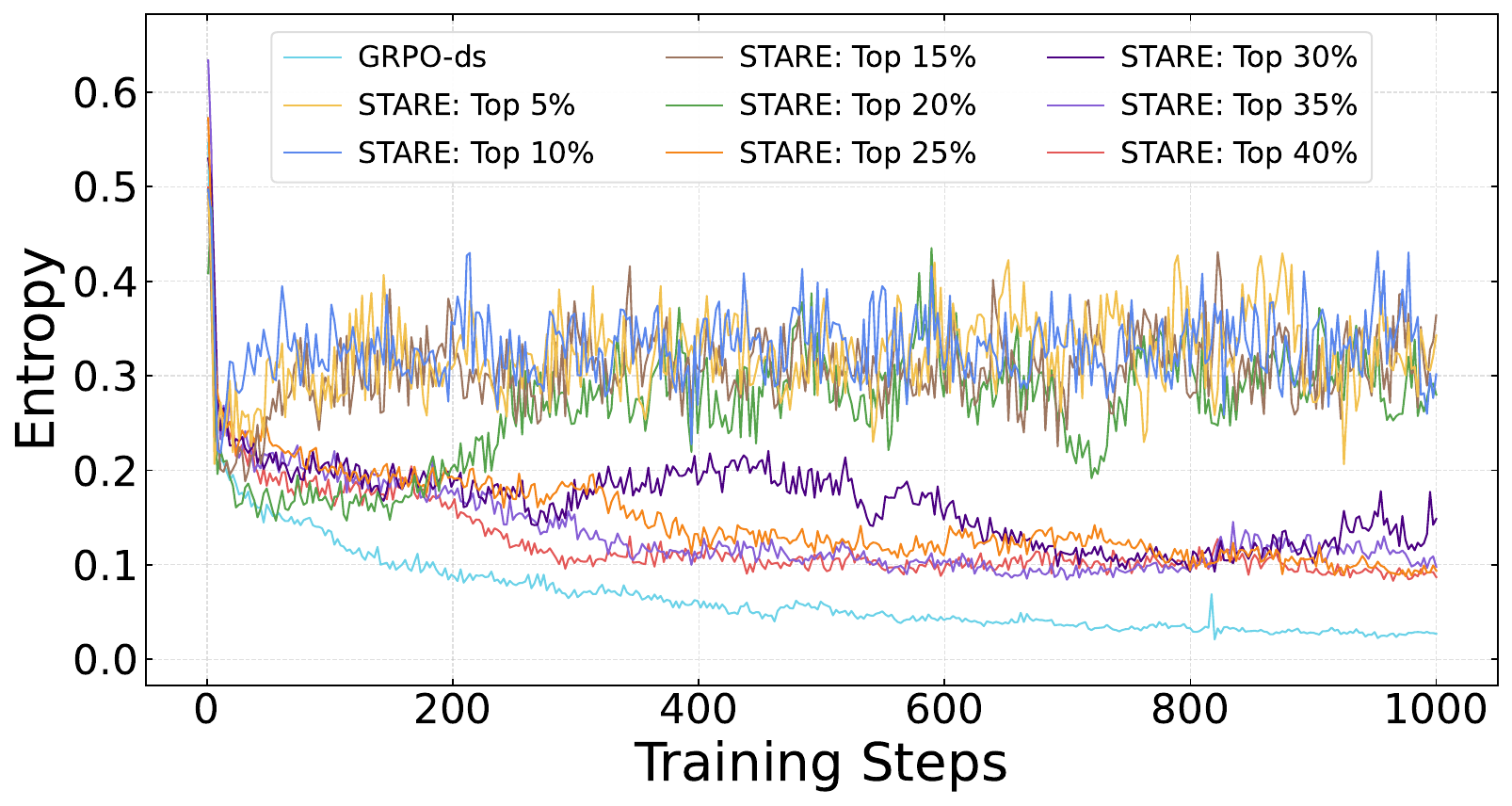}\caption{Entropy evolution under varying $\mathrm{Top}\text{-}P$ ratio ($W:1.1$, $H_{\text{tgt}}:0.3$) }\label{fig:7b_ablation_top_ratio}
    \end{subfigure}
    \caption{STARE policy entropy evolution under ablation on the varying reweighting factor $W$, target-entropy gating, and the high-surprisal selection ratio $P$ on Qwen2.5-Math-7B-Base.}\label{fig:qwen7b_ablation}
\end{figure}

\subsection{Cognitive Analysis}

\textbf{STARE vs. GRPO-ds: Training Dynamics across Scales and Scenarios.} To validate STARE in long-horizon RL, we run $5000$ training steps on Qwen2.5-Math-7B-Base under the Short CoT scenario. Figure~\ref{fig:qwen7b_key_metrics} compares STARE with GRPO-ds on key metrics, and Figures~\ref{fig:qwen14b}--\ref{fig:qwen7b_agent_first} further verify its effectiveness across $1.5$B--$32$B scales and diverse task scenarios. \textbf{Entropy stability and performance evolution.} GRPO-ds exhibits entropy collapse over steps $0$--$1000$, with policy entropy approaching zero (Figure~\ref{fig:qwen7b_key_metrics}(\subref{fig:qwen7b_Entropy})), consistent with Section~\ref{sec:theoretical-analysis}; correspondingly, its AIME24/25 accuracy peaks around step $1000$ and saturates thereafter (Figure~\ref{fig:qwen7b_key_metrics}(\subref{fig:qwen7b_AIME24_Acc})-(\subref{fig:qwen7b_AIME25_Acc})), indicating premature convergence. In contrast, STARE stabilizes entropy near $H_{\text{tgt}}=0.3$ via token-level reweighting and closed-loop gating, with accuracy continuing to rise beyond step $1000$ and peaking at $5000$, thereby unlocking long-horizon RL potential. \textbf{Exploration-exploitation balance.} STARE's Pass@32 consistently exceeds GRPO-ds throughout training (Figure~\ref{fig:qwen7b_key_metrics}(\subref{fig:qwen7b_AIME24_PassN})-(\subref{fig:qwen7b_AIME25_PassN})), preserving output diversity and mitigating mode-seeking; whereas GRPO-ds's reward and Full-Solve Ratio plateau early (Figure~\ref{fig:qwen7b_key_metrics}(\subref{fig:qwen7b_Reward})-(\subref{fig:qwen7b_Full_solve_Ratio})), STARE keeps growing, and its sustained response-length increase (Figure~\ref{fig:qwen7b_key_metrics}(\subref{fig:qwen7b_Response})) reflects deeper reasoning. \textbf{Cross-scale and cross-scenario generalization.} Figures~\ref{fig:qwen14b}-\ref{fig:qwen7b_agent_first} reveal a consistent pattern across Short CoT (14B, 32B), Long CoT (R1-Distill-Qwen-1.5B, Qwen3-8B-Base), and Tool-Use (7B): GRPO-ds suffers collapse with performance saturation, while STARE maintains stable entropy and continuous accuracy gains. Notably, on DeepSeek-R1-Distill-Qwen-1.5B, GRPO-ds's entropy decays below $0.2$ by step $5000$, whereas STARE rapidly restores entropy to the target band starting around step $3500$, with AIME24/25 accuracy improving in tandem (Figure~\ref{fig:qwen1.5b}), showing its intervention-recovery capability and robustness across scales and scenarios. Further details are provided in Appendix ~\ref{app:STARE_GRPO_Scenarios_Scales}

\textbf{Ablation on Key Hyperparameters and Target-Entropy Gating.}
We ablate $W$, $P$, and the target-entropy gate on Qwen2.5-Math-7B-Base (Figure~\ref{fig:qwen7b_ablation}). Under open-loop reweighting ( without the target-entropy gate; Figure~\ref{fig:qwen7b_ablation}(\subref{fig:7b_ablation_W_values})), $W=1.01$ already mitigates the entropy decay of GRPO-ds, $W\geq 1.05$ yields steady growth, and $W\geq 2.0$ triggers divergenc, corroborating the near-criticality property (Corollary~\ref{thm:near-criticality}) that beyond the critical threshold $W$ controls magnitude rather than direction. Open-loop reweighting, however, stabilizes entropy at an excessively high level, inducing over-exploration that hampers overall training(Appendix~\ref{app:STARE_Target_Entropy_Closed_Loop}). With the closed-loop gate ($H_{\text{tgt}}=0.3$), Figure~\ref{fig:qwen7b_ablation}(\subref{fig:7b_ablation_W_with_target_entropy}) shows that all $W\in[1.05, 1.5]$ steer entropy into the target band with bounded oscillation, confirming closed-loop stability and 
\clearpage
\begingroup

\begin{figure}[H]
    \centering
    \begin{subfigure}[t]{0.32\textwidth}\centering
        \includegraphics[width=\linewidth]{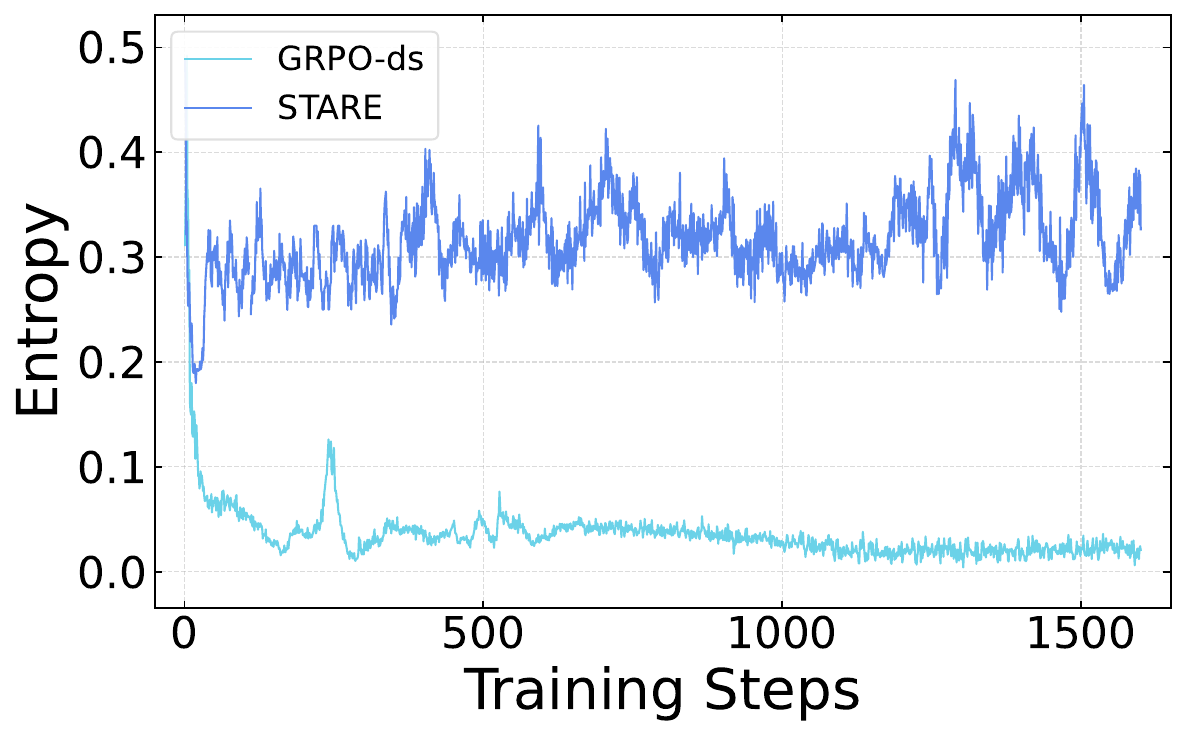}\caption{Training Entropy}\label{fig:14b_entropy}
    \end{subfigure}\hfill
    \begin{subfigure}[t]{0.32\textwidth}\centering
        \includegraphics[width=\linewidth]{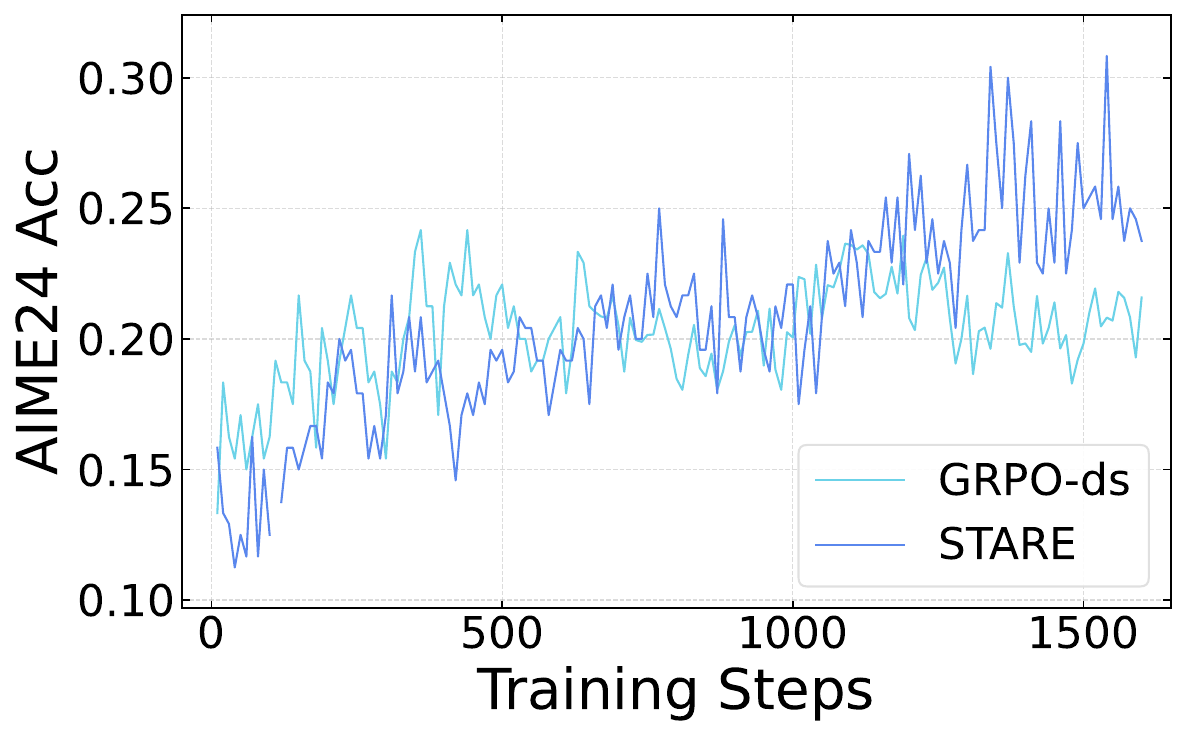}\caption{AIME24 Acc}\label{fig:14b_aime24}
    \end{subfigure}\hfill
    \begin{subfigure}[t]{0.32\textwidth}\centering
        \includegraphics[width=\linewidth]{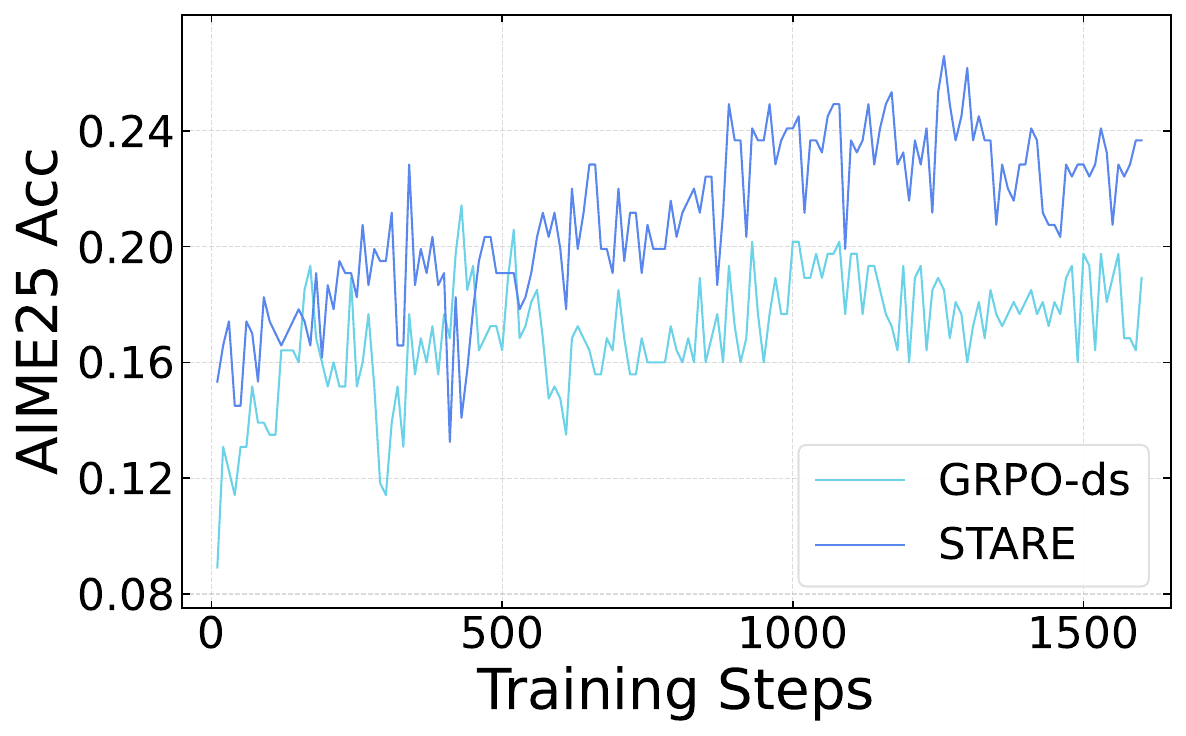}\caption{AIME25 Acc}\label{fig:14b_aime25}
    \end{subfigure}
    \caption{Training dynamics of STARE vs. GRPO-ds on Qwen2.5-14B-Instruct in the Short CoT scenario: policy entropy, AIME24 accuracy, and  AIME25 accuracy.}\label{fig:qwen14b}
\end{figure}

   \vspace{-8pt}

\begin{figure}[H]
    \centering
    \begin{subfigure}[t]{0.32\textwidth}\centering
        \includegraphics[width=\linewidth]{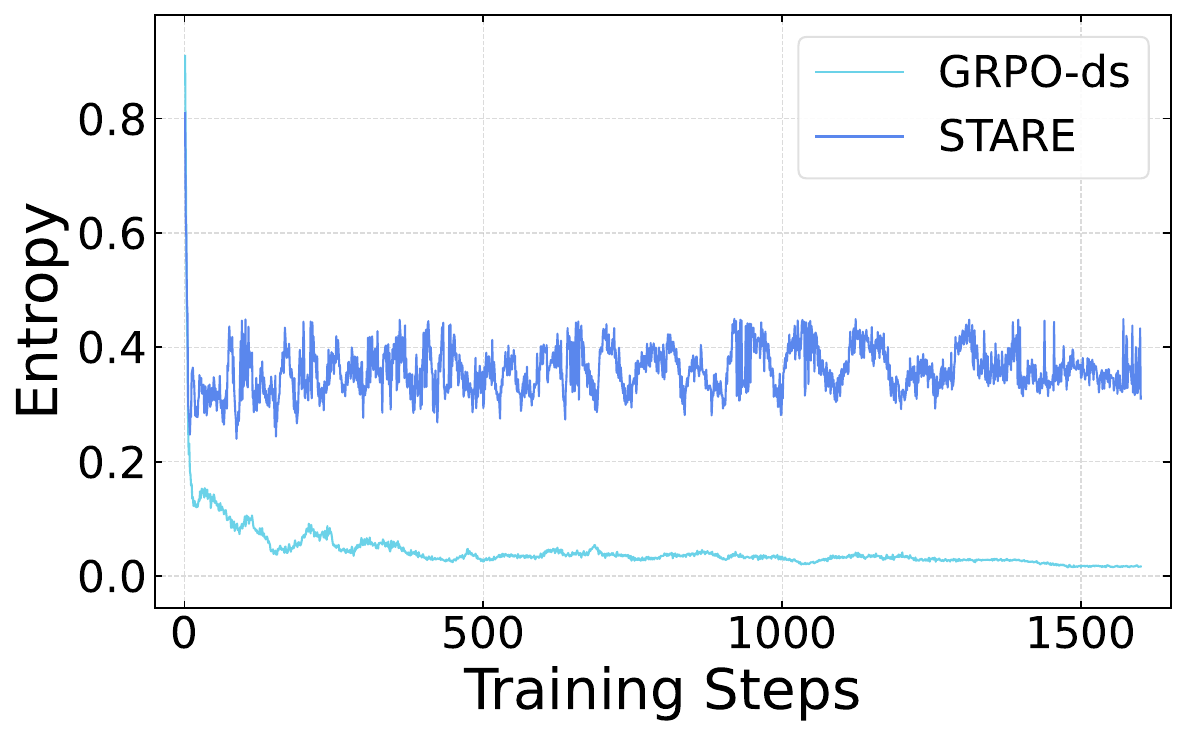}\caption{Training Entropy}\label{fig:32b_sub1}
    \end{subfigure}\hfill
    \begin{subfigure}[t]{0.32\textwidth}\centering
        \includegraphics[width=\linewidth]{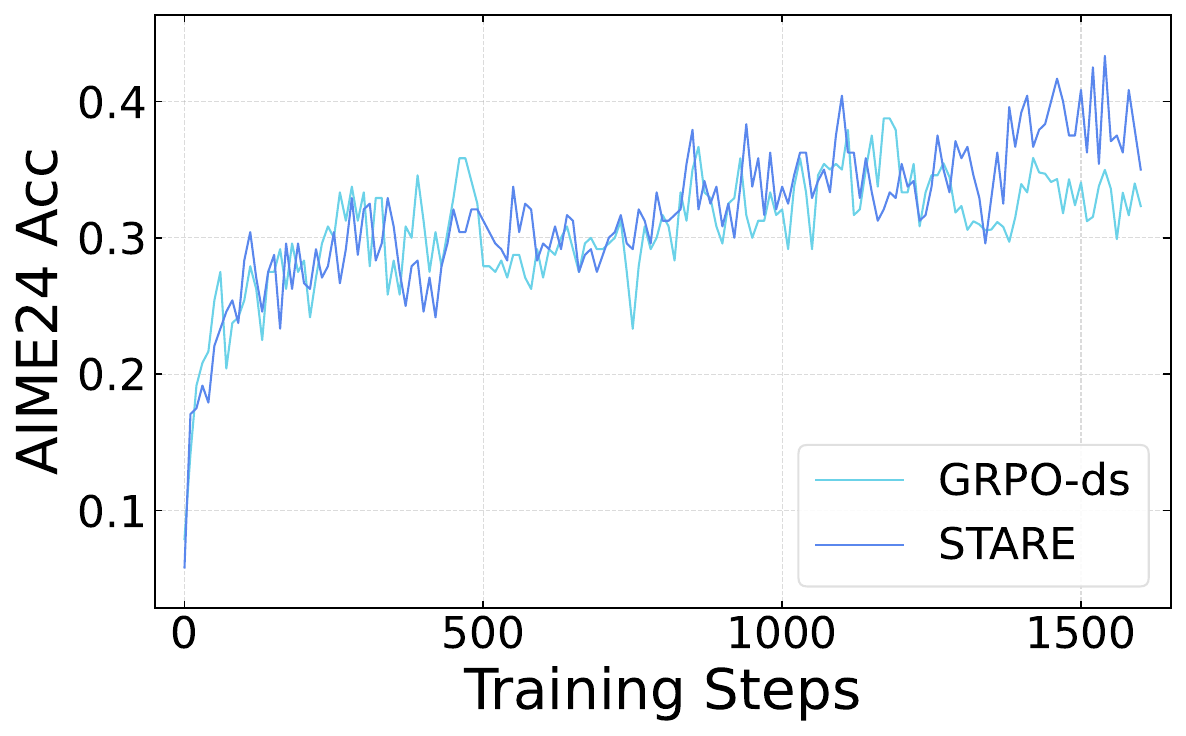}\caption{AIME24 Acc}\label{fig:32b_sub2}
    \end{subfigure}\hfill
    \begin{subfigure}[t]{0.32\textwidth}\centering
        \includegraphics[width=\linewidth]{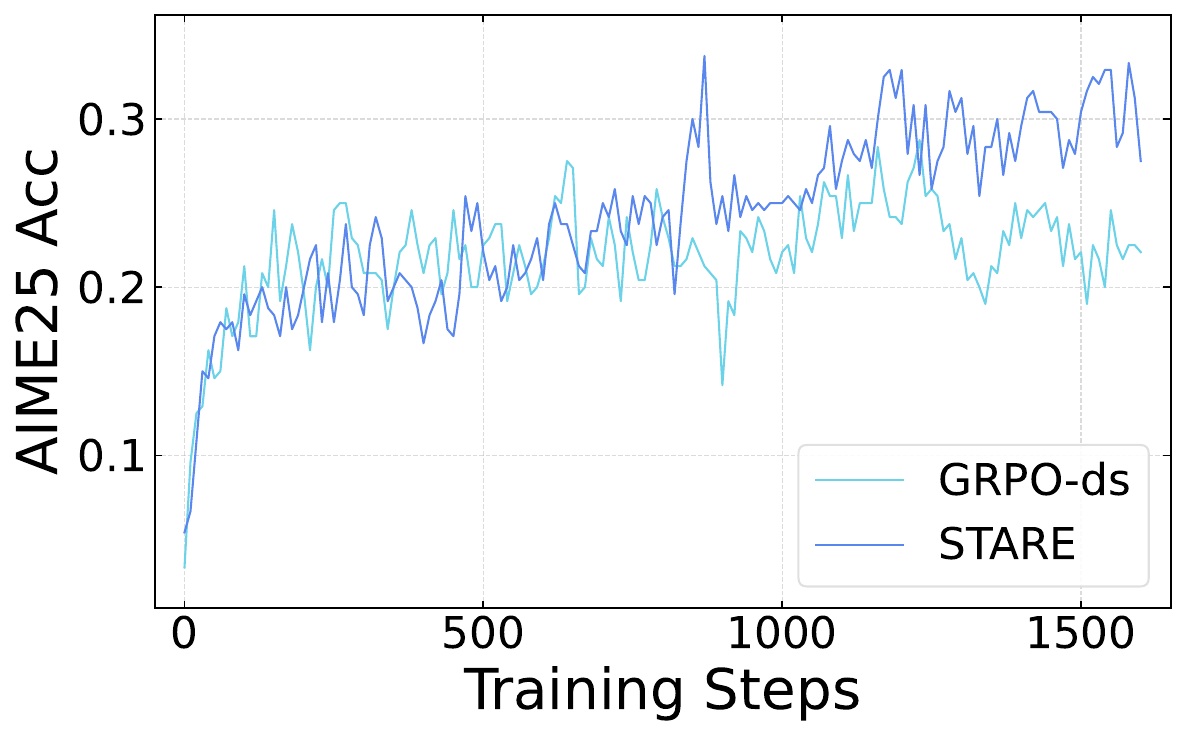}\caption{AIME25 Acc}\label{fig:32b_sub3}
    \end{subfigure}
    \caption{Training dynamics of STARE vs. GRPO-ds on Qwen2.5-32B-Base in the Short CoT scenario: policy entropy, AIME24 accuracy, and AIME25 accuracy.}\label{fig:qwen32b}
\end{figure}

   \vspace{-8pt}

\begin{figure}[H]
    \centering
    \begin{subfigure}[t]{0.32\textwidth}\centering
        \includegraphics[width=\linewidth]{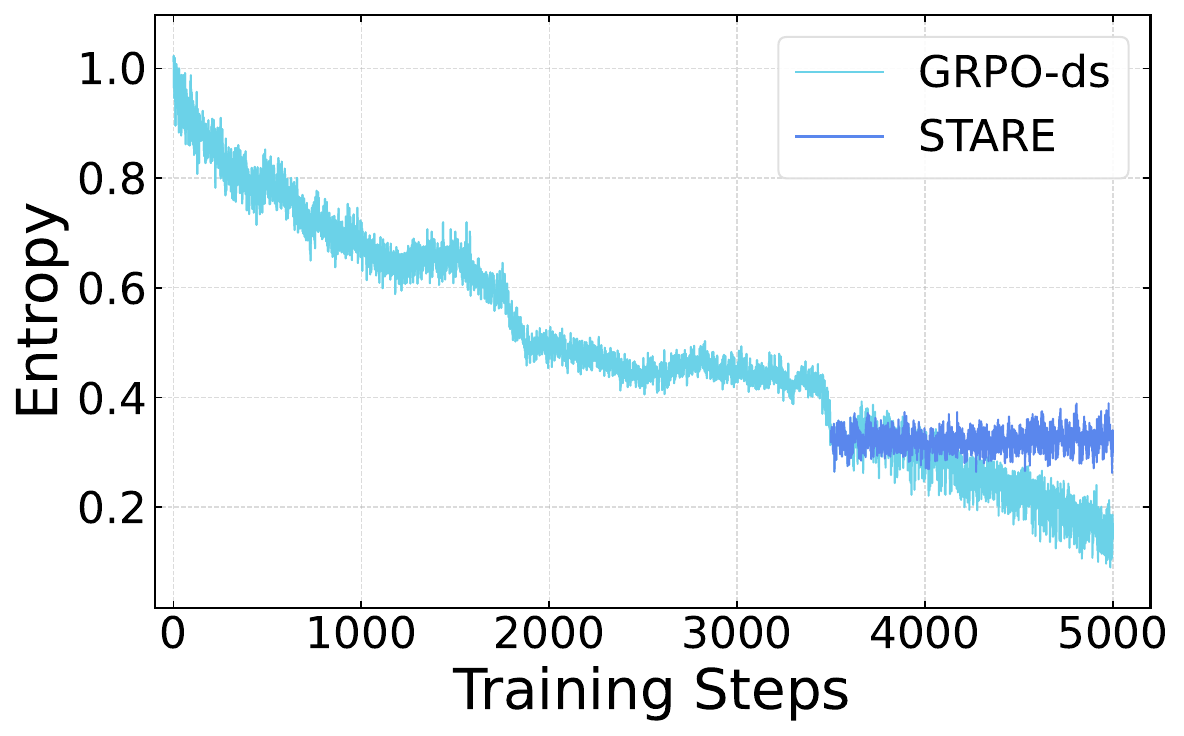}\caption{Training Entropy}\label{fig:1.5b_sub1}
    \end{subfigure}\hfill
    \begin{subfigure}[t]{0.32\textwidth}\centering
        \includegraphics[width=\linewidth]{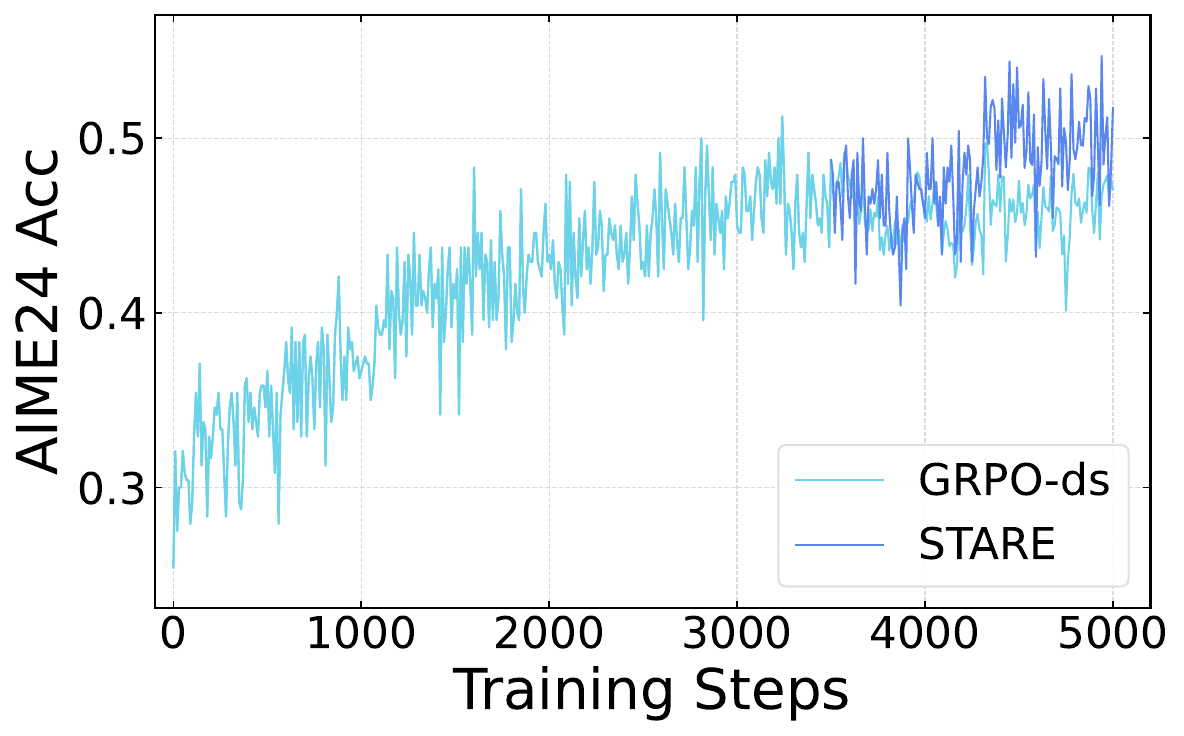}\caption{AIME24 Acc}\label{fig:1.5b_sub2}
    \end{subfigure}\hfill
    \begin{subfigure}[t]{0.32\textwidth}\centering
        \includegraphics[width=\linewidth]{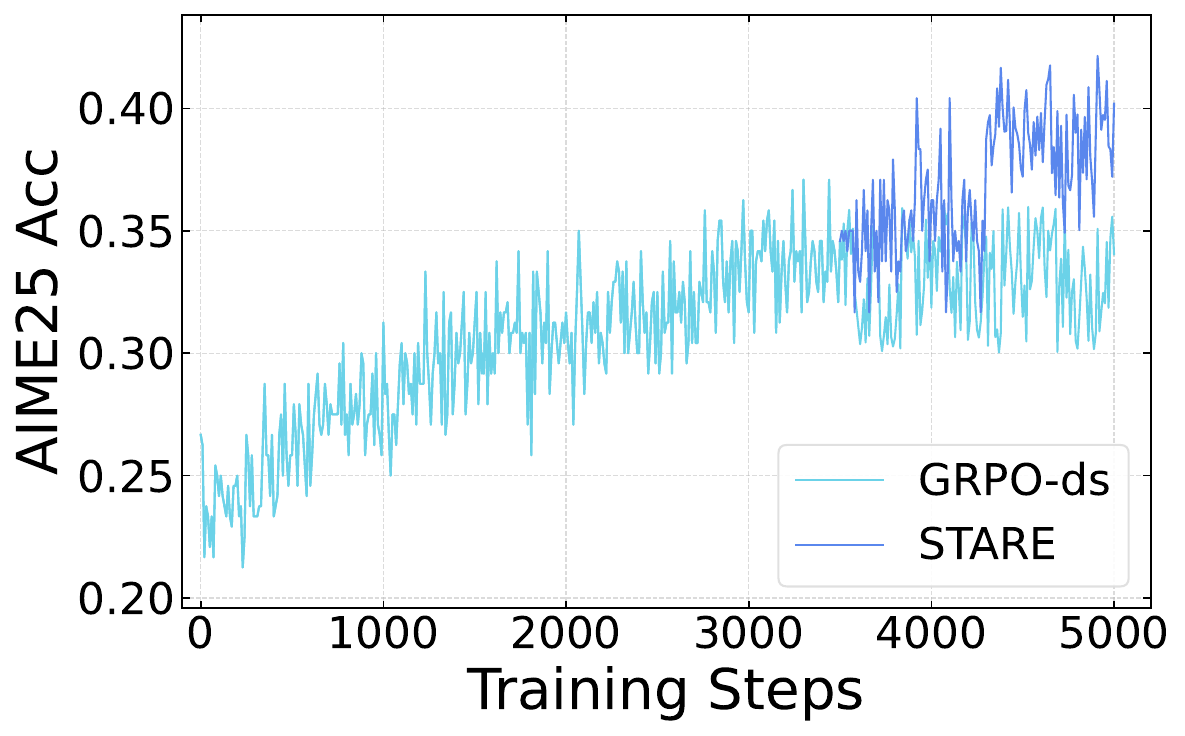}\caption{AIME25 Acc}\label{fig:1.5b_sub3}
    \end{subfigure}
    \caption{Training dynamics of STARE vs. GRPO-ds on DeepSeek-R1-Distill-Qwen-1.5B in the Long CoT scenario:  policy entropy, AIME24 accuracy, and AIME25 accuracy.}\label{fig:qwen1.5b}
\end{figure}

   \vspace{-8pt}

\begin{figure}[H]
    \centering
    \begin{subfigure}[t]{0.32\textwidth}\centering
        \includegraphics[width=\linewidth]{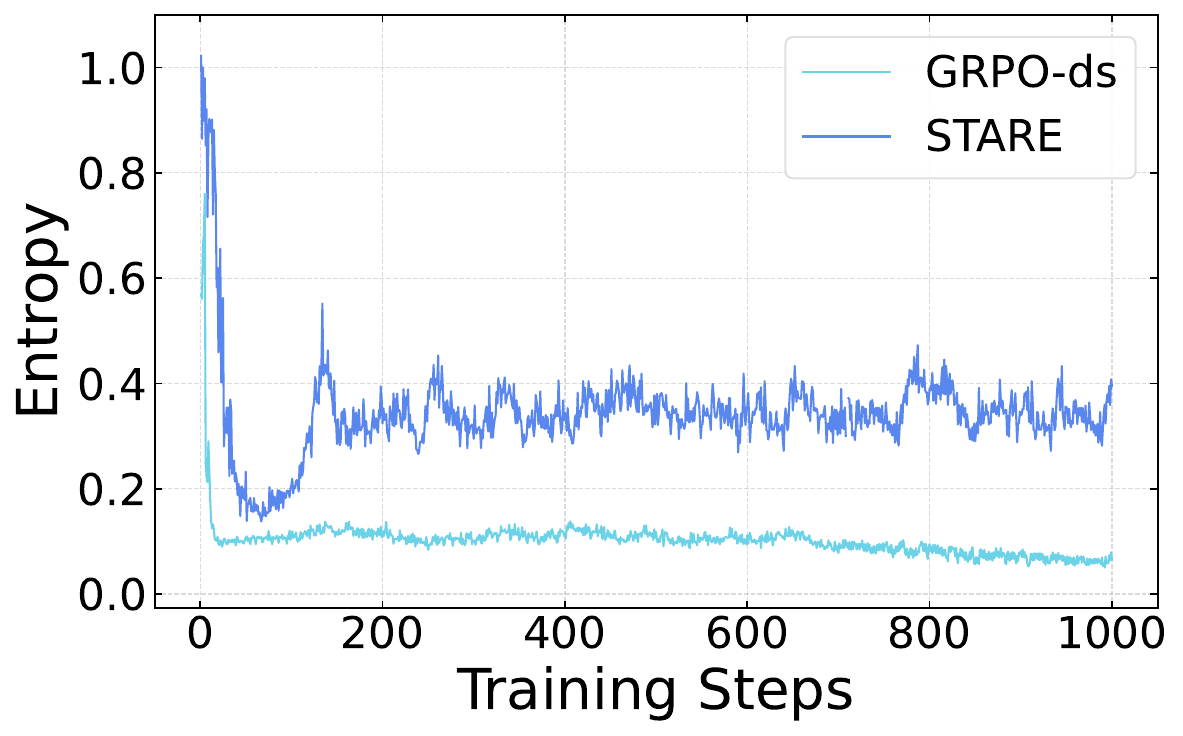}\caption{Training Entropy}\label{fig:3_8b_sub1}
    \end{subfigure}\hfill
    \begin{subfigure}[t]{0.32\textwidth}\centering
        \includegraphics[width=\linewidth]{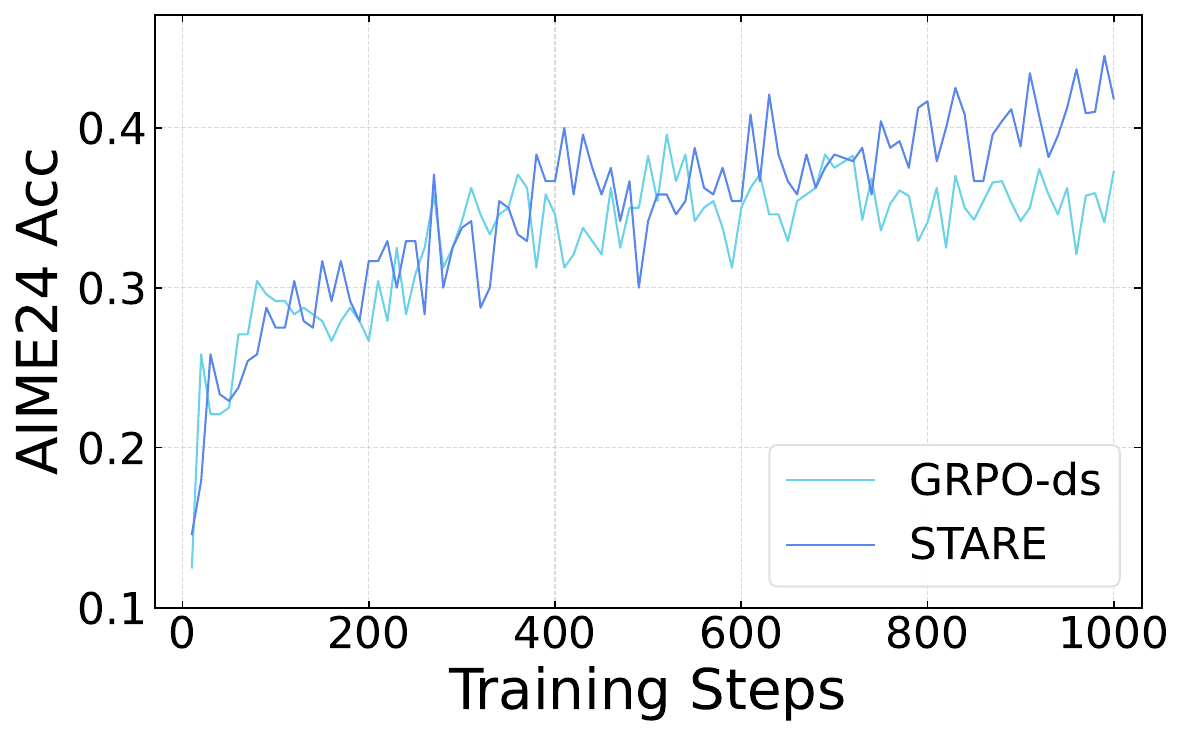}\caption{AIME24 Acc}\label{fig:3_8b_sub2}
    \end{subfigure}\hfill
    \begin{subfigure}[t]{0.32\textwidth}\centering
        \includegraphics[width=\linewidth]{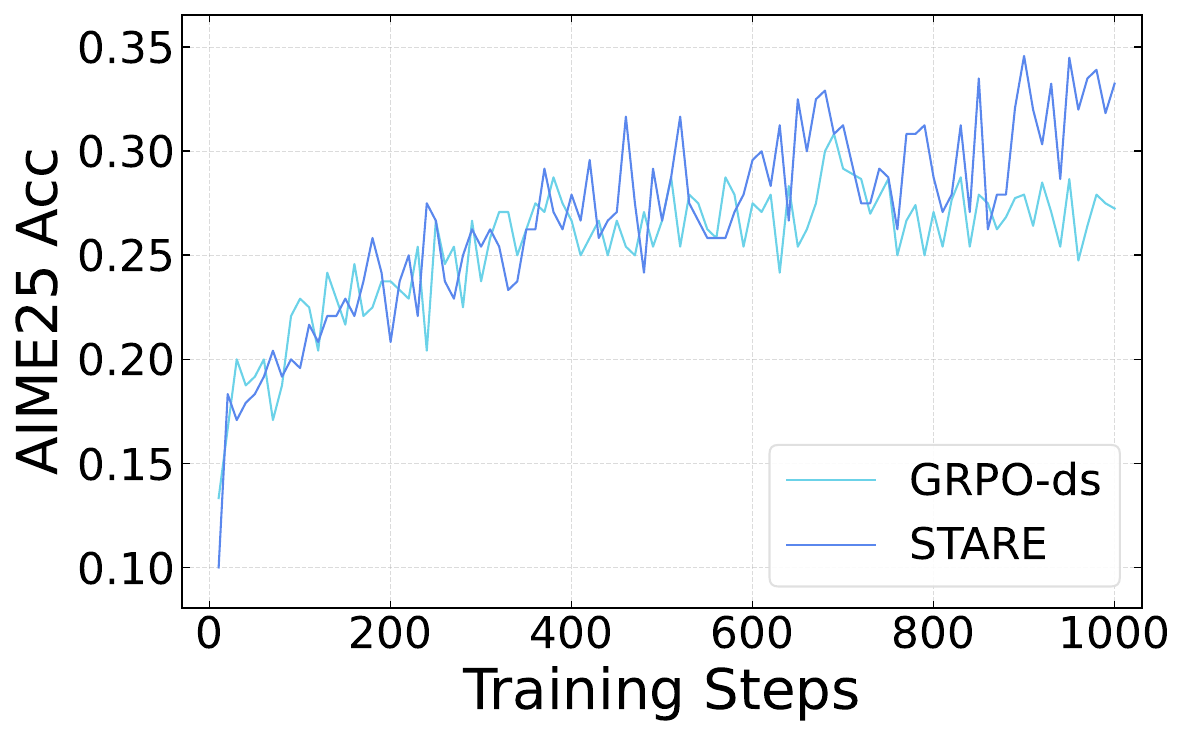}\caption{AIME25 Acc}\label{fig:3_8b_sub3}
    \end{subfigure}
    \caption{Training dynamics of STARE vs. GRPO-ds on Qwen3-8B-Base in the Long-CoT scenario: policy entropy, AIME24 accuracy, and AIME25 accuracy.}\label{fig:qwen3_8b}
\end{figure}


substantially reducing sensitivity to $W$. Fixing $W=1.1$ and $H_{\text{tgt}}=0.3$, Figure~\ref{fig:qwen7b_ablation}     (\subref{fig:7b_ablation_top_ratio}) further shows that entropy stays within the target band for $P\in[5\%, 20\%]$ and remains confined to $[0.1, 0.2]$ even at $P=40\%$, effectively preventing collapse and indicating a broad operating range. Overall, STARE is robust to $W$ and $P$, and the target-entropy gate stably constrains the policy entropy. More detailed entropy evolution and performance comparisons are reported in Appendix~\ref{app:STARE_Target_Entropy_Closed_Loop} ( Table~\ref{tab:stare_target_entropy_gate_ablation_aime}, Figure~\ref{fig:STARE_target_entropy}), with additional analysis provided in Appendix~\ref{app:STARE_Key_Hyperparameters_ablation}.

\endgroup
\clearpage

\begin{figure}[t]
    \centering
    \begin{subfigure}[t]{0.48\textwidth}\centering
    \includegraphics[width=\linewidth]{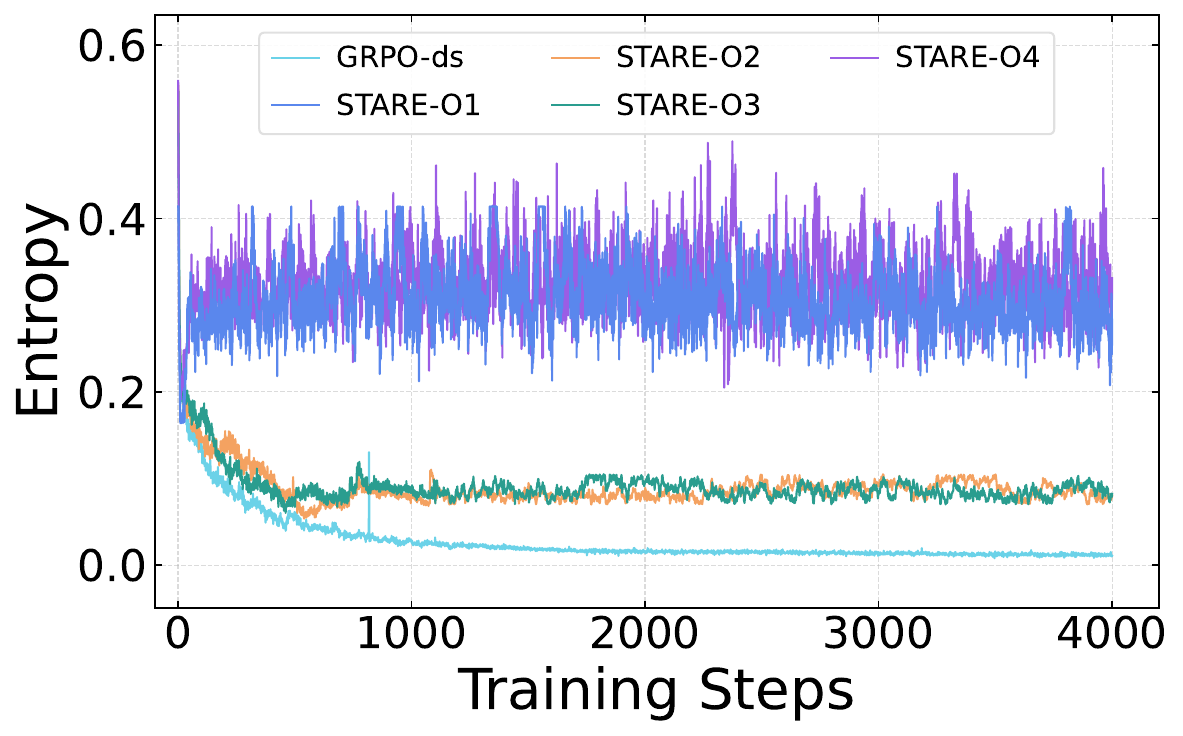}\caption{Entropy evolution under single-polarity operations O1-O4, each targeting one 
entropy-critical quadrant.}\label{fig:qwen7b_text_four_single_quant_entropy_0_4000}
    \end{subfigure}\hfill
    \begin{subfigure}[t]{0.48\textwidth}\centering
        \includegraphics[width=\linewidth]{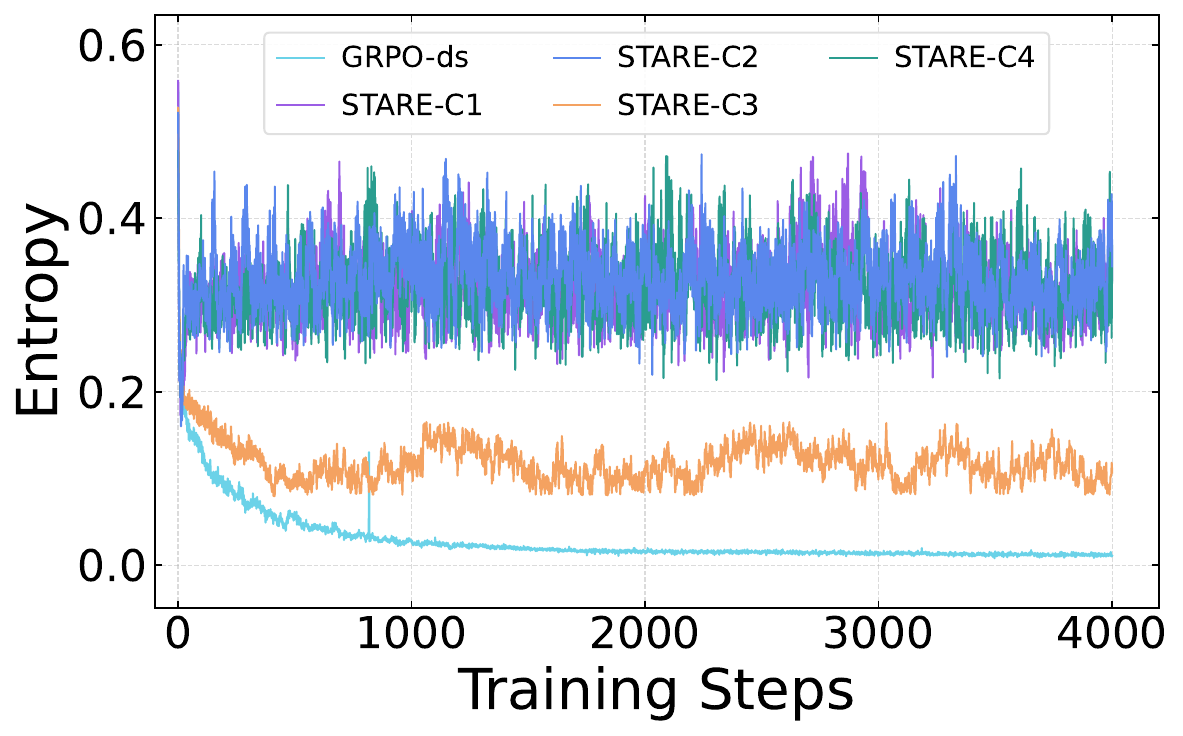}\caption{Entropy curves under combined operations C1-C4, each jointly regulating two 
entropy-critical quadrants.}\label{fig:qwen7b_text_combine_four_quant_entropy_0_4000}
    \end{subfigure}
    \caption{Policy entropy trajectories of all eight STARE variants versus GRPO-ds on 
Qwen2.5-Math-7B-Base over 4000 RL training steps, covering four single-polarity 
operations (O1-O4) and four combined operations (C1-C4) derived from the 
advantage-surprisal four-quadrant decomposition.}\label{fig:stare_single_and_combine_four_quant}
\end{figure}

\textbf{Validation of the High-Surprisal Quantile Proxy.}
Figure~\ref{fig:qwen7b_Surprisal_net_entropy} validates the batch-internal top-$P\%$ surprisal-quantile proxy. Panel~(\subref{fig:qwen7b_surprisal_entropy_below_tgt_s_ratio_0_1000}) shows that the fraction of selected tokens falling within the theoretical entropy-increasing region ($\mathfrak{s}_{i,t} > \mathfrak{s}^*_{i,t}$) rises steadily from $\sim$$60\%$ to $\sim$$95\%$ during training, indicating strong consistency with the critical threshold $\mathfrak{s}^*$ (Proposition~\ref{prop:critical-surprisal-threshold}); panel~(\subref{fig:qwen7b_cumulative_net_entropy_change_0_1000}) further confirms that the cumulative net entropy contribution of $\mathcal{L}_q^+$ remains positive and monotonically increasing, consistent with Corollary~\ref{cor:four-quadrant}. Thus the proxy reliably identifies entropy-critical tokens while obviating per-position solutions of $\Phi(p^*) = 0$; see Appendix~\ref{app:proxy-validation} for further details.

\begin{wraptable}[14]{r}{0.38\textwidth}
\vspace{-0.5cm}
\centering
\caption{The results of single-polarity and combined STARE operations on AIME24 and AIME25.}
\label{tab:stare_polarity_operation_ablation_aime}
\scalebox{0.9}{
\begin{tabular}{lcc}
\toprule
\textbf{Model}  & \textbf{AIME24} & \textbf{AIME25} \\
\midrule
GRPO-ds               & 37.1 & 17.7 \\
\midrule
STARE-O1 & 44.2 & 23.8 \\
STARE-O2  & 40.5 & 20.3 \\
STARE-O3  & 39.6 & 21.6 \\
STARE-O4 & 42.1 & 19.9 \\
\midrule
STARE-C1        & 43.1 & 23.5 \\
STARE-C2         & 42.5 & 24.2 \\
STARE-C3         & 39.9 & 20.8 \\
STARE-C4        & 41.7 & 22.6 \\
\bottomrule
\end{tabular}
}
\end{wraptable}

\textbf{Effects of Single-Polarity and Combined Operations in STARE.}
To validate the token-level reweighting mechanism under the four-quadrant decomposition, we ablate four single-polarity (O1--O4) and four combined (C1--C4) operations, where O1/O3 amplify entropy-increasing signals, O2/O4 attenuate entropy-decreasing ones, and C1--C4 jointly intervene on two quadrants. As shown in Table~\ref{tab:stare_polarity_operation_ablation_aime} and Figure~\ref{fig:stare_single_and_combine_four_quant}, while GRPO-ds suffers rapid entropy decay, all eight variants effectively mitigate this collapse and substantially outperform GRPO-ds AIME24/25 across all four entropy-critical quadrants. Among them, O1 (amplifying $\mathcal{L}_q^+$) and C2 (amplifying $\mathcal{L}_q^+$ while attenuating $\mathcal{L}_q^-$) perform best at $44.2\%/23.8\%$ and $42.5\%/24.2\%$, so we adopt STARE-O1 as the default configuration and STARE-C2 as the dual-sided variant. Further details are provided in Appendix~\ref{app:stare_polarity_ablation}.

\begin{wrapfigure}[14]{r}{0.4\textwidth}
\vspace{-0.45cm}
    \centering
    \includegraphics[width=\linewidth]{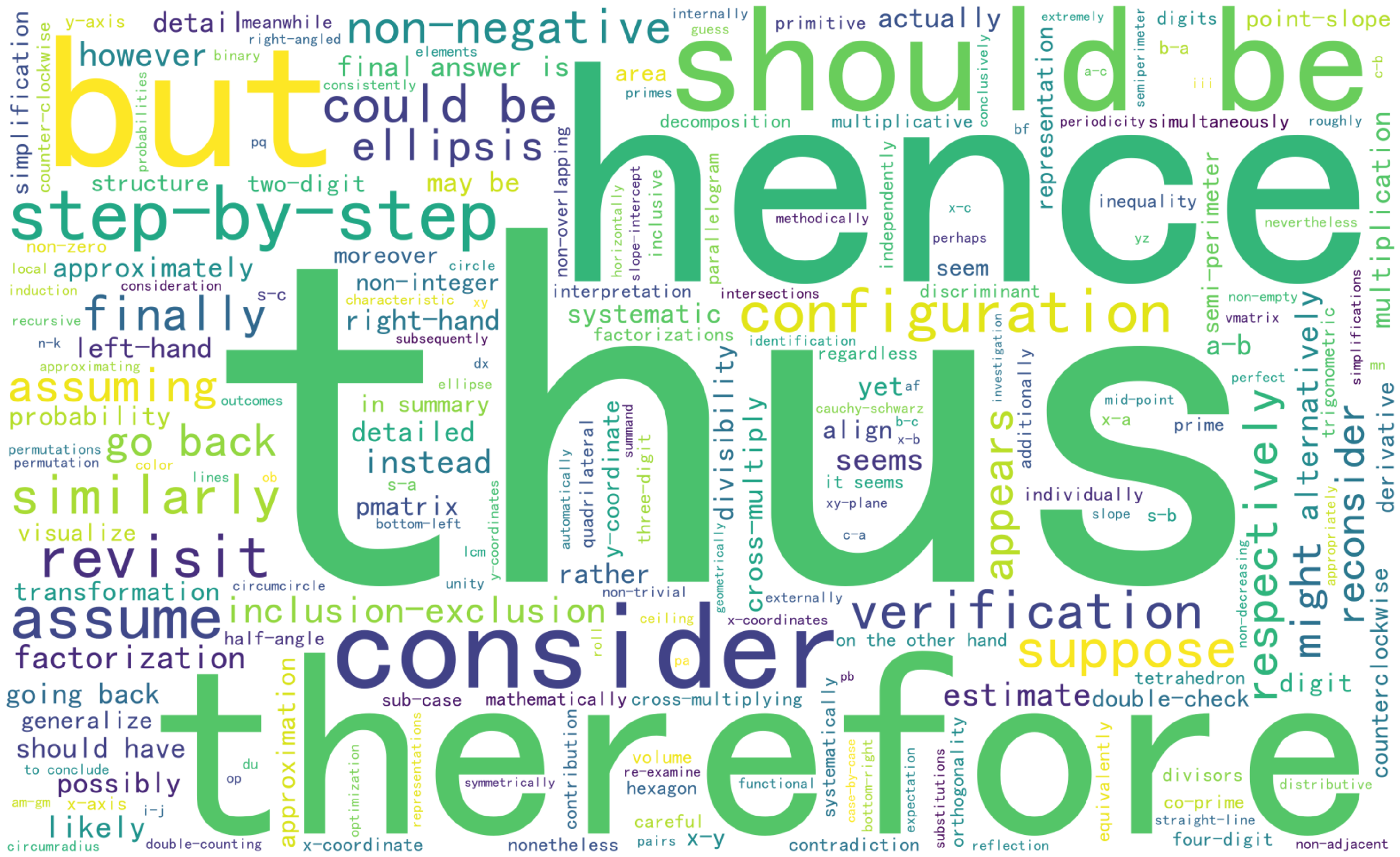}
    \caption{Word cloud of tokens selected by STARE for advantage reweighting.}
    \label{fig:merged_token_frequency_wordcloud}
\end{wrapfigure}
\textbf{Emergent Reflection Behaviors.} To examine how STARE elicits deep reasoning, we inspect the tokens selected for advantage amplification during RL training. As exemplified on Qwen2.5-32B-Base, the word cloud in Figure~\ref{fig:merged_token_frequency_wordcloud} shows reweighted tokens concentrating on uncertainty and self-correction vocabulary, such as \textit{should be}, \textit{but}, \textit{instead}, and \textit{verification}, confirming that the batch-internal surprisal-quantile proxy effectively identifies rare forking tokens~\citep{80_20_Rule} with exploratory semantics. A complementary count across six reflection categories (Figure~\ref{fig:reflection_count_bar_chart_optimized}) further shows STARE markedly surpasses GRPO-ds with the largest margins on reflection and self-correction, jointly demonstrating that STARE activates deep exploration and delivers consistent gains through token-level credit rebalancing. The complete analysis is deferred to Appendix~\ref{app:Emergent_Reflection}.  

Further ablations are deferred to the appendices: fixed vs. adaptive weighting(~\ref{app:fix_adaptive_weights}), target-entropy threshold(~\ref{app:Target_Entropy_Threshold}), gating granularity(~\ref{app:Target_Entropy_Threshold_Gate_Granularity}), off-policy training(~\ref{app:STARE_off_policy}), and fixed-threshold reweighting(~\ref{app:Fixed_Threshold_logprob}).

\section{Conclusion}
We present STARE, a surprisal-guided token-level advantage reweighting mechanism that mitigates policy entropy collapse in GRPO-style RLVR. A first-order analysis of entropy dynamics reveals an advantage–surprisal four-quadrant structure and a near-criticality property. Guided by this insight, STARE reweights entropy-critical tokens via batch-internal surprisal quantiles, coupled with a target-entropy closed-loop gate for stable, minimally intrusive regulation. Across 1.5B–32B models and three task scenarios, STARE sustains thousands of stable RL steps and outperforms DAPO and other competitive baselines by 4\%–8\% on AIME24/25, offering a principled foundation for entropy-aware credit assignment in long-horizon RL post-training of LLMs.

\bibliography{STARE}
\bibliographystyle{colm2024_conference}
\clearpage

\enableappendixcontents
\appendix
\startcontents[appendices]

\section*{Appendix}

\begingroup
\setcounter{tocdepth}{2}
\printcontents[appendices]{}{1}[2]{%
  \section*{Appendix Contents}%
}
\endgroup

\clearpage

\section*{Appendix Overview}

The appendices are organized as follows.
Appendix~\ref{app:related_work} reviews related work on reinforcement learning with verifiable rewards, existing entropy-collapse mitigation strategies, and emerging token-level perspectives, situating STARE within these literatures and clarifying its conceptual and methodological distinctions from prior approaches.
Appendix~\ref{app:additional-experiments} reports additional empirical results that complement the main paper, including a detailed comparison between STARE and GRPO-ds across diverse scenarios and model scales from 1.5B to 32B (Appendix~\ref{app:STARE_GRPO_Scenarios_Scales}), an empirical validation of the batch-internal high-surprisal quantile proxy against the theoretical critical threshold (Appendix~\ref{app:proxy-validation}), ablations of all four single-polarity and four combined token-level reweighting operations (Appendix~\ref{app:stare_polarity_ablation}), ablations on the key hyperparameters $W$ and $P$ together with the target-entropy gate (Appendices~\ref{app:STARE_Key_Hyperparameters_ablation} and~\ref{app:STARE_Target_Entropy_Closed_Loop}), an analysis of emergent reflection behaviors elicited by STARE (Appendix~\ref{app:Emergent_Reflection}), a comparison between fixed and adaptive reweighting schedules (Appendix~\ref{app:fix_adaptive_weights}), ablations on the target-entropy threshold and on the gate granularity (Appendices~\ref{app:Target_Entropy_Threshold} and~\ref{app:Target_Entropy_Threshold_Gate_Granularity}), a validation of STARE under off-policy training (Appendix~\ref{app:STARE_off_policy}), and a comparison with fixed-threshold low-probability token reweighting (Appendix~\ref{app:Fixed_Threshold_logprob}).
Appendix~\ref{app:Algorithm_stare} presents the complete pseudocode of the default STARE-O1 procedure, integrating surprisal-guided entropy-critical token selection, fixed-weight token-level advantage reweighting, and batch-level closed-loop target-entropy gating into a unified algorithm.
Appendix~\ref{app:basic-derivations} establishes the basic differential identities used in Sections~\ref{sec:preliminaries} and~\ref{sec:first-order-gradient} and proves Theorem~\ref{thm:token-entropy-variation}.
Appendix~\ref{app:theory-proofs} proves Proposition~\ref{prop:critical-surprisal-threshold}, Corollary~\ref{cor:four-quadrant}, Theorem~\ref{thm:entropy-neutrality}, Proposition~\ref{prop:entropy-gradient-reweighting}, and Theorem~\ref{thm:near-criticality}, provides the detailed asymmetric entropy contribution analysis, and states explicitly the statistical assumptions underlying the near-criticality result.
Appendix~\ref{app:cross-step} formalizes the cross-step entropy dynamics discussed in Section~\ref{sec:cross-step}.
Appendix~\ref{app:single-polarity} presents all single-polarity operations referenced in the main text, together with sample-level and token-level closed-loop extensions.
Appendix~\ref{app:combined-operations} presents all combined operations and the adaptive weighting scheme.
Finally, Appendix~\ref{app:limitations} discusses the limitations of our study and its broader societal impacts.

To keep the notation unambiguous, Appendices~\ref{app:basic-derivations}--\ref{app:cross-step} use the exact theoretical sets from Section~\ref{sec:theoretical-analysis},
\[
\widetilde{\mathcal L}^{+}
\triangleq
\{(i,t):\hat A_i>0,\ \mathfrak s_{i,t}>\mathfrak s_{i,t}^{*}\},
\qquad
\widetilde{\mathcal L}^{-}
\triangleq
\{(i,t):\hat A_i<0,\ \mathfrak s_{i,t}>\mathfrak s_{i,t}^{*}\},
\]
namely, the sets of token positions whose trajectories have positive (resp.\ negative) advantage and whose token surprisal exceeds the position-dependent critical threshold. By contrast, Appendices~\ref{app:single-polarity} and~\ref{app:combined-operations} use the quantile-based proxy sets employed in the practical implementation of Section~\ref{sec:method}, namely \(\mathcal L_q^{+}\) and \(\mathcal L_q^{-}\).

\section{Related Work}\label{app:related_work}

\textbf{Reinforcement Learning with Verifiable Rewards(RLVR).} In recent years, large language models (LLMs) have advanced rapidly\citep{openai2024o1,claude_example,deepseek_r1,team2025hunyuan,xu2024wizardlm,luo2024wizardcoder}. RLVR has emerged as the dominant post-training paradigm for enhancing the reasoning capability of LLMs, as it leverages verifiable signals to provide precise outcome-level rewards while avoiding the overfitting risks inherent in learned reward models\citep{PPO,VinePPO,Reward_Overoptimization}. GRPO removes the value network and estimates the advantage through group-normalized rewards, demonstrating strong effectiveness in mathematical reasoning, code generation, and tool-use tasks, and further eliciting emergent behaviors such as long chain-of-thought reasoning and self-reflection\citep{deepseek_r1,shao2024deepseekmathpushinglimitsmathematical,luo2023wizardmath,team2025hunyuan,wei2022chain,luo2024arenalearning,Long_cot_Demystifying,luo2024wizardarena,shen2025aienhanced,luo2025agentmath}. Subsequent studies extend the GRPO framework along several dimensions, including advantage estimation, loss aggregation, and sampling strategies, among which DAPO has become a representative baseline through a combination of asymmetric clipping, dynamic sampling, and token-level loss normalization\citep{DAPO,yue2025vapo,EntropyAdv,skywork-or1,chen2026flexible_DID}. As training proceeds over more optimization steps, however, GRPO-style algorithms commonly suffer from policy entropy collapse, in which the entropy decays rapidly, the output diversity vanishes, within-group rollouts become homogeneous, and the number of trainable steps is ultimately capped\citep{semantic_entropy,revisiting_entropy_Prog_Adv_Reweight_jin2025,chen2025overthinking,Leap_Does_rl}. Existing studies have confirmed the prevalence of this phenomenon and have established empirical correlations between policy entropy and downstream performance; nevertheless, a fine-grained theoretical characterization of the underlying token-level gradient causes of entropy collapse remains absent.

\textbf{Entropy Collapse Mitigation Methods}.Existing mitigation strategies fall into three categories. The first category protects low-probability tokens by adjusting the clipping thresholds of the importance-sampling ratio, as exemplified by the clip-higher mechanism in DAPO and by subsequent variants such as differentiated clipping and smooth gating; the influence of these mechanisms on entropy is asymmetric and difficult to control precisely, and in the on-policy regime where the sampling ratio remains close to one, clipping is rarely activated, leaving the actual regulatory capacity limited\citep{DAPO,yue2025vapo,off_policy_actor,chen2026flexible_DID,BAPO,zhou2026offseeker}. The second category applies trajectory-level differentiated weighting between positive and negative samples, including asymmetric importance sampling and the upweighting of rare correct rollouts; since these methods operate at the trajectory granularity, they cannot distinguish the opposing entropy effects of different tokens within the same trajectory\citep{W_Reinforce_Negative_Zhu2025TheSE,A3PO,80_20_Rule,NegativeGradientGRPO2025,EntropyPIC}. The third category couples token-level entropy information into the advantage or loss through entropy-induced advantages\citep{EntropyAdv,Lp_Reg,skywork-or1,Clip_cov}, explicit entropy regularization, or token filtering based on entropy variation; entropy rewards tend to overamplify the signal of high-entropy tokens and induce oscillations, regularization remains highly sensitive to the choice of its coefficient, and entropy-variation-based methods either rely on hard-to-estimate information about unsampled tokens or impose oversimplified binary partitions. In addition, raising the sampling temperature only delays rather than prevents entropy collapse. Overall, existing methods either operate at an excessively coarse granularity or lack a principled understanding of the underlying collapse mechanism.

\textbf{Token-Level Perspectives and STARE.} Recent studies have begun to focus on the differentiated contributions of individual tokens\citep{80_20_Rule,A3PO,Clip_beta_entropy_dynamic,chen2026flexible_DID,W_Reinforce_Negative_Zhu2025TheSE}. Some works demonstrate that a small subset of high-entropy tokens dominates the effective learning signal in RLVR and that performing gradient updates only on this subset already yields efficient performance gains\citep{80_20_Rule}; another line of work identifies critical tokens at decision branching points along reasoning chains and encourages exploration at these positions\citep{chen2026flexible_DID,A3PO,EntropyAdv,EntropyPIC}. These findings echo the analysis presented in this paper. Most existing methods, however, remain heuristic in nature and either fail to keep the policy entropy stable and controllable or lack comprehensive validation across multiple scenarios, model scales, and long-horizon training settings. Starting from a first-order analysis of token-level entropy dynamics, STARE establishes an advantage-surprisal four-quadrant decomposition that exposes the credit assignment mismatch under shared trajectory-level advantages, together with a near-criticality property; building on this analysis, STARE identifies entropy-critical tokens through a batch-internal surprisal-quantile proxy, selectively reweights their effective advantages, and incorporates a target-entropy closed-loop gate, thereby achieving principled token-level regulation of the policy entropy through a minimally invasive modification to the GRPO objective. STARE is comprehensively validated across model scales ranging from 1.5B to 32B and across three task families covering Short-CoT, Long-CoT, and tool-use scenarios, where it consistently delivers stable performance improvements and further unlocks the optimization potential of long-horizon RL training.

\section{Additional Experiments and Analysis on STARE}\label{app:additional-experiments}

\subsection{STARE vs. GRPO-ds: Detailed Comparison across Diverse Scenarios and Model Scales in RL Training}\label{app:STARE_GRPO_Scenarios_Scales}

To validate the effectiveness of STARE in long-horizon RL training, we conduct RL training of $5000$ steps on Qwen2.5-Math-7B-Base under the Short CoT scenario. Figure~\ref{fig:qwen7b_key_metrics} compares STARE against GRPO-ds on key metrics, while Figures~\ref{fig:qwen14b}--\ref{fig:qwen7b_agent_first} further verify the effectiveness of STARE across model scales ranging from $1.5$B to $32$B and across diverse task scenarios.

\textbf{Entropy stability and performance evolution.}
In Figure~\ref{fig:qwen7b_key_metrics}(\subref{fig:qwen7b_Entropy}), GRPO-ds exhibits entropy collapse during the early training phase ($0$--$1000$ steps), in which the policy entropy decreases sharply and approaches zero, which is consistent with the theoretical analysis presented in Section~\ref{sec:theoretical-analysis}. Meanwhile, the accuracy of GRPO-ds on AIME24 and AIME25 peaks around step $1000$ and subsequently saturates, fluctuating without further improvement (Figure~\ref{fig:qwen7b_key_metrics}(\subref{fig:qwen7b_AIME24_Acc})-(\subref{fig:qwen7b_AIME25_Acc})), thereby indicating premature convergence of the policy distribution. In contrast, STARE stabilizes the policy entropy near $H_{\text{tgt}}=0.3$ through token-level advantage reweighting and closed-loop entropy gating, and the accuracy of STARE continues to improve beyond step $1000$, reaching the optimum at step $5000$, which extends the number of trainable steps and unlocks the optimization potential of long-horizon RL training.

\textbf{Exploration-exploitation balance.}
In Figure~\ref{fig:qwen7b_key_metrics}(\subref{fig:qwen7b_AIME24_PassN})-(\subref{fig:qwen7b_AIME25_PassN}), the Pass@32 of STARE consistently surpasses that of GRPO-ds throughout training, which indicates that the policy retains sufficient output diversity and effectively mitigates mode-seeking behavior. In Figure~\ref{fig:qwen7b_key_metrics}(\subref{fig:qwen7b_Reward})-(\subref{fig:qwen7b_Full_solve_Ratio}), the training reward and Full-Solve Ratio of GRPO-ds rise rapidly during the early phase and subsequently plateau, whereas STARE maintains a steady upward trajectory throughout training. Meanwhile, the sustained growth of response length under STARE (Figure~\ref{fig:qwen7b_key_metrics}(\subref{fig:qwen7b_Response})) suggests that the model addresses complex problems by extending the reasoning depth.

\textbf{Cross-scale and cross-scenario generalization.}
Figures~\ref{fig:qwen14b}-\ref{fig:qwen7b_agent_first} systematically present the behavior of STARE under broader configurations. Across the Short CoT scenario(14B, 32B), the Long CoT scenario(R1-Distill-Qwen-1.5B, Qwen3-8B-Base), and the tool-use scenario(7B), the results exhibit a consistent pattern: GRPO-ds suffers from entropy collapse accompanied by performance saturation, whereas STARE maintains entropy stability and delivers continuous accuracy improvements. Notably, on DeepSeek-R1-Distill-Qwen-1.5B, GRPO-ds undergoes gradual entropy decay that drops below $0.2$ by step $5000$, whereas under STARE the policy entropy rapidly recovers to the target band and remains stable starting around step $3500$, with the accuracy on AIME24 and AIME25 improving in tandem (Figure~\ref{fig:qwen1.5b}), which validates the intervention-recovery capability of the proposed method.

In summary, STARE maintains stable policy entropy and delivers consistent performance gains across model scales ranging from $1.5$B to $32$B and across the three task scenarios, which confirms the effectiveness and robustness of the proposed method.

\begin{figure}[t]
    \centering
    \begin{subfigure}[t]{0.48\textwidth}\centering
    \includegraphics[width=\linewidth]{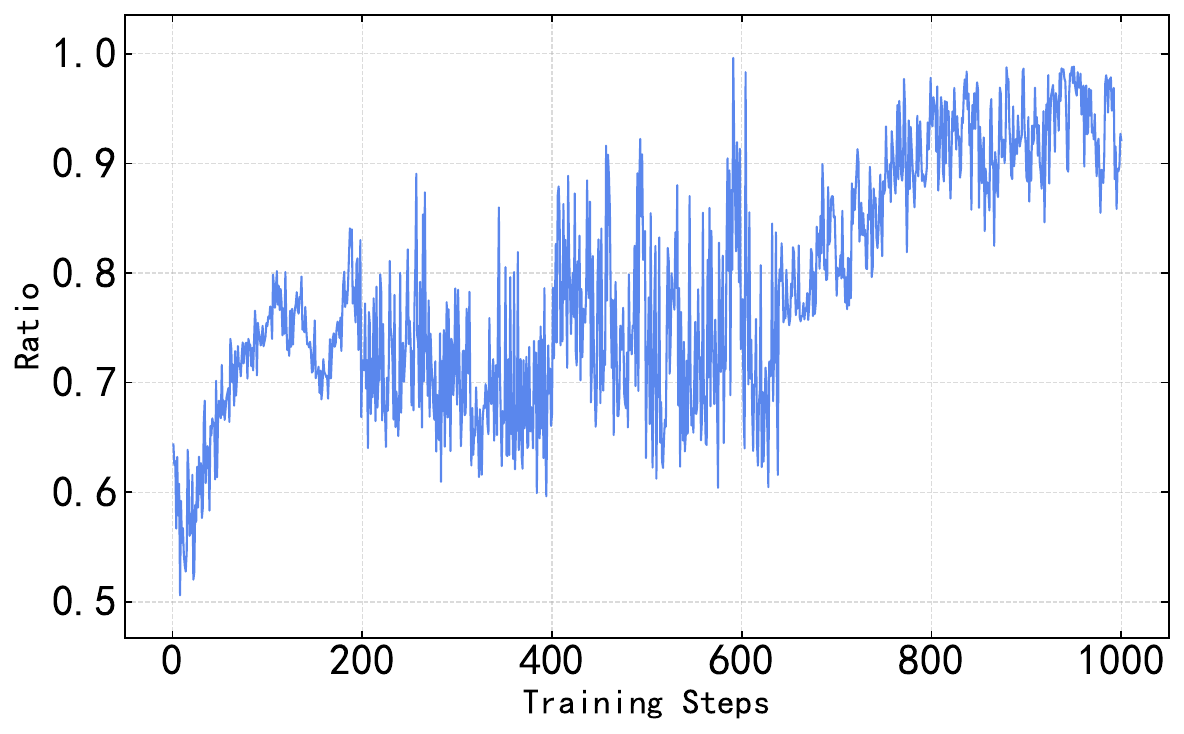}\caption{Fraction of top-$P\%$ high-surprisal tokens in $\mathcal{L}^+_q$ that fall within the theoretical entropy-increasing region ($s_{i,t} > s^*_{i,t}$) across training.}\label{fig:qwen7b_surprisal_entropy_below_tgt_s_ratio_0_1000}
    \end{subfigure}\hfill
    \begin{subfigure}[t]{0.48\textwidth}\centering
        \includegraphics[width=\linewidth]{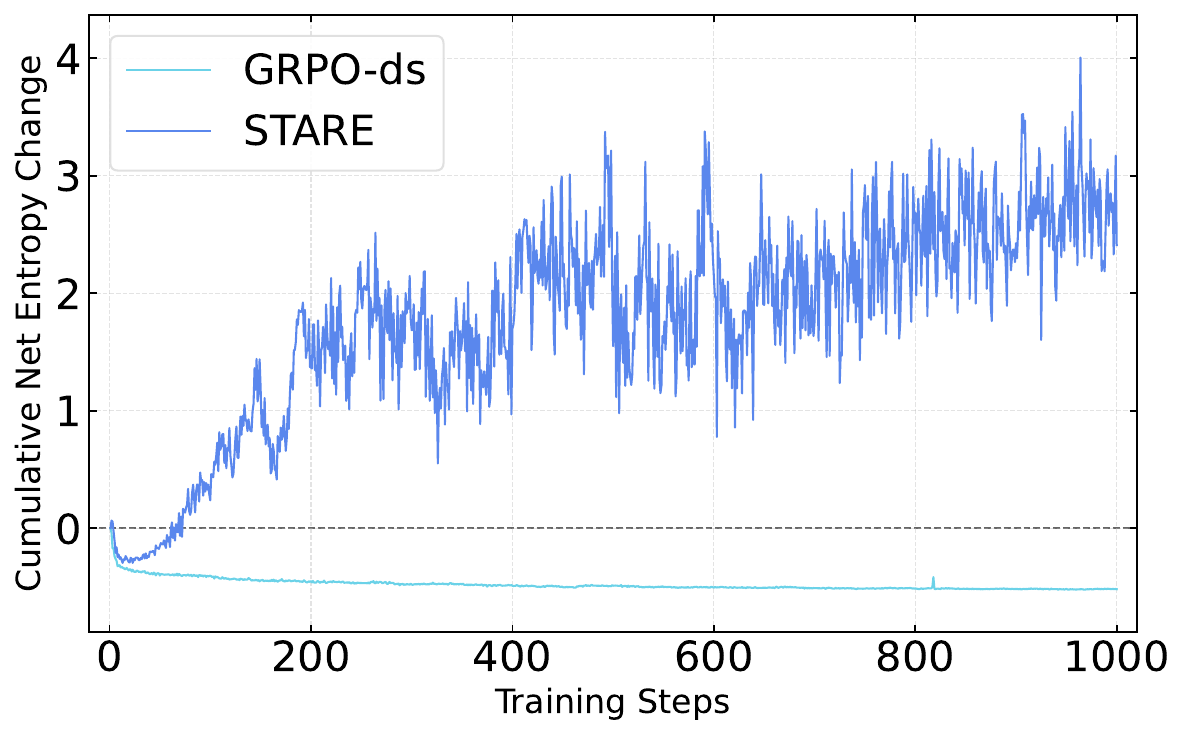}\caption{Cumulative net entropy contribution of $\mathcal{L}^+_q$ under STARE versus GRPO-ds.}\label{fig:qwen7b_cumulative_net_entropy_change_0_1000}
    \end{subfigure}
    \caption{Validation of the batch-internal surprisal-quantile proxy on Qwen2.5-Math-7B-Base 
over 1000 RL training steps, examining its alignment with the theoretical critical threshold and the cumulative net entropy contribution of the selected token subset.}\label{fig:qwen7b_Surprisal_net_entropy}
\end{figure}

\subsection{Validation of the High-Surprisal Quantile Proxy.}\label{app:proxy-validation}
To verify the effectiveness of the batch-internal top-$P\%$ surprisal-quantile proxy in STARE, Figure~\ref{fig:qwen7b_Surprisal_net_entropy} evaluates this proxy from two complementary perspectives. Figure~\ref{fig:qwen7b_Surprisal_net_entropy}(\subref{fig:qwen7b_surprisal_entropy_below_tgt_s_ratio_0_1000}) reports the fraction of tokens selected by the top-$P\%$ high-surprisal criterion that fall within the theoretical entropy-increasing region ($\mathfrak{s}_{i,t} > \mathfrak{s}^*_{i,t}$). This fraction rises steadily from approximately $60\%$ in the early training phase to around $95\%$, indicating strong consistency between the surprisal-quantile proxy and the theoretical critical threshold $\mathfrak{s}^*$ (Proposition~\ref{prop:critical-surprisal-threshold}). Figure~\ref{fig:qwen7b_Surprisal_net_entropy}(\subref{fig:qwen7b_cumulative_net_entropy_change_0_1000}) further shows that the cumulative net entropy contribution of this subset remains positive and monotonically increasing throughout training, thereby confirming that $\mathcal{L}_q^+$ consistently produces a net entropy-increasing effect, as predicted by Corollary~\ref{cor:four-quadrant}. These results demonstrate that the batch-internal surprisal-quantile mechanism effectively approximates the theoretical threshold $\mathfrak{s}^*$ and accurately identifies entropy-critical tokens while eliminating the need to solve $\Phi(p^*) = 0$ at each position, thereby validating both the feasibility and the effectiveness of the proposed proxy strategy.

\subsection{Detailed Effects of Single-Polarity and Combined Operations in STARE.}\label{app:stare_polarity_ablation}
To validate the effectiveness of the token-level reweighting mechanism under the four-quadrant decomposition framework, we ablate all four single-polarity operations (O1, O2, O3, O4) and all four combined operations (C1, C2, C3, C4). Table~\ref{tab:stare_polarity_operation_ablation_aime} reports the accuracy of each operation on AIME24 and AIME25, and Figure~\ref{fig:stare_single_and_combine_four_quant} presents the corresponding policy entropy evolution curves. As shown in Figure~\ref{fig:stare_single_and_combine_four_quant}, GRPO-ds suffers from rapid entropy decay toward zero in the early training phase, whereas all eight STARE variants effectively mitigate entropy collapse. Specifically, O1 and O3 strengthen entropy-increasing signals through weight amplification, O2 and O4 suppress entropy-decreasing signals through weight attenuation, and C1--C4 combine interventions from two quadrants simultaneously. Table~\ref{tab:stare_polarity_operation_ablation_aime} further shows that all STARE variants substantially outperform GRPO-ds, confirming that token-level reweighting is effective across all four entropy-critical quadrants. Among all variants, O1 (amplifying $\mathcal{L}_q^+$) and C2 (amplifying $\mathcal{L}_q^+$ while attenuating $\mathcal{L}_q^-$) achieve the strongest performance, reaching $44.2\%/23.8\%$ and $42.5\%/24.2\%$ on AIME24/AIME25 respectively. We therefore adopt STARE-O1 as the default configuration and STARE-C2 as the dual-sided variant in our work.

\subsection{Detailed Ablation on Key Hyperparameters and Target-Entropy Gating}\label{app:STARE_Key_Hyperparameters_ablation}
To validate the mechanisms underlying the key hyperparameters in STARE ($W$ and $P$) and the target-entropy gating, we conduct ablation experiments on Qwen2.5-Math-7B-Base (Figure~\ref{fig:qwen7b_ablation}). Figure~\ref{fig:qwen7b_ablation}(\subref{fig:7b_ablation_W_values}) examines the reweighting factor $W$ under an open-loop setting (without the target-entropy gate). A minimal perturbation of $W=1.01$ is sufficient to mitigate the entropy decay observed in GRPO-ds; $W \geq 1.05$ yields steady entropy growth, whereas $W \geq 2.0$ drives the entropy upward too aggressively and triggers divergence. This behavior corroborates the near-criticality property (Corollary~\ref{thm:near-criticality}): beyond the critical threshold, $W$ governs the magnitude rather than the direction of the entropy shift. The open-loop regime, however, tends to stabilize the entropy at an excessively high level and thereby induces over-exploration. Once the target-entropy closed-loop gate is introduced with $H_{\text{tgt}}=0.3$, Figure~\ref{fig:qwen7b_ablation}(\subref{fig:7b_ablation_W_with_target_entropy}) shows that values of $W$ across the range $[1.05, 1.5]$ steer the policy entropy into the target band and maintain bounded oscillation around it, which confirms the stability of the closed-loop mechanism and substantially reduces the sensitivity to $W$. Figure~\ref{fig:qwen7b_ablation}(\subref{fig:7b_ablation_top_ratio}) further investigates the high-surprisal token selection ratio $P$. With $W=1.1$ and $H_{\text{tgt}}=0.3$ held fixed, the policy entropy remains within the target band for $P \in [5\%, 20\%]$, and even at $P=40\%$ it is confined to $[0.1, 0.2]$, effectively preventing collapse and indicating that the surprisal-quantile selection admits a broad operating range. Overall, STARE exhibits strong robustness to the configurations of $W$ and $P$, and the target-entropy gate stably constrains the policy entropy. More detailed entropy evolution and performance comparisons are reported in Appendix~\ref{app:STARE_Target_Entropy_Closed_Loop}, Table~\ref{tab:stare_target_entropy_gate_ablation_aime}, and Figure~\ref{fig:STARE_target_entropy}

\begin{figure}[t]
    \centering
    \includegraphics[width=0.75\textwidth]{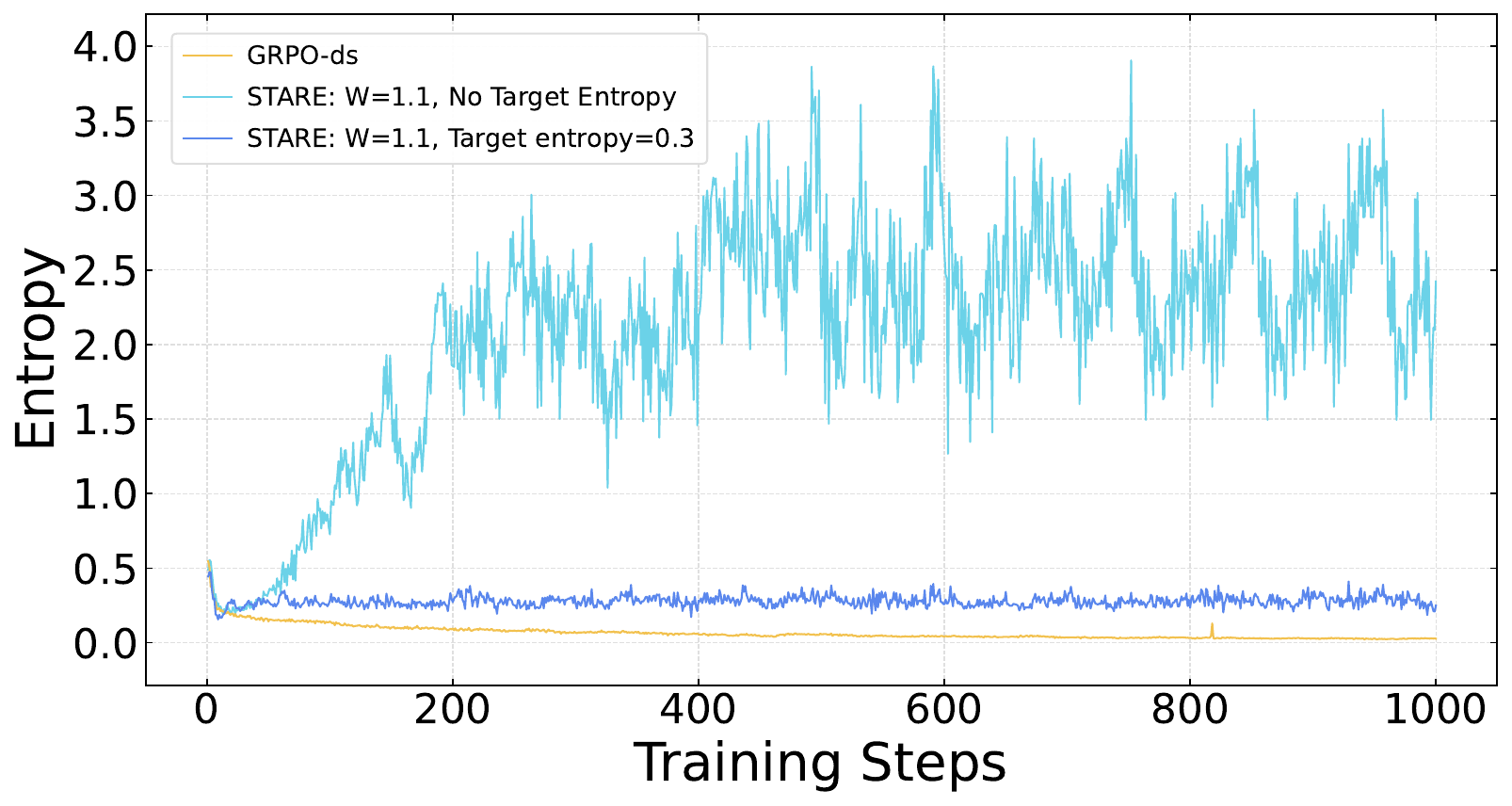}
    \caption{Effect of the target-entropy closed-loop gate on policy entropy regulation for 
STARE ($W{=}1.1$, $P{=}10\%$) on Qwen2.5-Math-7B-Base over 1000 RL training steps. 
The closed-loop gate effectively confines the policy entropy to bounded oscillation around 
the target level $H_{\mathrm{tgt}}{=}0.3$, preventing both entropy collapse and 
over-exploration observed under open-loop reweighting and standard GRPO-ds, respectively.}
    \label{fig:STARE_target_entropy}
\end{figure}

\begin{table}[ht]
\centering
\caption{Ablation of the target-entropy gate for STARE with $W=1.1$ and P=10 on AIME24 and AIME25.}
\label{tab:stare_target_entropy_gate_ablation_aime}
\begin{tabular}{lcc}
\toprule
\textbf{Model} & \textbf{AIME24} & \textbf{AIME25} \\
\midrule
GRPO-ds & 35.2 & 17.3 \\
\midrule
STARE without target-entropy gate & 36.7 & 18.9 \\
STARE with target-entropy gate    & 38.0 & 20.0 \\
\bottomrule
\end{tabular}
\end{table}

\subsection{Effectiveness of Target-Entropy Closed-Loop Gating}\label{app:STARE_Target_Entropy_Closed_Loop}
To validate the necessity of the closed-loop target-entropy gate, we fix $W=1.1$ and $P=10\%$ on Qwen2.5-Math-7B-Base, and compare three configurations over $1000$ RL training steps: GRPO-ds, open-loop STARE, and closed-loop STARE with $H_{\text{tgt}}=0.3$. As shown in Figure~\ref{fig:STARE_target_entropy}, open-loop STARE alleviates entropy collapse but drives the policy entropy upward to an excessively high level, which induces over-exploration; in contrast, closed-loop STARE confines the policy entropy to bounded oscillation around the target band. Table~\ref{tab:stare_target_entropy_gate_ablation_aime} further reports that open-loop STARE attains $36.7\%/18.9\%$ on AIME24/AIME25, already surpassing GRPO-ds, while closed-loop STARE further improves the accuracy to $38.0\%/20.0\%$, with both variants substantially outperforming the GRPO-ds baseline at $35.2\%/17.3\%$. These results confirm that token-level advantage reweighting in STARE mitigates entropy collapse and delivers consistent gains, and that the closed-loop gate stabilizes the policy entropy within a reasonable band to balance exploration and exploitation, thereby yielding additional performance improvements.

\begin{figure}[t]
    \centering
    \includegraphics[width=0.75\textwidth]{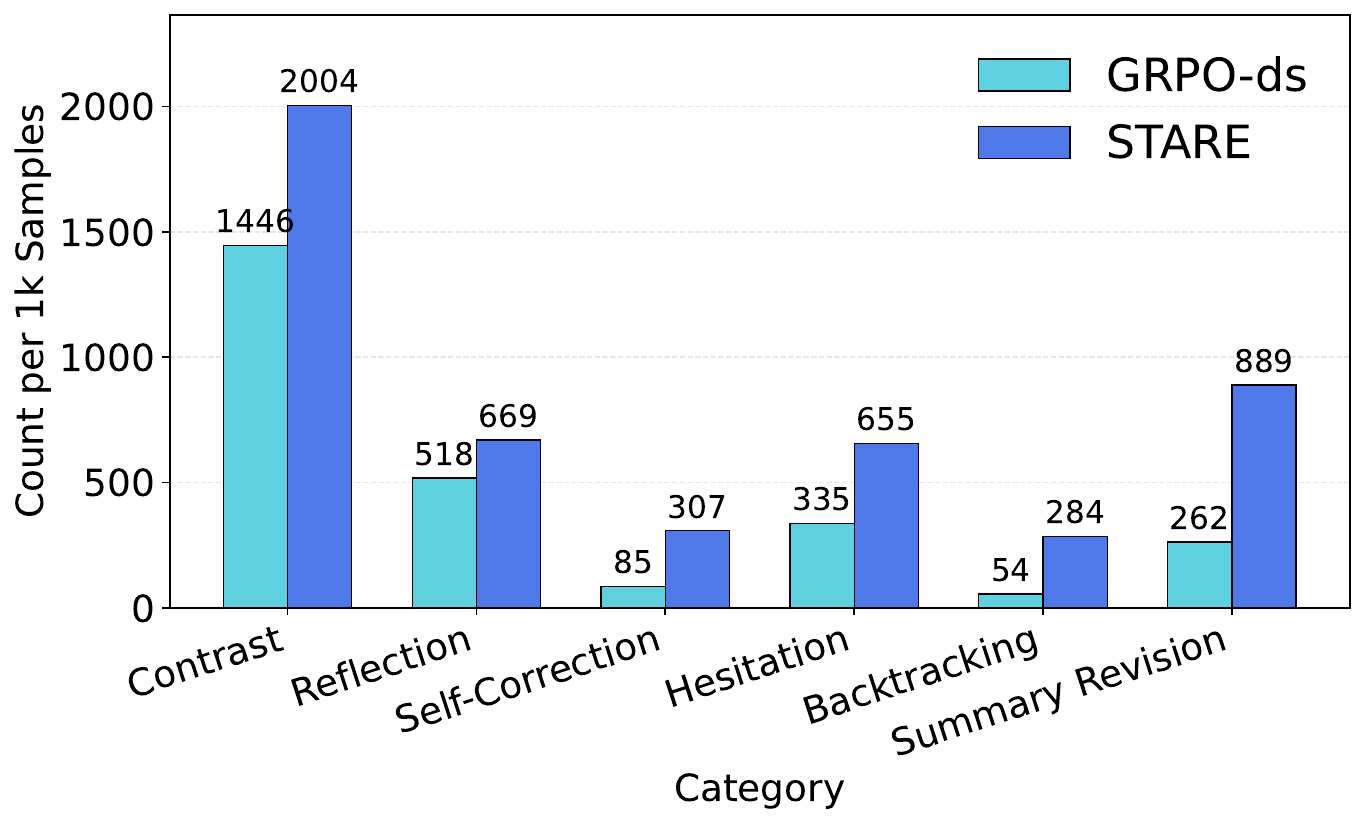}
    \caption{Reflection-related token counts per 1k samples for STARE vs. GRPO-ds on Qwen2.5-32B-Base.}
    \label{fig:reflection_count_bar_chart_optimized}
\end{figure}

\subsection{Details about Emergent Reflection Behaviors.}\label{app:Emergent_Reflection} 
To investigate how STARE elicits deep reasoning, we analyze the reflection behaviors that emerge during RL training. Taking Qwen2.5-32B-Base as a representative example, we randomly sample 50 training steps shared by both STARE and GRPO-ds. This analysis is intended as a coarse, qualitative diagnostic rather than a principled measurement of reasoning depth: we apply heuristic regular-expression matching based on manually-observed patterns to obtain an approximate count of reflection-related tokens, grouped into six categories: contrast (\textit{but, however}), reflection (\textit{wait, reconsider}), self-correction (\textit{made a mistake, correct}), hesitation (\textit{perhaps, possibly, maybe}), backtracking (\textit{recalculate, redo, go back}), and summary revision (\textit{summary, ultimately}). Figure~\ref{fig:reflection_count_bar_chart_optimized} reports reflection-related token counts per 1k samples; STARE substantially exceeds GRPO-ds across all six categories, with markedly larger margins on reflection and self-correction, indicating that the resulting policy retains stronger exploration and greater output diversity. Figure~\ref{fig:merged_token_frequency_wordcloud} further reports the frequency distribution of tokens selected by STARE for advantage amplification; the word cloud reveals that the reweighted tokens concentrate on vocabulary expressing uncertainty and self-correction, such as \textit{should be}, \textit{but}, \textit{instead}, and \textit{verification}, confirming that the batch-internal surprisal-quantile selection effectively identifies rare forking tokens carrying exploratory semantics. Taken together, the emergent growth of reflection tokens and the semantic bias of the reweighted token distribution demonstrate that STARE activates the deep exploration capability of the model and delivers consistent performance gains through token-level credit rebalancing.

\begin{table}[ht]
\centering
\caption{Ablation of fixed and adaptive reweighting coefficients for STARE in 1000 RL training steps on AIME24 and AIME25.}
\label{tab:stare_adaptive_w_ablation_aime}
\begin{tabular}{lcc}
\toprule
\textbf{Model} & \textbf{AIME24} & \textbf{AIME25} \\
\midrule
GRPO-ds & 35.2 & 17.3 \\
STARE: Fixed $W=1.1$ & 38.0 & 20.0 \\
STARE: Adaptive $W$, $W_{\max}=1.5$, $\alpha=0.01$ & 37.7 & 19.5 \\
STARE: Adaptive $W$, $W_{\max}=1.5$, $\alpha=0.02$ & 37.4 & 19.1 \\
STARE: Adaptive $W$, $W_{\max}=2.0$, $\alpha=0.01$ & 36.9 & 18.6 \\
STARE: Adaptive $W$, $W_{\max}=2.0$, $\alpha=0.02$ & 37.1 & 18.1 \\
\bottomrule
\end{tabular}
\end{table}

\subsection{Fixed vs. Adaptive Weights}\label{app:fix_adaptive_weights} 
Table~\ref{tab:stare_adaptive_w_ablation_aime} compares fixed and adaptive weighting schemes on Qwen2.5-Math-7B-Base after 1000 RL training steps. All STARE configurations substantially outperform the GRPO-ds baseline, confirming the robustness of the token-level reweighting mechanism. The fixed weight $W=1.1$ yields the strongest performance, reaching 38.0\% on AIME24 and 20.0\% on AIME25. This outcome aligns with the near-criticality property (Corollary~\ref{thm:near-criticality}): once the reweighting factor exceeds the critical threshold, its specific value primarily controls the magnitude rather than the direction of the per-step entropy shift, so that a moderate fixed value suffices for stable regulation. Among the adaptive variants, the configuration with $W_{\max}=1.5$ and $\alpha=0.01$ performs best and nearly matches the fixed-weight result, whereas enlarging either $W_{\max}$ or $\alpha$ leads to a slight degradation. We therefore adopt the fixed $W=1.1$ as the default configuration, with $W_{\max}=1.5$ and $\alpha=0.01$ recommended when an adaptive schedule is preferred.

\begin{table}[ht]
\centering
\caption{Ablation of target entropy threshlod for STARE in 4000 RL training steps on AIME24 and AIME25.}
\label{tab:stare_target_entropy_ablation_aime}
\begin{tabular}{lcc}
\toprule
\textbf{Model} & \textbf{AIME24} & \textbf{AIME25} \\
\midrule
GRPO-ds & 37.1 & 17.7 \\
STARE: $H_{\mathrm{target}}=0.1$ & 40.4 & 20.5 \\
STARE: $H_{\mathrm{target}}=0.2$ & 43.2 & 23.1 \\
STARE: $H_{\mathrm{target}}=0.3$ & 44.2 & 23.8 \\
STARE: $H_{\mathrm{target}}=0.4$ & 42.8 & 21.6 \\
\bottomrule
\end{tabular}
\end{table}

\begin{figure}[t]
    \centering
    \begin{subfigure}[t]{0.48\textwidth}\centering
    \includegraphics[width=\linewidth]{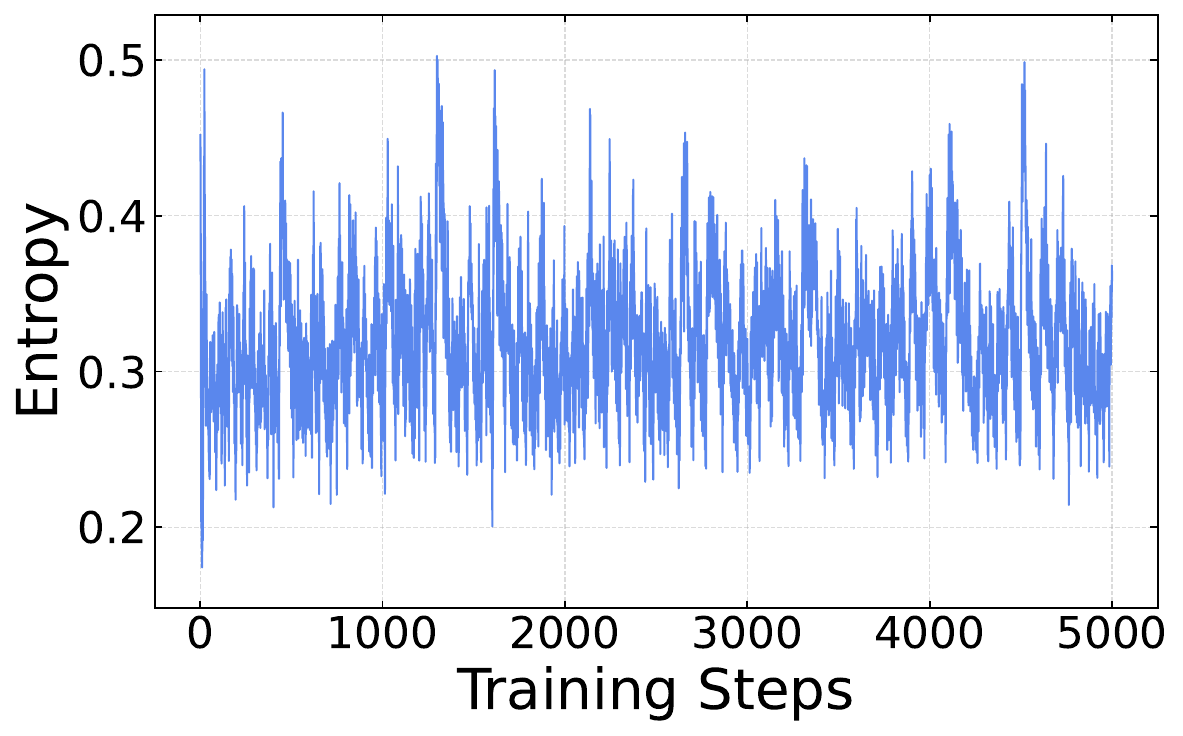}\caption{Policy entropy evolution of STARE under off-policy training (four gradient updates per batch).}\label{fig:qwen7b_text_entropy_off_policy_0_5000}
    \end{subfigure}\hfill
    \begin{subfigure}[t]{0.48\textwidth}\centering
        \includegraphics[width=\linewidth]{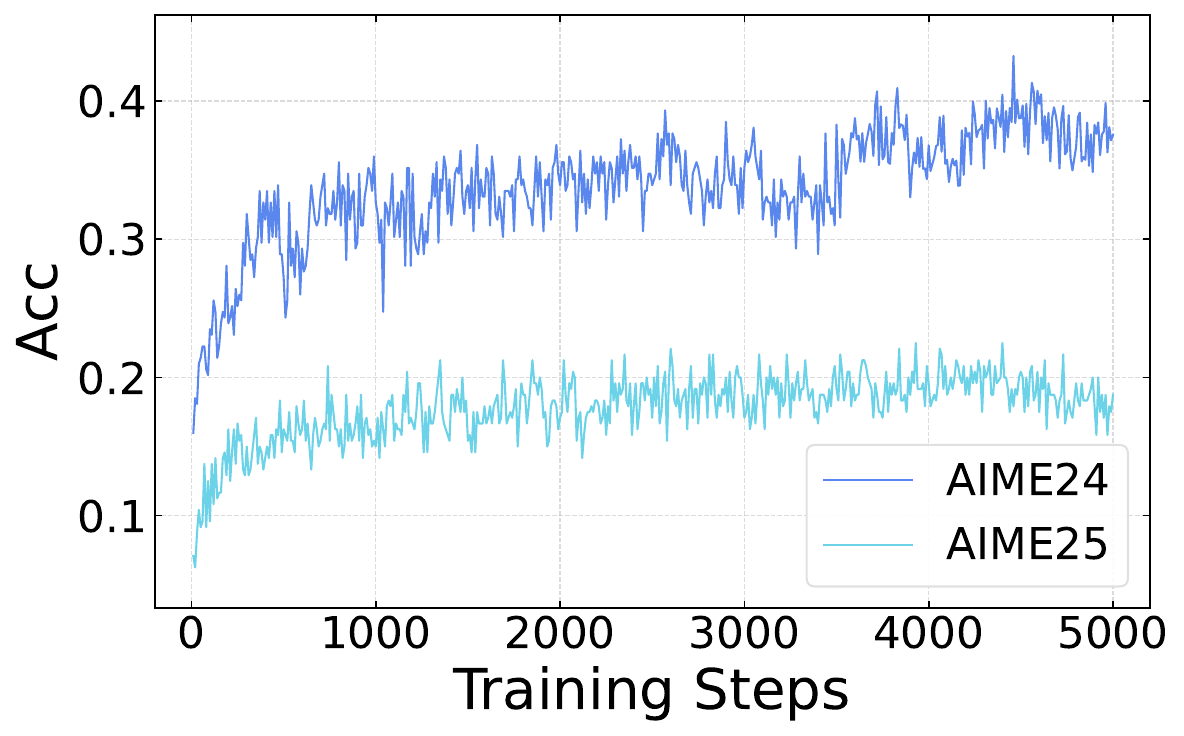}\caption{AIME24 and AIME25 accuracy of STARE under off-policy training.}\label{fig:qwen7b_text_off_policy_aime25_aime24_acc_0_5000_modified}
    \end{subfigure}
    \caption{Off-policy training dynamics of STARE on Qwen2.5-Math-7B-Base with four gradient 
updates per batch over 5{,}000 RL training steps, demonstrating that STARE preserves 
entropy stability and delivers consistent accuracy improvements beyond the on-policy setting.}\label{fig:qwen7b_off_policy_entropy and aime24_25}
\end{figure}

\subsection{Ablation on Target Entropy Threshold}\label{app:Target_Entropy_Threshold} 
Table~\ref{tab:stare_target_entropy_ablation_aime} reports the performance of STARE under different target entropy values $H_{\text{tgt}} \in \{0.1, 0.2, 0.3, 0.4\}$ on Qwen2.5-Math-7B-Base after 4000 RL training steps. All configurations substantially outperform the GRPO-ds baseline, confirming the general effectiveness of the closed-loop entropy regulation mechanism. Performance attains its optimum at $H_{\text{tgt}}=0.3$, reaching 44.2\% on AIME24 and 23.8\% on AIME25, with gains of 7.1\% and 6.1\% over the baseline respectively. The low target entropy restricts the exploration space, whereas the high value easily drives the policy distribution toward over-exploration, and both regimes yield suboptimal outcomes. These results indicate that $H_{\text{tgt}}$ effectively governs the exploration-exploitation balance, and we therefore adopt $H_{\text{tgt}}=0.3$ as the default configuration in the main experiments.

\begin{figure}[t]
    \centering
    \begin{subfigure}[t]{0.48\textwidth}\centering
    \includegraphics[width=\linewidth]{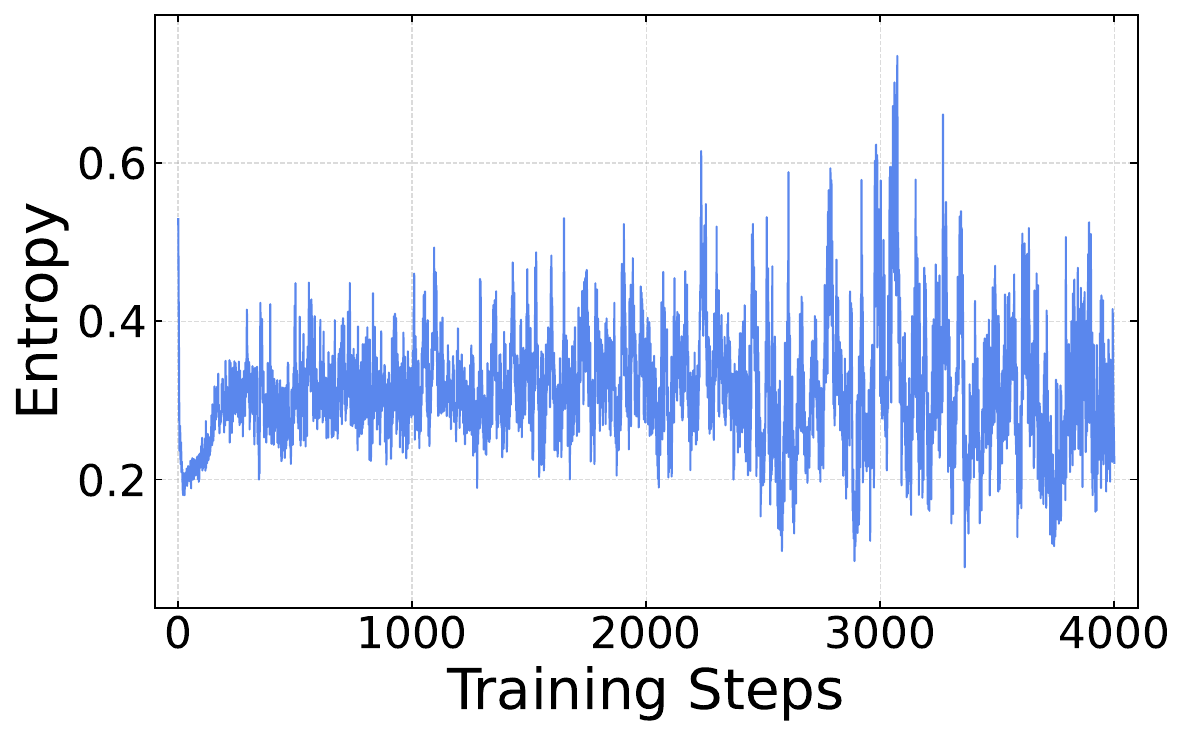}\caption{Policy entropy evolution under fixed-threshold low-probability token reweighting.}\label{fig:qwen7b_0.01_threshold_low_prob_entropy_0_4000}
    \end{subfigure}\hfill
    \begin{subfigure}[t]{0.48\textwidth}\centering
        \includegraphics[width=\linewidth]{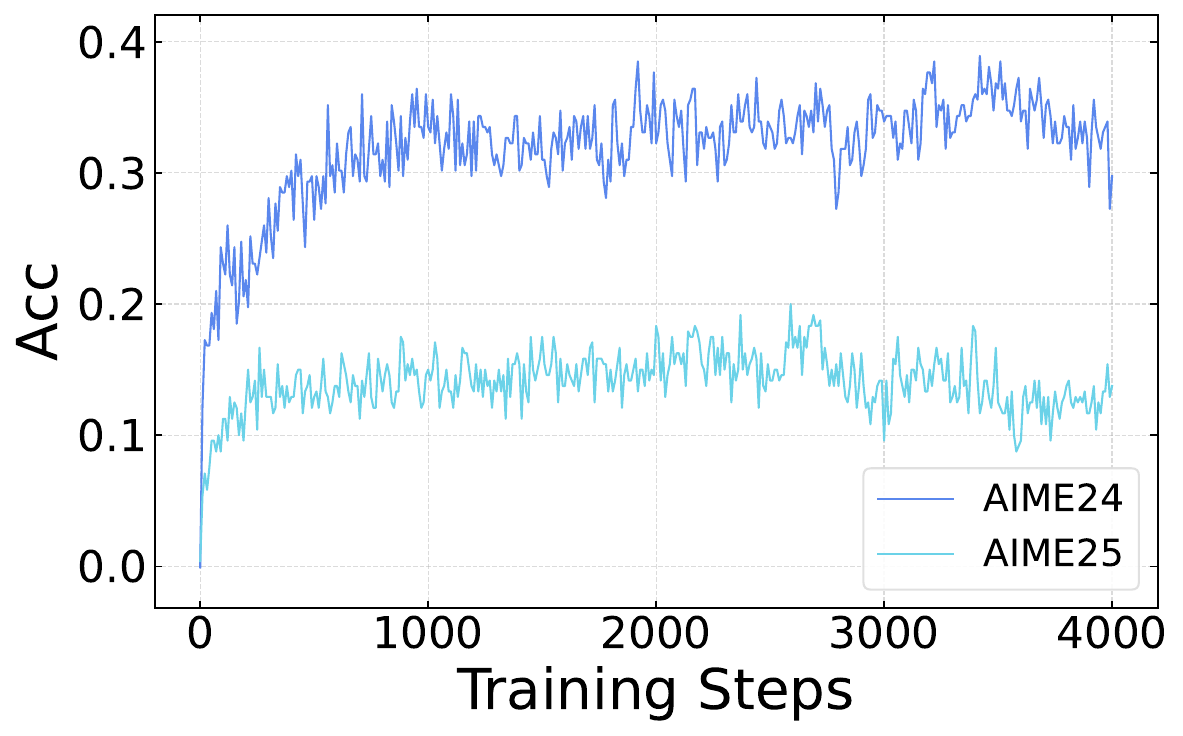}\caption{AIME24 and AIME25 accuracy under fixed-threshold low-probability token reweighting.}\label{fig:qwen7b_0.01_low_prob_aime24_aime25_acc_0_400}
    \end{subfigure}
    \caption{Comparison between fixed-threshold low-probability token reweighting ($p < 0.1$) 
and GRPO-ds on Qwen2.5-Math-7B-Base over 4000 RL training steps, demonstrating the 
inferior entropy regulation and benchmark performance of static probability-based selection 
relative to STARE's batch-internal surprisal-quantile proxy.}\label{fig:qwen7b_0.01_low_prob_entropy and aime24_25}
\end{figure}

\begin{table}[ht]
\centering
\caption{Ablation of target-entropy gate granularity for STARE in 1000 RL training steps on AIME24 and AIME25.}
\label{tab:stare_target_entropy_gate_granularity_aime}
\begin{tabular}{lcc}
\toprule
\textbf{Model} & \textbf{AIME24} & \textbf{AIME25} \\
\midrule
GRPO-ds & 35.2 & 17.3 \\
STARE (token-level target-entropy gate)  & 36.9 & 19.1 \\
STARE (sample-level target-entropy gate) & 37.6 & 19.3 \\
STARE (batch-level target-entropy gate)  & 38.0 & 20.0 \\
\bottomrule
\end{tabular}
\end{table}

\subsection{Ablation on Target-Entropy Gate Granularity}\label{app:Target_Entropy_Threshold_Gate_Granularity} 
On Qwen2.5-Math-7B-Base, we compare three closed-loop gating granularities: a token-level gate that evaluates $H_{i,t} < H_{\text{tgt}}$ at each position, a sample-level gate that conditions activation on the per-sample average entropy $\bar{H}_i$, and a batch-level gate that makes a unified decision based on the batch-average entropy $\bar{H}_k$. As reported in Table~\ref{tab:stare_target_entropy_gate_granularity_aime}, all three granularities substantially outperform the GRPO-ds baseline (35.2\%/17.3\%), confirming the general effectiveness of the closed-loop mechanism. The batch-level gate yields the best performance (38.0\%/20.0\%), surpassing the token-level (36.9\%/19.1\%) and sample-level (37.6\%/19.3\%) variants. Since local entropy estimates at the token and sample levels exhibit higher variance, they tend to induce frequent switching of the gating signal and undermine the stability of the regulation. We therefore adopt the batch-level gate as the default setting in our main experiments.

\begin{table}[ht]
\centering
\caption{Comparison of off-policy and on-policy STARE in 4000 RL training steps on AIME24 and AIME25.}
\label{tab:stare_off_policy_on_policy_aime}
\begin{tabular}{lcc}
\toprule
\textbf{Model} & \textbf{AIME24} & \textbf{AIME25} \\
\midrule
GRPO-ds & 37.1 & 17.7 \\
STARE with off-policy & 43.8 & 22.1 \\
STARE with on-policy  & 44.2 & 23.8 \\
\bottomrule
\end{tabular}
\end{table}

\subsection{Validation of STARE under Off-Policy Training}\label{app:STARE_off_policy} 
To assess the applicability of STARE in the off-policy setting, we conduct 4000 steps of RL training on Qwen2.5-Math-7B-Base with four gradient updates per batch. Figure~\ref{fig:qwen7b_off_policy_entropy and aime24_25} shows that STARE still maintains the policy entropy within the target band, while the accuracy improves steadily throughout training. As reported in Table~\ref{tab:stare_off_policy_on_policy_aime}, off-policy STARE attains $43.8\%/22.1\%$ on AIME24/AIME25, exceeding GRPO-ds ($37.1\%/17.7\%$) by $6.7\%$ and $4.4\%$ respectively, and falling only marginally short of on-policy STARE ($44.2\%/23.8\%$). These results confirm that STARE remains effective under off-policy training, demonstrating the robustness and generality of the proposed method.

\begin{table}[ht]
\centering
\caption{Comparison with Fixed-Threshold Low-Probability Token Reweighting on AIME24 and AIME25.}
\label{tab:fixed_threshold_reweighting_aime}
\begin{tabular}{lcc}
\toprule
\textbf{Method} & \textbf{AIME24} & \textbf{AIME25} \\
\midrule
GRPO-ds & 37.1 & 17.7 \\
Fixed-Threshold Reweighting ($p<0.1$) & 38.9 & 19.7 \\
STARE & 44.2 & 23.8 \\
\bottomrule
\end{tabular}
\end{table}

\subsection{STARE vs. Fixed-Threshold Low-Probability Token Reweighting}\label{app:Fixed_Threshold_logprob} 
STARE identifies entropy-critical tokens through a batch-internal surprisal-quantile proxy. A more direct alternative is to apply uniform weight amplification to all tokens whose sampling probability falls below a fixed threshold such as $p < 0.1$ (Figure ~\ref{fig:qwen7b_0.01_low_prob_entropy and aime24_25}). Under an identical configuration on Qwen2.5-Math-7B-Base with 4000 RL training steps, $W = 1.1$, and $H_{\text{tgt}} = 0.3$, we compare these two token-selection strategies. As reported in Table~\ref{tab:fixed_threshold_reweighting_aime}, fixed-threshold reweighting yields only marginal improvements over GRPO-ds, with gains of 1.8\% on AIME24 and 2.0\% on AIME25, whereas STARE achieves substantially larger gains of 7.1\% and 6.1\% respectively, confirming the superiority of the batch-internal top-$P\%$ surprisal-quantile proxy in identifying entropy-critical tokens. Moreover, the quantile proxy of STARE adaptively tracks the current policy distribution at every training step, and as shown in Figure~\ref{fig:qwen7b_Surprisal_net_entropy}(\subref{fig:qwen7b_surprisal_entropy_below_tgt_s_ratio_0_1000}), its agreement with the theoretical threshold $\mathfrak{s}^*$ rises from approximately 60\% to over 95\% throughout training, ensuring high-precision identification of entropy-critical tokens. These results jointly demonstrate the effectiveness and reliability of the batch-internal surprisal-quantile proxy adopted by STARE.

\newpage
\section{Algorithm: Main STARE Procedure}\label{app:Algorithm_stare} 

\begin{algorithm}[htbp]
\caption{\textsc{STARE-O1}: Surprisal-Guided Token-Level Advantage Reweighting with Fixed Weights and Batch-Level Closed-Loop Target-Entropy Gating}
\label{alg:stare-o1}
\begin{algorithmic}[1]
\Require Initial policy $\pi_{\theta_{0}}$; prompt distribution $\mathcal{D}$; batch size $B$; group size $G$; PPO clip range $\epsilon$;
reweighting factor $W>1$; top-surprisal ratio $P\in(0,1)$; target entropy $H_{\mathrm{tgt}}$; learning rate $\eta$; total steps $K$
\Ensure Trained policy $\pi_{\theta}$
\State Initialize $\theta \leftarrow \theta_{0}$,\ \ $\theta_{\mathrm{old}} \leftarrow \theta_{0}$
\For{$k = 1, 2, \ldots, K$}
    \Statex \quad\quad\textit{\# Stage 1: Rollout and group-normalized advantage estimation}
    \State Sample a batch of prompts $\{x_{i}\}_{i=1}^{B} \sim \mathcal{D}$
    \For{$i = 1, \ldots, B$}
        \State Roll out $G$ responses $\{o_{i,j}\}_{j=1}^{G} \sim \pi_{\theta_{\mathrm{old}}}(\cdot \mid x_{i})$ and obtain verifier rewards $\{r_{i,j}\}_{j=1}^{G}$
        \State $\hat{A}_{i,j} \leftarrow \big(r_{i,j} - \mathrm{mean}(\{r_{i,\cdot}\})\big)\big/\mathrm{std}(\{r_{i,\cdot}\})$ \Comment{shared trajectory-level advantage}
    \EndFor
    \State Flatten all rollouts into a token-level batch with total length $N = \sum_{i,j} T_{i,j}$;\ \ broadcast $\hat{A}_{i}$ to every token $(i,t)$
    \Statex
    \Statex \quad\quad\textit{\# Stage 2: Batch-level closed-loop entropy gating}
    \State $H_{i,t} \leftarrow -\!\sum_{v\in\mathcal{V}}\pi_{\theta}(v\mid c_{i,t})\ln \pi_{\theta}(v\mid c_{i,t})$ for every position $(i,t)$
    \State $\bar{H}_{k} \leftarrow \tfrac{1}{N}\sum_{(i,t)} H_{i,t}$
    \State $g_{k} \leftarrow \mathbb{1}\!\left[\bar{H}_{k} < H_{\mathrm{tgt}}\right]$
       \Comment{activate intervention only when entropy is below target}
    \Statex
    \Statex \quad\quad\textit{\# Stage 3: Surprisal-guided entropy-critical token selection (only when gated on)}
    \If{$g_{k} = 1$}
        \State $s_{i,t} \leftarrow -\ln \pi_{\theta}(o_{i,t}\mid x_{i}, o_{i,<t})$ for every token $(i,t)$
        \State $\mathcal{T}^{+} \leftarrow \{(i,t)\ :\ \hat{A}_{i} > 0\}$
            \Comment{positive-advantage token set}
        \State Sort $\mathcal{T}^{+}$ in \emph{descending} order of $s_{i,t}$; let $K^{+} \leftarrow \lceil P\cdot|\mathcal{T}^{+}|\rceil$
        \State $\mathcal{L}_{q}^{+} \leftarrow$ the first $K^{+}$ entries of the sorted list
            \Comment{top-$P$\% high-surprisal positive-advantage tokens}
    \Else
        \State $\mathcal{L}_{q}^{+} \leftarrow \emptyset$
    \EndIf
    \Statex
    \Statex \quad\quad\textit{\# Stage 4: Token-level advantage reweighting (Variant I, one-sided amplification)}
    \For{each token $(i,t)$ in the batch}
        \State $\omega_{i,t} \leftarrow
        \begin{cases}
            1 + g_{k}\,(W - 1), & (i,t) \in \mathcal{L}_{q}^{+} \\
            1, & \text{otherwise}
        \end{cases}$
    \EndFor
    \Statex
    \Statex \quad\quad\textit{\# Stage 5: STARE clipped-surrogate policy update}
    \State $\rho_{i,t}(\theta) \leftarrow \dfrac{\pi_{\theta}(o_{i,t}\mid x_{i}, o_{i,<t})}{\pi_{\theta_{\mathrm{old}}}(o_{i,t}\mid x_{i}, o_{i,<t})}$
    \State
    $\mathcal{J}_{\textsc{stare}}(\theta) \leftarrow
    \dfrac{1}{N}\!\sum_{(i,t)}\!\omega_{i,t}\,
    \min\!\Big(\rho_{i,t}(\theta)\,\hat{A}_{i},\ \mathrm{clip}\big(\rho_{i,t}(\theta),\,1-\epsilon,\,1+\epsilon\big)\hat{A}_{i}\Big)$
    \State $\theta \leftarrow \theta + \eta\,\nabla_{\theta}\,\mathcal{J}_{\textsc{stare}}(\theta)$
        \Comment{single on-policy gradient step}
    \State $\theta_{\mathrm{old}} \leftarrow \theta$
\EndFor
\State \Return $\pi_{\theta}$
\end{algorithmic}
\end{algorithm}

\newpage

\section{Basic Derivations for Sections~\ref{sec:preliminaries} and~\ref{sec:first-order-gradient}}\label{app:basic-derivations}

\subsection{Softmax Jacobian Derivation}\label{app:softmax-jacobian}

\begin{proposition}[Softmax Jacobian]\label{prop:app-softmax-jacobian}
Let
\[
\pi_\theta(v\mid c)
=
\frac{\exp(z_v)}{\sum_{u\in\mathcal V}\exp(z_u)}.
\]
Then, for any \(v,v'\in\mathcal V\),
\[
\frac{\partial \pi_{v'}}{\partial z_v}
=
\pi_{v'}(\delta_{v'v}-\pi_v).
\]
\end{proposition}

\begin{proof}
Let
\[
Z\triangleq \sum_{u\in\mathcal V}\exp(z_u),
\qquad
\pi_{v'}=\frac{\exp(z_{v'})}{Z}.
\]
Differentiating with respect to \(z_v\) yields
\[
\frac{\partial \pi_{v'}}{\partial z_v}
=
\frac{\delta_{v'v}\exp(z_{v'})\,Z-\exp(z_{v'})\exp(z_v)}{Z^2}.
\]
Factoring out \(\exp(z_{v'})\) from the numerator gives
\[
\frac{\partial \pi_{v'}}{\partial z_v}
=
\frac{\exp(z_{v'})}{Z}
\left(
\delta_{v'v}-\frac{\exp(z_v)}{Z}
\right).
\]
Using
\[
\frac{\exp(z_{v'})}{Z}=\pi_{v'},
\qquad
\frac{\exp(z_v)}{Z}=\pi_v,
\]
we obtain
\[
\frac{\partial \pi_{v'}}{\partial z_v}
=
\pi_{v'}(\delta_{v'v}-\pi_v).
\]
This proves the result.
\end{proof}

\subsection{Token-Level Logit Update in the Unclipped GRPO Regime}\label{app:logit-update}

\begin{proposition}[Token-level logit update in the unclipped GRPO regime]\label{prop:app-logit-update}
In the unclipped regime of GRPO, at a given decoding position, the gradient of the local surrogate objective with respect to the logit vector is aligned with
\[
\hat A\,\nabla_z\log\pi_\theta(a\mid c).
\]
Absorbing the positive proportionality constant into an infinitesimal step size \(\eta>0\) gives the equivalent logit update
\[
\Delta z_v
=
\eta\,\hat A\,(\delta_{va}-\pi_v),
\qquad v\in\mathcal V.
\]
\end{proposition}

\begin{proof}
When clipping is inactive, the gradient direction of the local surrogate objective with respect to the model parameters is proportional to
\[
\hat A\,\nabla_\theta \log \pi_\theta(a\mid c).
\]
By the chain rule, an equivalent ascent direction in logit space is
\[
\frac{\partial}{\partial z_v}
\bigl[\hat A\log\pi_\theta(a\mid c)\bigr]
=
\hat A\,\frac{\partial}{\partial z_v}\log\pi_a.
\]
Moreover,
\[
\frac{\partial}{\partial z_v}\log\pi_a
=
\frac{1}{\pi_a}\frac{\partial \pi_a}{\partial z_v}.
\]
Substituting Proposition~\ref{prop:app-softmax-jacobian} gives
\[
\frac{\partial \pi_a}{\partial z_v}
=
\pi_a(\delta_{va}-\pi_v),
\]
and therefore
\[
\frac{\partial}{\partial z_v}\log\pi_a
=
\frac{1}{\pi_a}\cdot \pi_a(\delta_{va}-\pi_v)
=
\delta_{va}-\pi_v.
\]
Hence
\[
\frac{\partial}{\partial z_v}
\bigl[\hat A\log\pi_a\bigr]
=
\hat A(\delta_{va}-\pi_v).
\]
Taking an infinitesimal gradient ascent step yields
\[
\Delta z_v
=
\eta\,\hat A(\delta_{va}-\pi_v).
\]
This proves the result.
\end{proof}

\subsection{Lemma~\ref{lem:entropy-gradient-logits} (Entropy Gradient with Respect to Logits: Surprisal-Deviation Form)}\label{app:entropy-gradient-logits}

For any \(v\in\mathcal V\),
\[
\frac{\partial H}{\partial z_v}
=
\pi_v(\mathfrak s_v-H),
\qquad
\mathfrak s_v\triangleq -\ln\pi_v.
\]

\begin{proof}
Fix the context \(c\) and suppress it in the notation. By the definition of the Shannon entropy of the conditional next-token distribution,
\[
H=-\sum_{u\in\mathcal V}\pi_u\ln\pi_u.
\]
Differentiating with respect to \(z_v\) gives
\[
\frac{\partial H}{\partial z_v}
=
-\sum_{u\in\mathcal V}
\frac{\partial}{\partial z_v}\bigl(\pi_u\ln\pi_u\bigr).
\]
Applying the product rule to each term,
\[
\frac{\partial}{\partial z_v}(\pi_u\ln\pi_u)
=
\frac{\partial \pi_u}{\partial z_v}\ln\pi_u
+
\pi_u\cdot \frac{1}{\pi_u}\frac{\partial \pi_u}{\partial z_v}
=
\frac{\partial \pi_u}{\partial z_v}(\ln\pi_u+1).
\]
Thus
\[
\frac{\partial H}{\partial z_v}
=
-\sum_{u\in\mathcal V}
\frac{\partial \pi_u}{\partial z_v}(\ln\pi_u+1).
\]
Using Proposition~\ref{prop:app-softmax-jacobian} yields
\[
\frac{\partial H}{\partial z_v}
=
-\sum_{u\in\mathcal V}
\pi_u(\delta_{uv}-\pi_v)(\ln\pi_u+1).
\]
Separating the two terms gives
\[
\frac{\partial H}{\partial z_v}
=
-\sum_{u\in\mathcal V}\pi_u\delta_{uv}(\ln\pi_u+1)
+
\sum_{u\in\mathcal V}\pi_u\pi_v(\ln\pi_u+1).
\]
The first term reduces to
\[
-\sum_{u\in\mathcal V}\pi_u\delta_{uv}(\ln\pi_u+1)
=
-\pi_v(\ln\pi_v+1),
\]
and the second term becomes
\[
\sum_{u\in\mathcal V}\pi_u\pi_v(\ln\pi_u+1)
=
\pi_v\sum_{u\in\mathcal V}\pi_u(\ln\pi_u+1).
\]
Since
\[
\sum_{u\in\mathcal V}\pi_u\ln\pi_u=-H,
\qquad
\sum_{u\in\mathcal V}\pi_u=1,
\]
we have
\[
\sum_{u\in\mathcal V}\pi_u(\ln\pi_u+1)
=
-H+1.
\]
Substituting back yields
\[
\frac{\partial H}{\partial z_v}
=
-\pi_v(\ln\pi_v+1)+\pi_v(-H+1)
=
-\pi_v\ln\pi_v-\pi_v H.
\]
Using \(\mathfrak s_v=-\ln\pi_v\), we obtain
\[
-\pi_v\ln\pi_v=\pi_v\mathfrak s_v,
\]
and therefore
\[
\frac{\partial H}{\partial z_v}
=
\pi_v(\mathfrak s_v-H).
\]
This proves the result.
\end{proof}

\subsection{Theorem~\ref{thm:token-entropy-variation} (Token-Level Entropy Variation)}\label{app:token-entropy-variation}

In the unclipped regime of GRPO, let \(\hat A\) denote the trajectory-level normalized advantage assigned to the current token position, and let \(a\) denote the token sampled at the current decoding position, with conditional probability
\[
p=\pi(a\mid c),
\qquad
\mathfrak s_a=-\ln p.
\]
Define
\[
S_2\triangleq \sum_{v\in\mathcal V}\pi_v^2(\ln\pi_v+H),
\]
and
\[
\Phi(p)\triangleq p(\ln p+H)-S_2.
\]
Then the first-order directional derivative of the conditional policy entropy along the GRPO policy-gradient direction satisfies
\[
\left.\frac{dH}{d\eta}\right|_{\eta=0}
=
- \hat A\,\Phi(p).
\]

\begin{proof}
By Proposition~\ref{prop:app-logit-update}, the logit velocity along the GRPO policy-gradient direction is
\[
\left.\frac{dz_v}{d\eta}\right|_{\eta=0}
=
\hat A(\delta_{va}-\pi_v).
\]
Therefore, by the definition of the directional derivative,
\[
\left.\frac{dH}{d\eta}\right|_{\eta=0}
=
\sum_{v\in\mathcal V}
\frac{\partial H}{\partial z_v}
\left.\frac{dz_v}{d\eta}\right|_{\eta=0}.
\]
Substituting Lemma~\ref{lem:entropy-gradient-logits} gives
\[
\left.\frac{dH}{d\eta}\right|_{\eta=0}
=
\sum_{v\in\mathcal V}
\pi_v(\mathfrak s_v-H)\,\hat A(\delta_{va}-\pi_v).
\]
Factoring out \(\hat A\) yields
\[
\left.\frac{dH}{d\eta}\right|_{\eta=0}
=
\hat A
\sum_{v\in\mathcal V}
\pi_v(\mathfrak s_v-H)(\delta_{va}-\pi_v).
\]
Splitting the sum gives
\[
\left.\frac{dH}{d\eta}\right|_{\eta=0}
=
\hat A
\left[
\sum_{v\in\mathcal V}\pi_v(\mathfrak s_v-H)\delta_{va}
-
\sum_{v\in\mathcal V}\pi_v^2(\mathfrak s_v-H)
\right].
\]
Only the term \(v=a\) contributes to the first sum, so
\[
\sum_{v\in\mathcal V}\pi_v(\mathfrak s_v-H)\delta_{va}
=
\pi_a(\mathfrak s_a-H)
=
p(\mathfrak s_a-H).
\]
Using \(\mathfrak s_v=-\ln\pi_v\), the second term becomes
\[
-\sum_{v\in\mathcal V}\pi_v^2(\mathfrak s_v-H)
=
\sum_{v\in\mathcal V}\pi_v^2(-\mathfrak s_v+H)
=
\sum_{v\in\mathcal V}\pi_v^2(\ln\pi_v+H)
=
S_2.
\]
Hence
\[
\left.\frac{dH}{d\eta}\right|_{\eta=0}
=
\hat A\bigl[p(\mathfrak s_a-H)+S_2\bigr].
\]
Substituting \(\mathfrak s_a=-\ln p\) yields
\[
p(\mathfrak s_a-H)+S_2
=
p(-\ln p-H)+S_2
=
-\bigl[p(\ln p+H)-S_2\bigr]
=
-\Phi(p).
\]
Therefore,
\[
\left.\frac{dH}{d\eta}\right|_{\eta=0}
=
- \hat A\,\Phi(p).
\]
This proves the result.
\end{proof}

\section{Complete Proofs for Sections~\ref{sec:four-quadrant} and~\ref{sec:batch-level} and Near-Criticality Analysis}\label{app:theory-proofs}

\subsection{Positivity of \texorpdfstring{$S_2$}{S2} under Non-Uniform Distributions}\label{app:s2-positive}

\begin{lemma}[Positivity of \texorpdfstring{$S_2$}{S2} under non-uniform distributions]\label{lem:app-s2-positive}
If the conditional distribution $\pi$ is non-uniform, then
\[
S_2=\sum_{v\in\mathcal V}\pi_v^2(\ln\pi_v+H)>0.
\]
If $\pi$ is uniform, then $S_2=0$.
\end{lemma}

\begin{proof}
Let $V\sim\pi$ and define
\[
X\triangleq \pi_V.
\]
Then
\begin{align*}
\mathbb E[\ln X+H]
&= \sum_{v\in\mathcal V}\pi_v(\ln\pi_v+H) \\
&= \sum_v\pi_v\ln\pi_v+H\sum_v\pi_v \\
&= -H+H \\
&=0.
\end{align*}
Moreover,
\begin{align*}
S_2
&= \sum_v\pi_v^2(\ln\pi_v+H) \\
&= \mathbb E\bigl[X(\ln X+H)\bigr].
\end{align*}
Since $\mathbb E[\ln X+H]=0$, we obtain
\begin{align*}
S_2
&= \mathbb E\bigl[X(\ln X+H)\bigr] \\
&\quad - \mathbb E[X]\mathbb E[\ln X+H] \\
&= \operatorname{Cov}(X,\ln X+H) \\
&= \operatorname{Cov}(X,\ln X).
\end{align*}
Using the symmetric form of covariance, and letting $X'$ be an independent copy of $X$, we have
\[
\operatorname{Cov}(X,\ln X)
=
\frac12\mathbb E\bigl[(X-X')(\ln X-\ln X')\bigr].
\]
Therefore,
\[
S_2
=
\frac12\sum_{u,v\in\mathcal V}
\pi_u\pi_v(\pi_u-\pi_v)(\ln\pi_u-\ln\pi_v).
\]
Because the logarithm is strictly increasing, for any $a,b>0$,
\[
(a-b)(\ln a-\ln b)\ge 0,
\]
with equality if and only if $a=b$. Hence every term in the summation above is nonnegative.

If $\pi$ is non-uniform, there exist $u,v$ such that $\pi_u\neq\pi_v$. In the softmax setting considered in this paper, all token probabilities are strictly positive, so the corresponding weight $\pi_u\pi_v$ is also strictly positive, and the associated summand is strictly positive. Thus the full sum is strictly positive, which gives $S_2>0$.

If $\pi$ is uniform, then $\pi_u=\pi_v$ for all $u,v$, so every summand is zero and $S_2=0$.
\end{proof}

\subsection{\texorpdfstring{$H>S_2$}{H > S2} under Non-Degenerate Distributions}\label{app:h-greater-s2}

\begin{lemma}[\texorpdfstring{$H>S_2$}{H > S2} under non-degenerate distributions]\label{lem:app-h-greater-s2}
If the distribution $\pi$ is non-degenerate, namely it is not a point mass on a single token, then
\[
H>S_2.
\]
\end{lemma}

\begin{proof}
By definition,
\begin{align*}
S_2
&= \sum_{v\in\mathcal V}\pi_v^2(\ln\pi_v+H) \\
&= \sum_v\pi_v^2\ln\pi_v
+
H\sum_v\pi_v^2.
\end{align*}
Hence
\[
H-S_2
=
H-\sum_v\pi_v^2\ln\pi_v-H\sum_v\pi_v^2.
\]
Collecting the terms that contain $H$ yields
\[
H-S_2
=
H\Bigl(1-\sum_v\pi_v^2\Bigr)
-
\sum_v\pi_v^2\ln\pi_v.
\]
Equivalently,
\[
H-S_2
=
H\Bigl(1-\sum_v\pi_v^2\Bigr)
+
\sum_v\pi_v^2(-\ln\pi_v).
\]
We now inspect the two terms. Since $\pi$ is non-degenerate, at least two tokens have positive probability. Therefore,
\[
\sum_v\pi_v^2<\Bigl(\sum_v\pi_v\Bigr)^2=1,
\]
which implies
\[
1-\sum_v\pi_v^2>0.
\]
The Shannon entropy of a non-degenerate distribution is strictly positive, so
\[
H\Bigl(1-\sum_v\pi_v^2\Bigr)>0.
\]
For the second term, $0<\pi_v\le 1$ in the softmax setting, and hence
\[
-\ln\pi_v\ge 0.
\]
Since $\pi_v^2\ge 0$, each term satisfies
\[
\pi_v^2(-\ln\pi_v)\ge 0.
\]
Thus
\[
\sum_v\pi_v^2(-\ln\pi_v)\ge 0.
\]
Combining a strictly positive term with a nonnegative term gives
\[
H-S_2>0,
\]
and therefore
\[
H>S_2.
\]
\end{proof}

\subsection{Proposition~\ref{prop:critical-surprisal-threshold} (Uniqueness of the Critical Surprisal Threshold)}\label{app:critical-surprisal-threshold}

For any non-uniform and non-degenerate distribution $\pi$, there exists a unique
\[
p^*\in (e^{-H},1),
\qquad
\mathfrak{s}^*\triangleq -\ln p^*\in(0,H),
\]
such that
\[
\Phi(p^*)=0,
\]
and
\[
\Phi(p)>0
\iff
p>p^*
\iff
\mathfrak{s}_a<\mathfrak{s}^*.
\]

\begin{proof}
By definition,
\[
\Phi(p)=p(\ln p+H)-S_2,\qquad p\in(0,1].
\]
We first determine the endpoint signs. As $p\to 0^+$, the standard limit $p\ln p\to 0$ gives
\[
\Phi(0^+)=-S_2<0,
\]
where the inequality follows from Lemma~\ref{lem:app-s2-positive}. At $p=1$,
\[
\Phi(1)=H-S_2>0,
\]
where the inequality follows from Lemma~\ref{lem:app-h-greater-s2}.

Next, we show that $\Phi$ is strictly increasing on $[e^{-H},1]$. Differentiating gives
\begin{align*}
\Phi'(p)
&= \frac{d}{dp}\bigl[p(\ln p+H)-S_2\bigr] \\
&= \ln p+H+1.
\end{align*}
For any $p\in[e^{-H},1]$, we have $\ln p\ge -H$, and hence
\[
\Phi'(p)\ge -H+H+1=1>0.
\]
Therefore, $\Phi$ is strictly increasing on $[e^{-H},1]$.

Moreover,
\begin{align*}
\Phi(e^{-H})
&= e^{-H}(-H+H)-S_2 \\
&= -S_2<0,
\end{align*}
whereas
\[
\Phi(1)=H-S_2>0.
\]
By continuity and the intermediate value theorem, there exists at least one
\[
p^*\in(e^{-H},1)
\]
such that $\Phi(p^*)=0$. Since $\Phi$ is strictly increasing on this interval, the zero is unique.

Define
\[
\mathfrak{s}^*\triangleq -\ln p^*.
\]
Because $p^*\in(e^{-H},1)$, applying the negative logarithm gives
\[
0<-\ln p^*<H,
\]
or equivalently,
\[
\mathfrak{s}^*\in(0,H).
\]
Finally, the strict monotonicity of $\Phi$ implies
\[
p>p^*
\iff
\Phi(p)>0.
\]
Since $-\ln(\cdot)$ is strictly decreasing on $(0,1]$,
\[
p>p^*
\iff
-\ln p<-\ln p^*
\iff
\mathfrak{s}_a<\mathfrak{s}^*.
\]
Combining the two equivalences yields
\[
\Phi(p)>0
\iff
p>p^*
\iff
\mathfrak{s}_a<\mathfrak{s}^*.
\]
\end{proof}

\subsection{Corollary~\ref{cor:four-quadrant} (Four-Quadrant Decomposition)}\label{app:four-quadrant}

The sign of the first-order entropy variation at a single token position is jointly determined by $(\operatorname{sign}\hat A,\mathbf 1[\mathfrak{s}_a<\mathfrak{s}^*])$. When $\hat A>0$ and $\mathfrak{s}_a<\mathfrak{s}^*$, entropy decreases. When $\hat A>0$ and $\mathfrak{s}_a>\mathfrak{s}^*$, entropy increases. When $\hat A<0$ and $\mathfrak{s}_a<\mathfrak{s}^*$, entropy increases. When $\hat A<0$ and $\mathfrak{s}_a>\mathfrak{s}^*$, entropy decreases.

\begin{proof}
By Theorem~\ref{thm:token-entropy-variation},
\[
\left.\frac{dH}{d\eta}\right|_{\eta=0}
=
-\hat A\Phi(p).
\]
Thus the sign of the first-order entropy variation is exactly the sign of $-\hat A\Phi(p)$. Proposition~\ref{prop:critical-surprisal-threshold} gives
\[
\mathfrak{s}_a<\mathfrak{s}^*
\iff
\Phi(p)>0,
\qquad
\mathfrak{s}_a>\mathfrak{s}^*
\iff
\Phi(p)<0.
\]
If $\hat A>0$ and $\mathfrak{s}_a<\mathfrak{s}^*$, then $\Phi(p)>0$, so $-\hat A\Phi(p)<0$ and entropy decreases. If $\hat A>0$ and $\mathfrak{s}_a>\mathfrak{s}^*$, then $\Phi(p)<0$, so $-\hat A\Phi(p)>0$ and entropy increases. If $\hat A<0$ and $\mathfrak{s}_a<\mathfrak{s}^*$, then $\Phi(p)>0$, so $-\hat A\Phi(p)>0$ and entropy increases. If $\hat A<0$ and $\mathfrak{s}_a>\mathfrak{s}^*$, then $\Phi(p)<0$, so $-\hat A\Phi(p)<0$ and entropy decreases. At the boundary $\mathfrak{s}_a=\mathfrak{s}^*$, the first-order entropy variation is zero.
\end{proof}

\subsection{Asymmetric Entropy Contributions under Shared Trajectory-Level Advantages}\label{app:asymmetric-entropy}

This subsection provides the detailed quantitative analysis of the asymmetric entropy contributions summarized in Section~\ref{sec:four-quadrant}.

Consider the subset of trajectories assigned positive advantages ($\hat{A}_i > 0$). Because rollouts are sampled from the current policy $\pi_\theta$, the probability of drawing a low-surprisal token at each decoding position is higher than that of drawing a high-surprisal token. Within positive-advantage trajectories, entropy-decreasing tokens (low surprisal, $\Phi > 0$) therefore constitute the statistical majority by sampling frequency, while entropy-increasing tokens (high surprisal, $\Phi < 0$) remain a statistical minority. Denoting $p_{i,t} = \pi_\theta(o_{i,t} \mid x_i, o_{i,<t})$ and $\mathfrak{s}_{i,t} = -\ln p_{i,t}$, the net first-order entropy contribution from positive-advantage samples decomposes as
\[
\left.\frac{d\bar{H}^{+}}{d\eta}\right|_{\eta=0}
=
\frac{1}{N}
\bigg[
\underbrace{
-\!\!\!\sum_{\substack{i,t:\,\hat{A}_i>0 \\
\mathfrak{s}_{i,t}<\mathfrak{s}^*_{i,t}}}
\hat{A}_i\,\Phi_{i,t}
}_{\text{entropy-decreasing majority}\;(<0)}
\;+\;
\underbrace{
\sum_{\substack{i,t:\,\hat{A}_i>0 \\
\mathfrak{s}_{i,t}>\mathfrak{s}^*_{i,t}}}
\hat{A}_i\,|\Phi_{i,t}|
}_{\text{entropy-increasing minority}\;(>0)}
\bigg].
\]

During GRPO training, all tokens within a trajectory share a single trajectory-level advantage $\hat{A}_i$, leaving the algorithm unable to distinguish the opposing entropy effects of these two token categories. This observation reveals a fundamental gradient-level mechanism underlying entropy collapse in GRPO. Under the shared trajectory-level advantage, the reinforced low-surprisal majority systematically drives the distribution toward concentration, with entropy-decreasing contributions dominating in expectation. In contrast, the high-surprisal minority that could preserve diversity contributes limited entropy-increasing effects. Upon aggregation over the batch, the net outcome is a systematic reduction in entropy. An analogous but mirror-image asymmetry governs the negative-advantage subset.

\subsection{Theorem~\ref{thm:entropy-neutrality} (Entropy Neutrality Identity)}\label{app:entropy-neutrality}

For any conditional distribution $\pi$,
\[
\mathbb E_{a\sim\pi}[\Phi(a)]
=
\sum_{v\in\mathcal V}\pi_v\Phi(\pi_v)
=
0.
\]

\begin{proof}
Here $\Phi(a)$ denotes the value obtained by first sampling $a\sim\pi$ and then evaluating $\Phi(p)$ at $p=\pi(a)$. Therefore,
\[
\mathbb E_{a\sim\pi}[\Phi(a)]
=
\sum_{v\in\mathcal V}\pi_v\Phi(\pi_v).
\]
Substituting the definition of $\Phi$ gives
\[
\sum_{v\in\mathcal V}\pi_v\Phi(\pi_v)
=
\sum_v \pi_v\bigl[\pi_v(\ln\pi_v+H)-S_2\bigr].
\]
Expanding the right-hand side yields
\[
\sum_v \pi_v^2(\ln\pi_v+H)-S_2\sum_v\pi_v.
\]
By the definition of $S_2$,
\[
\sum_v \pi_v^2(\ln\pi_v+H)=S_2,
\]
and by normalization,
\[
\sum_v\pi_v=1.
\]
Thus
\[
\sum_{v\in\mathcal V}\pi_v\Phi(\pi_v)
=
S_2-S_2=0,
\]
which proves
\[
\mathbb E_{a\sim\pi}[\Phi(a)]=0.
\]
\end{proof}

\subsection{Proposition~\ref{prop:entropy-gradient-reweighting} (Entropy Gradient under Token-Level Reweighting)}\label{app:entropy-gradient-reweighting}

Let
\[
\mathcal L^{+}
=
\{(i,t):\hat A_i>0,\ \mathfrak{s}_{i,t}>\mathfrak{s}_{i,t}^{*}\}.
\]
Suppose the effective advantage of every token in $\mathcal L^{+}$ is multiplied by a factor $W\ge 1$, while all other token positions keep unit weight. Then the first-order variation of the batch-averaged entropy satisfies
\[
\left.\frac{d\bar H}{d\eta}\right|_{W}
=
-\frac{1}{N}\Big[\Lambda-(W-1)\Gamma\Big],
\]
where
\[
\Lambda\triangleq \sum_{i,t}\hat A_i\Phi_{i,t},
\qquad
\Gamma\triangleq \sum_{(i,t)\in\mathcal L^{+}}\hat A_i|\Phi_{i,t}|>0.
\]
The critical weight is
\[
W^*=1+\frac{\Lambda}{\Gamma}.
\]
The batch-level net entropy variation is positive when $W>W^*$ and negative when $W<W^*$.

\begin{proof}
Under token-level reweighting, the policy-gradient contribution at position $(i,t)$ is linearly scaled by a positive weight $\omega_{i,t}$. Since the first-order entropy derivative in Theorem~\ref{thm:token-entropy-variation} is linear in the policy-gradient direction, the corresponding entropy contribution becomes
\[
-\omega_{i,t}\hat A_i\Phi_{i,t}.
\]
Hence the batch-averaged first-order entropy variation is
\[
\left.\frac{d\bar H}{d\eta}\right|_{\omega}
=
-\frac{1}{N}\sum_{i,t}\omega_{i,t}\hat A_i\Phi_{i,t}.
\]
In the present reweighting scheme,
\[
\omega_{i,t}
=
\begin{cases}
W, & (i,t)\in\mathcal L^{+},\\
1, & \text{otherwise}.
\end{cases}
\]
Therefore,
\[
\left.\frac{d\bar H}{d\eta}\right|_{W}
=
-\frac{1}{N}
\left[
\sum_{(i,t)\notin\mathcal L^{+}}\hat A_i\Phi_{i,t}
+
W\sum_{(i,t)\in\mathcal L^{+}}\hat A_i\Phi_{i,t}
\right].
\]
Equivalently,
\[
\left.\frac{d\bar H}{d\eta}\right|_{W}
=
-\frac{1}{N}
\left[
\sum_{i,t}\hat A_i\Phi_{i,t}
+
(W-1)\sum_{(i,t)\in\mathcal L^{+}}\hat A_i\Phi_{i,t}
\right].
\]
Define
\[
\Lambda\triangleq \sum_{i,t}\hat A_i\Phi_{i,t}.
\]
For any $(i,t)\in\mathcal L^{+}$, we have $\hat A_i>0$ and $\mathfrak{s}_{i,t}>\mathfrak{s}_{i,t}^{*}$. By Proposition~\ref{prop:critical-surprisal-threshold}, $\Phi_{i,t}<0$. Hence
\[
\hat A_i\Phi_{i,t}
=
-\hat A_i|\Phi_{i,t}|.
\]
It follows that
\[
\sum_{(i,t)\in\mathcal L^{+}}\hat A_i\Phi_{i,t}
=
-\sum_{(i,t)\in\mathcal L^{+}}\hat A_i|\Phi_{i,t}|.
\]
Define
\[
\Gamma\triangleq \sum_{(i,t)\in\mathcal L^{+}}\hat A_i|\Phi_{i,t}|.
\]
When $\mathcal L^{+}$ is nonempty, $\Gamma>0$ because $\hat A_i>0$ and $|\Phi_{i,t}|>0$ on this set. The expression above becomes
\[
\left.\frac{d\bar H}{d\eta}\right|_{W}
=
-\frac{1}{N}\Big[\Lambda-(W-1)\Gamma\Big].
\]
Let
\[
W^*\triangleq 1+\frac{\Lambda}{\Gamma}.
\]
Then
\[
\Lambda-(W-1)\Gamma
=
\Gamma(W^*-W).
\]
Since $\Gamma>0$,
\begin{align*}
\operatorname{sign}\!\left(\left.\frac{d\bar H}{d\eta}\right|_{W}\right)
&=
-\operatorname{sign}(W^*-W) \\
&=
\operatorname{sign}(W-W^*).
\end{align*}
Thus $\left.d\bar H/d\eta\right|_{W}>0$ when $W>W^*$, $\left.d\bar H/d\eta\right|_{W}<0$ when $W<W^*$, and the first-order entropy variation is zero when $W=W^*$.
\end{proof}

\subsection{Two Auxiliary Bounds: Lower Bound in the High-Surprisal Region and Upper Bound in the Low-Surprisal Region}\label{app:auxiliary-bounds}

\subsubsection{Pointwise Lower Bound in the High-Surprisal Region}\label{app:phi-lower-high-surprisal}

\begin{lemma}[Pointwise lower bound in the high-surprisal region]\label{lem:app-phi-lower-high-surprisal}
If a token satisfies
\[
\mathfrak{s}\ge H,
\]
then
\[
|\Phi(p)|\ge S_2.
\]
\end{lemma}

\begin{proof}
The condition $\mathfrak{s}\ge H$ is equivalent to
\[
-\ln p\ge H
\iff
\ln p+H\le 0.
\]
Since $p>0$, multiplying both sides by $p$ gives
\[
p(\ln p+H)\le 0.
\]
Therefore,
\[
\Phi(p)=p(\ln p+H)-S_2\le -S_2<0.
\]
Consequently,
\[
|\Phi(p)|=-\Phi(p)=S_2-p(\ln p+H)\ge S_2.
\]
\end{proof}

\subsubsection{Pointwise Upper Bound in the Low-Surprisal Region}\label{app:phi-upper-low-surprisal}

\begin{lemma}[Pointwise upper bound in the low-surprisal region]\label{lem:app-phi-upper-low-surprisal}
If a token satisfies
\[
\mathfrak{s}<\mathfrak{s}^*,
\]
then
\[
|\Phi(p)|\le H-S_2.
\]
\end{lemma}

\begin{proof}
By Proposition~\ref{prop:critical-surprisal-threshold},
\[
\mathfrak{s}<\mathfrak{s}^*
\iff
\Phi(p)>0.
\]
Therefore,
\[
|\Phi(p)|=\Phi(p)=p(\ln p+H)-S_2.
\]
Since $0<p\le 1$, we have $\ln p\le 0$, and hence
\[
p\ln p\le 0.
\]
Also, $pH\le H$. Thus
\[
p(\ln p+H)=p\ln p+pH\le H.
\]
It follows that
\[
|\Phi(p)|\le H-S_2.
\]
\end{proof}

\subsection{Statistical Assumptions for the Near-Criticality Analysis}\label{app:near-criticality-assumptions}

\subsubsection{Single-Token Credit Dilution}\label{app:assumption-credit-dilution}

\begin{assumption}[Single-token credit dilution]\label{assump:credit-dilution}
For long sequences and sufficiently large batches, define the conditional advantage function at any position $(i,t)$ as
\[
g_{i,t}(v)\triangleq \mathbb E[\hat A_i\mid c_{i,t},\ o_{i,t}=v].
\]
There exists a constant $C_g>0$, independent of the sequence length $T$, such that for every position and every $u,v\in\mathcal V$,
\[
|g_{i,t}(u)-g_{i,t}(v)|\le \frac{C_g}{T}.
\]
This assumption states that, along a trajectory of length $T$, the marginal effect of any single token on the trajectory-level advantage is diluted at rate $1/T$.
\end{assumption}

\subsubsection{Non-Degenerate Absolute Scale of the Advantage}\label{app:assumption-advantage-scale}

\begin{assumption}[Non-degenerate absolute scale of the advantage]\label{assump:advantage-scale}
There exists a constant $a_->0$ such that, for every position $(i,t)$ and every candidate token $v$,
\[
\mathbb E\bigl[|\hat A_i|\mid c_{i,t},\ o_{i,t}=v\bigr]\ge a_-.
\]
This assumption states that the conditional absolute scale of the normalized advantage remains $O(1)$ rather than degenerating to zero.
\end{assumption}

\subsubsection{Non-Vanishing Mass of a Strong High-Surprisal Positive-Advantage Subset}\label{app:assumption-high-surprisal-mass}

\begin{assumption}[Non-vanishing mass of a strong high-surprisal positive-advantage subset]\label{assump:high-surprisal-mass}
Define
\[
\mathcal H^{+}\triangleq \{(i,t):\hat A_i>0,\ \mathfrak{s}_{i,t}\ge H_{i,t}\}.
\]
There exist constants $\rho_H>0$ and $c_H>0$ such that, for sufficiently large batches, with high probability,
\[
|\mathcal H^{+}|\ge \rho_H N,
\]
and
\[
\frac{1}{|\mathcal H^{+}|}
\sum_{(i,t)\in\mathcal H^{+}} \hat A_i S_{2,i,t}
\ge c_H.
\]
\end{assumption}

\subsubsection{Linear Scale of Total Absolute Entropy Sensitivity}\label{app:assumption-linear-absolute-scale}

\begin{assumption}[Linear scale of total absolute entropy sensitivity]\label{assump:linear-absolute-scale}
Let
\[
\Sigma_{\mathrm{abs}}\triangleq \sum_{i,t}|\hat A_i|\,|\Phi_{i,t}|.
\]
There exists a constant $c_\Sigma<\infty$ such that, for sufficiently large batches, with high probability,
\[
\Sigma_{\mathrm{abs}}\le c_\Sigma N.
\]
\end{assumption}

\subsection{\texorpdfstring{$\Lambda$}{Lambda} Is a Weak Residual}\label{app:lambda-weak-residual}

\begin{lemma}[\texorpdfstring{$\Lambda$}{Lambda} is a weak residual]\label{lem:app-lambda-weak-residual}
Under Assumptions~\ref{assump:credit-dilution} and~\ref{assump:advantage-scale}, and under standard large-batch concentration with uniformly bounded second moments,
\[
\frac{|\Lambda|}{\Sigma_{\mathrm{abs}}}=O(T^{-1}),
\qquad
\Lambda\triangleq \sum_{i,t}\hat A_i\Phi_{i,t}.
\]
\end{lemma}

\begin{proof}
The proof first establishes the $O(T^{-1})$ scaling at the population level and then transfers it to empirical batch quantities by concentration.

Fix a position $(i,t)$ and condition on $c_{i,t}=c$. Define
\[
\Phi_c(v)
\triangleq
\pi_\theta(v\mid c)\bigl(\ln\pi_\theta(v\mid c)+H(c)\bigr)-S_2(c).
\]
When the sampled token is $o_{i,t}$,
\[
\Phi_{i,t}=\Phi_c(o_{i,t}).
\]
By the law of total expectation,
\[
\mathbb E[\hat A_i\Phi_{i,t}]
=
\mathbb E\!\left[
\mathbb E[\hat A_i\Phi_{i,t}\mid c_{i,t}]
\right].
\]
Given $c_{i,t}=c$, the token $o_{i,t}$ is sampled from $\pi_\theta(\cdot\mid c)$. Thus
\[
\mathbb E[\hat A_i\Phi_{i,t}\mid c_{i,t}=c]
=
\sum_{v\in\mathcal V}
\pi_\theta(v\mid c)
\mathbb E[\hat A_i\mid c_{i,t}=c,\ o_{i,t}=v]
\Phi_c(v).
\]
Let
\[
g_c(v)\triangleq \mathbb E[\hat A_i\mid c_{i,t}=c,\ o_{i,t}=v].
\]
Then
\[
\mathbb E[\hat A_i\Phi_{i,t}\mid c_{i,t}=c]
=
\sum_v \pi_\theta(v\mid c)g_c(v)\Phi_c(v).
\]
We now use the Entropy Neutrality Identity to remove the mean component of $g_c$. Define
\[
\bar g_c\triangleq \sum_v \pi_\theta(v\mid c)g_c(v).
\]
By Theorem~\ref{thm:entropy-neutrality},
\[
\sum_v \pi_\theta(v\mid c)\Phi_c(v)=0.
\]
Therefore,
\[
\sum_v \pi_\theta(v\mid c)g_c(v)\Phi_c(v)
=
\sum_v \pi_\theta(v\mid c)\bigl(g_c(v)-\bar g_c\bigr)\Phi_c(v),
\]
which implies
\[
\mathbb E[\hat A_i\Phi_{i,t}\mid c_{i,t}=c]
=
\sum_v \pi_\theta(v\mid c)\bigl(g_c(v)-\bar g_c\bigr)\Phi_c(v).
\]
Since $\bar g_c$ is a $\pi_\theta(\cdot\mid c)$-weighted average of $\{g_c(v)\}_{v\in\mathcal V}$, for every $v$,
\[
|g_c(v)-\bar g_c|
\le
\max_{u,v'}|g_c(u)-g_c(v')|.
\]
Assumption~\ref{assump:credit-dilution} gives
\[
|g_c(v)-\bar g_c|
\le
\frac{C_g}{T}.
\]
Consequently,
\[
\left|
\mathbb E[\hat A_i\Phi_{i,t}\mid c_{i,t}=c]
\right|
\le
\frac{C_g}{T}
\sum_v\pi_\theta(v\mid c)|\Phi_c(v)|.
\]
Define
\[
\Psi(c)\triangleq \sum_v\pi_\theta(v\mid c)|\Phi_c(v)|.
\]
Then
\[
\left|
\mathbb E[\hat A_i\Phi_{i,t}\mid c_{i,t}=c]
\right|
\le
\frac{C_g}{T}\Psi(c).
\]
Summing over all token positions yields
\[
|\mathbb E[\Lambda]|
=
\left|
\sum_{i,t}\mathbb E[\hat A_i\Phi_{i,t}]
\right|
\le
\sum_{i,t}
\left|
\mathbb E[\hat A_i\Phi_{i,t}]
\right|
\le
\frac{C_g}{T}
\sum_{i,t}\mathbb E[\Psi(c_{i,t})].
\]
We next lower bound $\mathbb E[\Sigma_{\mathrm{abs}}]$. Conditional on $c_{i,t}=c$,
\[
\mathbb E\bigl[|\hat A_i|\,|\Phi_{i,t}|\mid c_{i,t}=c\bigr]
=
\sum_v \pi_\theta(v\mid c)
\mathbb E\bigl[|\hat A_i|\mid c_{i,t}=c,\ o_{i,t}=v\bigr]
|\Phi_c(v)|.
\]
By Assumption~\ref{assump:advantage-scale},
\[
\mathbb E\bigl[|\hat A_i|\mid c_{i,t}=c,\ o_{i,t}=v\bigr]\ge a_-.
\]
Therefore,
\[
\mathbb E\bigl[|\hat A_i|\,|\Phi_{i,t}|\mid c_{i,t}=c\bigr]
\ge
a_-\sum_v\pi_\theta(v\mid c)|\Phi_c(v)|
=
a_-\Psi(c).
\]
Summing over positions and taking expectations gives
\[
\mathbb E[\Sigma_{\mathrm{abs}}]
=
\sum_{i,t}
\mathbb E\bigl[|\hat A_i|\,|\Phi_{i,t}|\bigr]
\ge
a_-
\sum_{i,t}\mathbb E[\Psi(c_{i,t})].
\]
Combining the upper bound on $|\mathbb E[\Lambda]|$ and the lower bound on $\mathbb E[\Sigma_{\mathrm{abs}}]$ gives
\[
\frac{|\mathbb E[\Lambda]|}{\mathbb E[\Sigma_{\mathrm{abs}}]}
\le
\frac{C_g}{a_-T}.
\]
Thus, at the population level,
\[
\frac{|\mathbb E[\Lambda]|}{\mathbb E[\Sigma_{\mathrm{abs}}]}=O(T^{-1}).
\]
For sufficiently large batches, and assuming uniformly bounded second moments of the corresponding summands, $\Lambda$ and $\Sigma_{\mathrm{abs}}$ concentrate around their population values at a scale that does not change the dependence on $T$. Hence, with high probability,
\[
\frac{|\Lambda|}{\Sigma_{\mathrm{abs}}}=O(T^{-1}).
\]
\end{proof}

\subsection{Theorem~\ref{thm:near-criticality} (Near-Criticality)}\label{app:near-criticality}

If Assumptions~\ref{assump:credit-dilution}--\ref{assump:linear-absolute-scale} hold and both the sequence length $T$ and the batch size are sufficiently large, then
\[
W^*-1=\frac{\Lambda}{\Gamma}=O(T^{-1}).
\]

\begin{proof}
By Proposition~\ref{prop:entropy-gradient-reweighting},
\[
W^*-1=\frac{\Lambda}{\Gamma}.
\]
It remains to control the numerator and denominator.

For the numerator, Lemma~\ref{lem:app-lambda-weak-residual} gives
\[
\frac{|\Lambda|}{\Sigma_{\mathrm{abs}}}=O(T^{-1}).
\]
For the denominator, we show that $\Gamma$ is lower bounded by a constant-order fraction of the total absolute entropy sensitivity. First, $\mathcal H^{+}\subset\mathcal L^{+}$. Indeed, Proposition~\ref{prop:critical-surprisal-threshold} gives, at every position,
\[
\mathfrak{s}_{i,t}^{*}<H_{i,t}.
\]
Therefore, any token satisfying $\mathfrak{s}_{i,t}\ge H_{i,t}$ also satisfies
\[
\mathfrak{s}_{i,t}>\mathfrak{s}_{i,t}^{*}.
\]
Together with $\hat A_i>0$, this implies
\[
\mathcal H^{+}\subset\mathcal L^{+}.
\]
By the definition of $\Gamma$ and the inclusion above,
\[
\Gamma
=
\sum_{(i,t)\in\mathcal L^{+}}\hat A_i|\Phi_{i,t}|
\ge
\sum_{(i,t)\in\mathcal H^{+}}\hat A_i|\Phi_{i,t}|.
\]
On $\mathcal H^{+}$, the condition $\mathfrak{s}_{i,t}\ge H_{i,t}$ holds. Lemma~\ref{lem:app-phi-lower-high-surprisal} therefore gives
\[
|\Phi_{i,t}|\ge S_{2,i,t}.
\]
Hence
\[
\Gamma
\ge
\sum_{(i,t)\in\mathcal H^{+}}\hat A_iS_{2,i,t}.
\]
By Assumption~\ref{assump:high-surprisal-mass},
\[
\frac{1}{|\mathcal H^{+}|}
\sum_{(i,t)\in\mathcal H^{+}}\hat A_iS_{2,i,t}
\ge c_H,
\qquad
|\mathcal H^{+}|\ge \rho_H N.
\]
Therefore,
\[
\Gamma
\ge
c_H|\mathcal H^{+}|
\ge
c_H\rho_H N.
\]
Assumption~\ref{assump:linear-absolute-scale} gives
\[
\Sigma_{\mathrm{abs}}\le c_\Sigma N.
\]
Consequently,
\[
\frac{\Sigma_{\mathrm{abs}}}{\Gamma}
\le
\frac{c_\Sigma N}{c_H\rho_H N}
=
\frac{c_\Sigma}{c_H\rho_H}
=
O(1).
\]
We now decompose the critical offset as
\[
\frac{\Lambda}{\Gamma}
=
\frac{\Lambda}{\Sigma_{\mathrm{abs}}}
\cdot
\frac{\Sigma_{\mathrm{abs}}}{\Gamma}.
\]
Taking absolute values gives
\[
\left|\frac{\Lambda}{\Gamma}\right|
=
\frac{|\Lambda|}{\Sigma_{\mathrm{abs}}}
\cdot
\frac{\Sigma_{\mathrm{abs}}}{\Gamma}.
\]
The first factor is $O(T^{-1})$ by Lemma~\ref{lem:app-lambda-weak-residual}, and the second factor is $O(1)$ by the bound above. Hence
\[
\left|\frac{\Lambda}{\Gamma}\right|
=
O(T^{-1}).
\]
Equivalently,
\[
W^*-1=\frac{\Lambda}{\Gamma}=O(T^{-1}).
\]
\end{proof}

In the entropy-collapse regime studied in the main text, baseline GRPO corresponds to $\Lambda>0$. Hence $W^*>1$, and the critical weight exceeds one only by an $O(T^{-1})$ offset. The proof lower bounds $\Gamma$ using the stronger subset $\mathcal H^{+}$ rather than the full entropy-increasing positive-advantage set $\mathcal L^{+}$, so the argument is conservative. In actual training, additional contributions from $\mathcal L^{+}\setminus\mathcal H^{+}$ further enlarge $\Gamma$ and make the critical offset closer to zero.

\section{Formalization of the Cross-Step Entropy Dynamics in Section~\ref{sec:cross-step}}\label{app:cross-step}

This appendix casts the dynamical discussion in Section~\ref{sec:cross-step} as a discrete-time system under a mean-field closure. The purpose is to formalize the qualitative mechanism described in the main text. It is not intended as an unconditional global convergence theorem for general deep neural network training.

\subsection{Mean-Field Approximation}\label{app:mean-field}

Assume that, within the training interval of interest, the batch-level quantities $\Lambda_k$ and $\Gamma_k$ are primarily determined by the current batch-averaged policy entropy $\bar H_k$. More precisely, suppose that there exist continuous functions $\Lambda(h)$ and $\Gamma(h)$ such that
\[
\Lambda_k\approx \Lambda(\bar H_k),
\qquad
\Gamma_k\approx \Gamma(\bar H_k),
\qquad
\Gamma(h)>0.
\]
The associated critical-weight curve is defined as
\[
W^*(h)\triangleq 1+\frac{\Lambda(h)}{\Gamma(h)}.
\]
The statement in the main text that flatter policy distributions require a lower critical weight to sustain entropy growth is formalized through the following monotonicity condition:
\[
W^*(h)\ \text{is strictly decreasing in }h\text{ on the relevant interval.}
\]

\subsection{Sign of the One-Step Batch Entropy Change under a Fixed Weight}\label{app:cross-step-sign}

\begin{proposition}[Sign of the one-step batch entropy change under a fixed weight]\label{prop:app-cross-step-sign}
For a fixed token-level reweighting factor $W$, the one-step change in batch-averaged policy entropy satisfies
\[
\Delta\bar H_k
=
-\frac{1}{N}\Big[\Lambda(\bar H_k)-(W-1)\Gamma(\bar H_k)\Big]
=
\frac{\Gamma(\bar H_k)}{N}\Big[W-W^*(\bar H_k)\Big].
\]
Consequently,
\[
\operatorname{sign}(\Delta\bar H_k)
=
\operatorname{sign}\bigl(W-W^*(\bar H_k)\bigr).
\]
\end{proposition}

\begin{proof}
By the mean-field form of the batch-level entropy dynamics,
\[
\Delta\bar H_k
=
-\frac{1}{N}\Big[\Lambda(\bar H_k)-(W-1)\Gamma(\bar H_k)\Big].
\]
The bracketed term can be rewritten as
\[
\Lambda(\bar H_k)-(W-1)\Gamma(\bar H_k)
=
\Gamma(\bar H_k)\left[\frac{\Lambda(\bar H_k)}{\Gamma(\bar H_k)}-(W-1)\right].
\]
Since
\[
W^*(\bar H_k)=1+\frac{\Lambda(\bar H_k)}{\Gamma(\bar H_k)},
\]
we have
\[
\frac{\Lambda(\bar H_k)}{\Gamma(\bar H_k)}-(W-1)=W^*(\bar H_k)-W.
\]
Substituting this identity gives
\[
\Delta\bar H_k
=
-\frac{\Gamma(\bar H_k)}{N}\bigl[W^*(\bar H_k)-W\bigr]
=
\frac{\Gamma(\bar H_k)}{N}\bigl[W-W^*(\bar H_k)\bigr].
\]
Because $N>0$ and $\Gamma(\bar H_k)>0$, the sign of $\Delta\bar H_k$ is entirely determined by $W-W^*(\bar H_k)$. This proves the proposition.
\end{proof}

\subsection{Open-Loop Self-Reinforcing Entropy Collapse and Recovery}\label{app:cross-step-feedback}

\begin{corollary}[Open-loop self-reinforcing entropy collapse and recovery]\label{cor:app-cross-step-feedback}
Under the mean-field approximation above, the open-loop dynamics exhibit two symmetric feedback regimes. If $W=1$ and the current entropy level satisfies $W^*(\bar H_k)>1$, then Proposition~\ref{prop:app-cross-step-sign} gives $\Delta\bar H_k<0$. Since $W^*(h)$ is strictly decreasing in $h$, the next step satisfies
\[
\bar H_{k+1}<\bar H_k
\quad\Longrightarrow\quad
W^*(\bar H_{k+1})\ge W^*(\bar H_k).
\]
Thus, the unit-weight GRPO baseline moves farther below the critical-weight curve, which strengthens the entropy-decreasing pressure and reinforces entropy collapse.

Conversely, if at some step $W>W^*(\bar H_k)$, then Proposition~\ref{prop:app-cross-step-sign} gives $\Delta\bar H_k>0$. By the same monotonicity condition,
\[
\bar H_{k+1}>\bar H_k
\quad\Longrightarrow\quad
W^*(\bar H_{k+1})\le W^*(\bar H_k).
\]
Therefore, the margin by which the fixed weight exceeds the critical-weight curve increases after the entropy rises. The entropy-increasing pressure is consequently strengthened, which reinforces entropy recovery.
\end{corollary}

\begin{proof}
Both claims follow from Proposition~\ref{prop:app-cross-step-sign} and the monotonicity of $W^*(h)$. For the collapse regime, if $W=1$ and $W^*(\bar H_k)>1$, then $\Delta\bar H_k<0$, so $\bar H_{k+1}<\bar H_k$. Since the critical-weight curve decreases with entropy, evaluating it at the smaller entropy value $\bar H_{k+1}$ yields a critical weight that is no smaller than before. The unit-weight baseline is therefore even farther below the critical threshold at the next step. The recovery regime follows by the same argument with the inequalities reversed. This proves the corollary.
\end{proof}

\subsection{Local Stability of Batch-Level Target-Entropy Gating}\label{app:target-gating-stability}

\begin{proposition}[Local stability of batch-level target-entropy gating]\label{prop:app-target-gating-stability}
Consider the batch-level target-entropy gating rule introduced in Section~\ref{sec:target-entropy-gating},
\[
g_k=\mathbf 1[\bar H_k<H_{\mathrm{tgt}}].
\]
For the one-sided STARE variant, the one-step change in batch-averaged policy entropy is
\[
\Delta\bar H_k
=
-\frac{1}{N}\Big[\Lambda(\bar H_k)-g_k(W-1)\Gamma(\bar H_k)\Big].
\]
Define the gated and ungated update maps as
\[
F_{\mathrm{on}}(h)\triangleq -\frac{1}{N}\big[\Lambda(h)-(W-1)\Gamma(h)\big],
\]
\[
F_{\mathrm{off}}(h)\triangleq -\frac{1}{N}\Lambda(h).
\]
Assume that there exists a target-entropy neighborhood
\[
I=[H_{\mathrm{tgt}}-\delta,\ H_{\mathrm{tgt}}+\delta]
\]
such that the gated update points toward higher entropy below the target, while the ungated update points toward lower entropy at and above the target:
\[
F_{\mathrm{on}}(h)>0,
\qquad
h\in[H_{\mathrm{tgt}}-\delta,H_{\mathrm{tgt}}),
\]
\[
F_{\mathrm{off}}(h)<0,
\qquad
h\in[H_{\mathrm{tgt}},H_{\mathrm{tgt}}+\delta].
\]
Assume in addition that the one-step displacement inside this neighborhood is bounded by the neighborhood half-width:
\[
\Delta_{\max}
\triangleq
\max\Bigg\{
\sup_{h\in[H_{\mathrm{tgt}}-\delta,H_{\mathrm{tgt}})}F_{\mathrm{on}}(h),
\ 
\sup_{h\in[H_{\mathrm{tgt}},H_{\mathrm{tgt}}+\delta]}|F_{\mathrm{off}}(h)|
\Bigg\}
\le \delta.
\]
Then $I$ is forward invariant for the induced discrete-time entropy dynamics. Equivalently, once $\bar H_k$ enters $I$, all subsequent batch-averaged entropy values remain inside this neighborhood. The closed-loop system therefore exhibits bounded oscillation around the target entropy, rather than continuing to drift monotonically away from it.
\end{proposition}

\begin{proof}
Within $I$, the assumed sign conditions imply
\[
F_{\mathrm{on}}(h)>0\quad (h<H_{\mathrm{tgt}}),
\qquad
F_{\mathrm{off}}(h)<0\quad (h\ge H_{\mathrm{tgt}}).
\]
The bounded-step condition gives $\Delta_{\max}\le\delta$.

If $\bar H_k\in[H_{\mathrm{tgt}}-\delta,H_{\mathrm{tgt}})$, then $g_k=1$, and hence
\[
\bar H_{k+1}=\bar H_k+F_{\mathrm{on}}(\bar H_k).
\]
Since $F_{\mathrm{on}}(\bar H_k)>0$,
\[
\bar H_{k+1}>\bar H_k\ge H_{\mathrm{tgt}}-\delta.
\]
Moreover, $F_{\mathrm{on}}(\bar H_k)\le \Delta_{\max}\le \delta$, so
\[
\bar H_{k+1}
\le
\bar H_k+\delta
<
H_{\mathrm{tgt}}+\delta.
\]
Therefore,
\[
\bar H_{k+1}\in I.
\]

If $\bar H_k\in[H_{\mathrm{tgt}},H_{\mathrm{tgt}}+\delta]$, then $g_k=0$, and hence
\[
\bar H_{k+1}=\bar H_k+F_{\mathrm{off}}(\bar H_k).
\]
Since $F_{\mathrm{off}}(\bar H_k)<0$,
\[
\bar H_{k+1}<\bar H_k\le H_{\mathrm{tgt}}+\delta.
\]
At the same time, $|F_{\mathrm{off}}(\bar H_k)|\le \Delta_{\max}\le \delta$, which gives
\[
\bar H_{k+1}
\ge
\bar H_k-\delta
\ge
H_{\mathrm{tgt}}-\delta.
\]
Thus,
\[
\bar H_{k+1}\in I.
\]
In both cases, one step of the dynamics maps $I$ into itself. Hence $I$ is forward invariant, and any trajectory that enters the target-entropy neighborhood remains there thereafter. This proves the proposition.
\end{proof}

\section{Single-Polarity Operations and Finer-Grained Closed-Loop Extensions}\label{app:single-polarity}

\subsection{Definition of Surprisal-Quantile Proxy Sets}\label{app:quantile-proxy-sets}

In Section~\ref{sec:entropy-critical-partition} of the main text, we define
\[
\mathcal T^{+}=\{(i,t):\hat A_i>0\},
\qquad
\mathcal T^{-}=\{(i,t):\hat A_i<0\}.
\]
Within $\mathcal T^{+}$ and $\mathcal T^{-}$, token positions are sorted separately in descending order of surprisal,
\[
\mathfrak{s}_{i,t}=-\ln\pi_\theta(o_{i,t}\mid x_i,o_{i,<t}),
\]
and the top $P\%$ positions are selected to form
\[
\mathcal L_q^{+}\subset \mathcal T^{+},
\qquad
\mathcal L_q^{-}\subset \mathcal T^{-}.
\]
To cover all four single-polarity operations, we also define the corresponding low-surprisal proxy sets. Let $\mathcal U_q^{+}$ denote the bottom $P\%$ positions with the lowest surprisal in $\mathcal T^{+}$, and let $\mathcal U_q^{-}$ denote the bottom $P\%$ positions with the lowest surprisal in $\mathcal T^{-}$.

These four proxy sets approximate the four theoretical quadrants in the advantage-surprisal decomposition. The set $\mathcal L_q^{+}$ corresponds to positive-advantage high-surprisal positions and targets the entropy-increasing quadrant. The set $\mathcal U_q^{+}$ corresponds to positive-advantage low-surprisal positions and targets the entropy-decreasing quadrant. The set $\mathcal U_q^{-}$ corresponds to negative-advantage low-surprisal positions and targets the entropy-increasing quadrant. The set $\mathcal L_q^{-}$ corresponds to negative-advantage high-surprisal positions and targets the entropy-decreasing quadrant.

\subsection{Position-Level Logit Update under STARE Reweighting}\label{app:position-logit-update-stare}

\begin{proposition}[Position-level logit update under STARE reweighting]\label{prop:app-position-logit-update-stare}
Under the weighted clipped surrogate objective
\[
\mathcal J_{\mathrm{STARE}}(\theta)
=
\frac{1}{N}\sum_{i,t}\omega_{i,t}
\min\!\Big(
\rho_{i,t}(\theta)\hat A_i,\,
\operatorname{clip}(\rho_{i,t}(\theta),1-\epsilon,1+\epsilon)\hat A_i
\Big),
\]
if position $(i,t)$ lies in the unclipped regime, its effective logit update is
\[
\Delta z_v
=
\eta\,\omega_{i,t}\hat A_i(\delta_{va}-\pi_v),
\qquad v\in\mathcal V.
\]
Since $\omega_{i,t}>0$ always holds, STARE never reverses any token-level policy-gradient direction. It only selectively rescales the magnitude of the original GRPO learning signal.
\end{proposition}

\begin{proof}
The proof follows the same argument as Proposition~\ref{prop:app-logit-update}. The only change is that the local surrogate term is multiplied by the positive weight $\omega_{i,t}$. Therefore, the corresponding logit gradient is scaled by $\omega_{i,t}$:
\[
\frac{\partial}{\partial z_v}\bigl[\omega_{i,t}\hat A_i\log\pi_a\bigr]
=
\omega_{i,t}\hat A_i(\delta_{va}-\pi_v).
\]
Taking one infinitesimal gradient ascent step gives the stated update. Because $\omega_{i,t}>0$, the policy-gradient direction is preserved. This proves the proposition.
\end{proof}

\subsection{Exact Batch-Level Entropy Variation under Reweighting an Arbitrary Token Subset}\label{app:arbitrary-subset-reweighting}

\begin{proposition}[Exact batch-level entropy variation under reweighting an arbitrary token subset]\label{prop:app-arbitrary-subset-reweighting}
Let $\mathcal S$ be an arbitrary subset of token positions. Apply a uniform weight $r>0$ to positions in $\mathcal S$ and unit weight to all remaining positions:
\[
\omega_{i,t}
=
\begin{cases}
r, & (i,t)\in\mathcal S,\\
1, & \text{otherwise}.
\end{cases}
\]
Then, in the unclipped regime,
\[
\left.\frac{d\bar H}{d\eta}\right|_{\mathcal S,r}
=
-\frac{1}{N}\Big[\Lambda+(r-1)\Xi(\mathcal S)\Big],
\]
where
\[
\Lambda\triangleq \sum_{i,t}\hat A_i\Phi_{i,t},
\qquad
\Xi(\mathcal S)\triangleq \sum_{(i,t)\in\mathcal S}\hat A_i\Phi_{i,t}.
\]
\end{proposition}

\begin{proof}
By Proposition~\ref{prop:app-position-logit-update-stare} and Theorem~\ref{thm:token-entropy-variation}, the first-order entropy contribution at each position is linearly scaled by its token-level weight:
\[
-\omega_{i,t}\hat A_i\Phi_{i,t}.
\]
Thus,
\[
\left.\frac{d\bar H}{d\eta}\right|_{\omega}
=
-\frac{1}{N}\sum_{i,t}\omega_{i,t}\hat A_i\Phi_{i,t}.
\]
Substituting the single-subset reweighting form gives
\[
\sum_{i,t}\omega_{i,t}\hat A_i\Phi_{i,t}
=
\sum_{(i,t)\notin\mathcal S}\hat A_i\Phi_{i,t}
+
r\sum_{(i,t)\in\mathcal S}\hat A_i\Phi_{i,t}.
\]
Rearranging the expression into the unweighted batch sum plus the reweighting correction yields
\[
\sum_{i,t}\omega_{i,t}\hat A_i\Phi_{i,t}
=
\sum_{i,t}\hat A_i\Phi_{i,t}
+
(r-1)\sum_{(i,t)\in\mathcal S}\hat A_i\Phi_{i,t}
=
\Lambda+(r-1)\Xi(\mathcal S).
\]
Substitution completes the proof:
\[
\left.\frac{d\bar H}{d\eta}\right|_{\mathcal S,r}
=
-\frac{1}{N}\Big[\Lambda+(r-1)\Xi(\mathcal S)\Big].
\]
\end{proof}

\subsection{Four Single-Polarity Operations}\label{app:single-polarity-operations}

Operation O1 amplifies $\mathcal L_q^{+}$. Define
\[
\omega_{i,t}^{(\mathrm{O1})}
=
\begin{cases}
W_+, & (i,t)\in\mathcal L_q^{+},\\
1, & \text{otherwise},
\end{cases}
\qquad W_+>1.
\]
By Proposition~\ref{prop:app-arbitrary-subset-reweighting},
\[
\left.\frac{d\bar H}{d\eta}\right|_{\mathrm{O1}}
=
-\frac{1}{N}\Big[\Lambda+(W_+-1)\Xi(\mathcal L_q^{+})\Big].
\]
When $\mathcal L_q^{+}$ is dominated by positive-advantage high-surprisal entropy-increasing positions, we have $\Xi(\mathcal L_q^{+})<0$. Define
\[
\Gamma_{L^+}\triangleq -\Xi(\mathcal L_q^{+})
=
\sum_{(i,t)\in\mathcal L_q^{+}}|\hat A_i\Phi_{i,t}|.
\]
Then,
\[
\left.\frac{d\bar H}{d\eta}\right|_{\mathrm{O1}}
=
-\frac{1}{N}\Big[\Lambda-(W_+-1)\Gamma_{L^+}\Big].
\]

Operation O2 attenuates $\mathcal U_q^{+}$. Define
\[
\omega_{i,t}^{(\mathrm{O2})}
=
\begin{cases}
M_+, & (i,t)\in\mathcal U_q^{+},\\
1, & \text{otherwise},
\end{cases}
\qquad 0<M_+<1.
\]
By Proposition~\ref{prop:app-arbitrary-subset-reweighting},
\[
\left.\frac{d\bar H}{d\eta}\right|_{\mathrm{O2}}
=
-\frac{1}{N}\Big[\Lambda+(M_+-1)\Xi(\mathcal U_q^{+})\Big].
\]
When $\mathcal U_q^{+}$ is dominated by positive-advantage low-surprisal entropy-decreasing positions, we have $\Xi(\mathcal U_q^{+})>0$. Define
\[
\Gamma_{U^+}\triangleq \Xi(\mathcal U_q^{+})
=
\sum_{(i,t)\in\mathcal U_q^{+}}|\hat A_i\Phi_{i,t}|.
\]
Then,
\[
\left.\frac{d\bar H}{d\eta}\right|_{\mathrm{O2}}
=
-\frac{1}{N}\Big[\Lambda-(1-M_+)\Gamma_{U^+}\Big].
\]

Operation O3 amplifies $\mathcal U_q^{-}$. Define
\[
\omega_{i,t}^{(\mathrm{O3})}
=
\begin{cases}
W_-, & (i,t)\in\mathcal U_q^{-},\\
1, & \text{otherwise},
\end{cases}
\qquad W_->1.
\]
By Proposition~\ref{prop:app-arbitrary-subset-reweighting},
\[
\left.\frac{d\bar H}{d\eta}\right|_{\mathrm{O3}}
=
-\frac{1}{N}\Big[\Lambda+(W_--1)\Xi(\mathcal U_q^{-})\Big].
\]
When $\mathcal U_q^{-}$ is dominated by negative-advantage low-surprisal entropy-increasing positions, we have $\Xi(\mathcal U_q^{-})<0$. Define
\[
\Gamma_{U^-}\triangleq -\Xi(\mathcal U_q^{-})
=
\sum_{(i,t)\in\mathcal U_q^{-}}|\hat A_i\Phi_{i,t}|.
\]
Then,
\[
\left.\frac{d\bar H}{d\eta}\right|_{\mathrm{O3}}
=
-\frac{1}{N}\Big[\Lambda-(W_--1)\Gamma_{U^-}\Big].
\]

Operation O4 attenuates $\mathcal L_q^{-}$. Define
\[
\omega_{i,t}^{(\mathrm{O4})}
=
\begin{cases}
M_-, & (i,t)\in\mathcal L_q^{-},\\
1, & \text{otherwise},
\end{cases}
\qquad 0<M_-<1.
\]
By Proposition~\ref{prop:app-arbitrary-subset-reweighting},
\[
\left.\frac{d\bar H}{d\eta}\right|_{\mathrm{O4}}
=
-\frac{1}{N}\Big[\Lambda+(M_--1)\Xi(\mathcal L_q^{-})\Big].
\]
When $\mathcal L_q^{-}$ is dominated by negative-advantage high-surprisal entropy-decreasing positions, we have $\Xi(\mathcal L_q^{-})>0$. Define
\[
\Gamma_{L^-}\triangleq \Xi(\mathcal L_q^{-})
=
\sum_{(i,t)\in\mathcal L_q^{-}}|\hat A_i\Phi_{i,t}|.
\]
Then,
\[
\left.\frac{d\bar H}{d\eta}\right|_{\mathrm{O4}}
=
-\frac{1}{N}\Big[\Lambda-(1-M_-)\Gamma_{L^-}\Big].
\]

\subsection{Exact Entropy Variation under Unified Closed-Loop Gating}\label{app:closed-loop-gating-exact}

\begin{proposition}[Exact entropy variation under unified closed-loop gating]\label{prop:app-closed-loop-gating-exact}
Let $g_{i,t}\in\{0,1\}$ be an arbitrary binary gate. For the one-sided variant, define
\[
\omega_{i,t}^{(\mathrm{V1},g)}
=
1+g_{i,t}(W-1)\mathbf 1[(i,t)\in\mathcal L_q^{+}],
\qquad W>1.
\]
Then,
\[
\left.\frac{d\bar H}{d\eta}\right|_{\mathrm{V1},g}
=
-\frac{1}{N}
\left[
\Lambda
+
(W-1)\sum_{i,t}g_{i,t}\mathbf 1[(i,t)\in\mathcal L_q^{+}]\hat A_i\Phi_{i,t}
\right].
\]
If the activated proxy set is dominated by the corresponding entropy-increasing quadrant, the expression reduces to
\[
\left.\frac{d\bar H}{d\eta}\right|_{\mathrm{V1},g}
=
-\frac{1}{N}\bigl[\Lambda-(W-1)\Gamma_g^{+}\bigr],
\]
where
\[
\Gamma_g^{+}
\triangleq
\sum_{i,t}g_{i,t}\mathbf 1[(i,t)\in\mathcal L_q^{+}]|\hat A_i\Phi_{i,t}|.
\]

For the two-sided variant, define
\[
\resizebox{0.97\linewidth}{!}{$\displaystyle
\omega_{i,t}^{(\mathrm{V2},g)}
=
1+g_{i,t}(W-1)\mathbf 1[(i,t)\in\mathcal L_q^{+}]
-g_{i,t}(1-M)\mathbf 1[(i,t)\in\mathcal L_q^{-}],
\qquad W>1,\quad 0<M<1.
$}
\]
Then,
\[
\resizebox{0.97\linewidth}{!}{$\displaystyle
\left.\frac{d\bar H}{d\eta}\right|_{\mathrm{V2},g}
=
-\frac{1}{N}
\left[
\Lambda
+
(W-1)\sum_{i,t}g_{i,t}\mathbf 1[(i,t)\in\mathcal L_q^{+}]\hat A_i\Phi_{i,t}
+
(M-1)\sum_{i,t}g_{i,t}\mathbf 1[(i,t)\in\mathcal L_q^{-}]\hat A_i\Phi_{i,t}
\right].
$}
\]

If the two activated proxy sets are respectively dominated by the corresponding entropy-increasing and entropy-decreasing quadrants, the expression reduces to
\[
\left.\frac{d\bar H}{d\eta}\right|_{\mathrm{V2},g}
=
-\frac{1}{N}\bigl[\Lambda-(W-1)\Gamma_g^{+}-(1-M)\Gamma_g^{-}\bigr],
\]
where
\[
\Gamma_g^{-}
\triangleq
\sum_{i,t}g_{i,t}\mathbf 1[(i,t)\in\mathcal L_q^{-}]|\hat A_i\Phi_{i,t}|.
\]
\end{proposition}

\begin{proof}
The result follows by substituting the corresponding gated weights into the general weighted entropy identity derived in the proof of Proposition~\ref{prop:app-arbitrary-subset-reweighting}. In the one-sided case,
\[
\omega_{i,t}^{(\mathrm{V1},g)}-1
=
g_{i,t}(W-1)\mathbf 1[(i,t)\in\mathcal L_q^{+}],
\]
which gives
\[
\left.\frac{d\bar H}{d\eta}\right|_{\mathrm{V1},g}
=
-\frac{1}{N}
\left[
\Lambda
+
(W-1)\sum_{i,t}g_{i,t}\mathbf 1[(i,t)\in\mathcal L_q^{+}]\hat A_i\Phi_{i,t}
\right].
\]
The two-sided case follows by the same linearity argument. When the gated proxy sets preserve the dominant signs of their target quadrants, the corresponding signed sums can be rewritten as sums of absolute contributions. This proves the proposition.
\end{proof}

\subsection{Unified Instantiations of Three Closed-Loop Granularities}\label{app:closed-loop-granularities}

The gate $g_{i,t}$ in Proposition~\ref{prop:app-closed-loop-gating-exact} admits three closed-loop instantiations. For batch-level gating, define
\[
g_{i,t}^{(\mathrm{batch})}
=
\mathbf 1[\bar H_k<H_{\mathrm{tgt}}].
\]
For sample-level gating, define
\[
g_{i,t}^{(\mathrm{sample})}
=
\mathbf 1[\bar H_i<H_{\mathrm{tgt}}],
\qquad
\bar H_i\triangleq \frac{1}{T_i}\sum_{t=1}^{T_i}H_{i,t}.
\]
For token-level gating, define
\[
g_{i,t}^{(\mathrm{token})}
=
\mathbf 1[H_{i,t}<H_{\mathrm{tgt}}].
\]
Because $\omega_{i,t}>0$ holds in all three cases, every closed-loop variant preserves the direction of the original GRPO learning signal. The variants differ only in the granularity at which they modulate the local strength of the token-level update.

\section{Combined Reweighting Operations and Adaptive Weights}\label{app:combined-operations}

\subsection{Exact Entropy Variation for Arbitrary Two-Subset Reweighting}\label{app:two-subset-reweighting}

\begin{proposition}[Exact entropy variation for arbitrary two-subset reweighting]\label{prop:app-two-subset-reweighting}
Let $\mathcal S_1$ and $\mathcal S_2$ be two disjoint token subsets. We assign multiplicative token weights $r_1>0$ and $r_2>0$ to them, respectively, while all remaining token positions receive unit weight:
\[
\omega_{i,t}
=
\begin{cases}
r_1, & (i,t)\in\mathcal S_1,\\
r_2, & (i,t)\in\mathcal S_2,\\
1, & \text{otherwise}.
\end{cases}
\]
Then, in the unclipped regime of the weighted GRPO surrogate,
\[
\left.\frac{d\bar H}{d\eta}\right|_{\mathcal S_1,r_1;\mathcal S_2,r_2}
=
-\frac{1}{N}
\Big[
\Lambda+(r_1-1)\Xi(\mathcal S_1)+(r_2-1)\Xi(\mathcal S_2)
\Big].
\]
\end{proposition}

\begin{proof}
The proof follows directly from the same argument as Proposition~\ref{prop:app-arbitrary-subset-reweighting}. By the linearity of multiplicative token weighting,
\[
\sum_{i,t}\omega_{i,t}\hat A_i\Phi_{i,t}
=
\sum_{i,t}\hat A_i\Phi_{i,t}
+
(r_1-1)\sum_{(i,t)\in\mathcal S_1}\hat A_i\Phi_{i,t}
+
(r_2-1)\sum_{(i,t)\in\mathcal S_2}\hat A_i\Phi_{i,t}.
\]
Using the definitions of $\Lambda$ and $\Xi(\cdot)$, this becomes
\[
\sum_{i,t}\omega_{i,t}\hat A_i\Phi_{i,t}
=
\Lambda+(r_1-1)\Xi(\mathcal S_1)+(r_2-1)\Xi(\mathcal S_2).
\]
Substituting this identity into the first-order batch-averaged entropy derivative yields the result.
\end{proof}

\subsection{Four Representative Combined Reweighting Operations}\label{app:representative-combined-ops}

We next instantiate Proposition~\ref{prop:app-two-subset-reweighting} for four representative two-subset reweighting schemes. Each scheme combines two single-polarity interventions from the four-quadrant entropy decomposition.

\paragraph{(1) Combination C1: amplifying $\mathcal L_q^{+}$ and amplifying $\mathcal U_q^{-}$.}

C1 amplifies the two dominant entropy-increasing token categories on both advantage sides.

Define
\[
\omega_{i,t}^{(\mathrm{C1})}
=
\begin{cases}
W_+, & (i,t)\in\mathcal L_q^{+},\\
W_-, & (i,t)\in\mathcal U_q^{-},\\
1, & \text{otherwise},
\end{cases}
\qquad W_+>1,\quad W_->1.
\]
By Proposition~\ref{prop:app-two-subset-reweighting},
\[
\left.\frac{d\bar H}{d\eta}\right|_{\mathrm{C1}}
=
-\frac{1}{N}
\Big[
\Lambda+(W_+-1)\Xi(\mathcal L_q^{+})+(W_--1)\Xi(\mathcal U_q^{-})
\Big].
\]
When the two proxy sets capture the dominant mass of their corresponding entropy-increasing quadrants, we have
\[
\left.\frac{d\bar H}{d\eta}\right|_{\mathrm{C1}}
=
-\frac{1}{N}
\Big[
\Lambda-(W_+-1)\Gamma_{L^+}-(W_--1)\Gamma_{U^-}
\Big].
\]

\paragraph{(2) Combination C2: amplifying $\mathcal L_q^{+}$ and attenuating $\mathcal L_q^{-}$.}

C2 strengthens high-surprisal entropy-increasing tokens under positive advantages and weakens high-surprisal entropy-decreasing tokens under negative advantages. This is the combined operation used by Variant II in the main text.

Define
\[
\omega_{i,t}^{(\mathrm{C2})}
=
\begin{cases}
W, & (i,t)\in\mathcal L_q^{+},\\
M, & (i,t)\in\mathcal L_q^{-},\\
1, & \text{otherwise},
\end{cases}
\qquad W>1,\quad 0<M<1.
\]
By Proposition~\ref{prop:app-two-subset-reweighting},
\[
\left.\frac{d\bar H}{d\eta}\right|_{\mathrm{C2}}
=
-\frac{1}{N}
\Big[
\Lambda+(W-1)\Xi(\mathcal L_q^{+})+(M-1)\Xi(\mathcal L_q^{-})
\Big].
\]
When the two proxy sets capture the dominant mass of the corresponding entropy-increasing and entropy-decreasing quadrants, respectively, we obtain
\[
\left.\frac{d\bar H}{d\eta}\right|_{\mathrm{C2}}
=
-\frac{1}{N}
\Big[
\Lambda-(W-1)\Gamma_{L^+}-(1-M)\Gamma_{L^-}
\Big].
\]
This expression matches the Variant II batch entropy shift in the main text:
\[
\Lambda_{\mathrm{V2}}
=
\Lambda-(W-1)\Gamma^{+}-(1-M)\Gamma^{-},
\]
where
\[
\Gamma^{+}=\Gamma_{L^+},
\qquad
\Gamma^{-}=\Gamma_{L^-}.
\]

\paragraph{(3) Combination C3: attenuating $\mathcal U_q^{+}$ and amplifying $\mathcal U_q^{-}$.}

C3 regulates the low-surprisal region from both advantage sides. It weakens the entropy-decreasing signal from positive-advantage low-surprisal tokens and strengthens the entropy-increasing signal from negative-advantage low-surprisal tokens.

Define
\[
\omega_{i,t}^{(\mathrm{C3})}
=
\begin{cases}
M_+, & (i,t)\in\mathcal U_q^{+},\\
W_-, & (i,t)\in\mathcal U_q^{-},\\
1, & \text{otherwise},
\end{cases}
\qquad 0<M_+<1,\quad W_->1.
\]
By Proposition~\ref{prop:app-two-subset-reweighting},
\[
\left.\frac{d\bar H}{d\eta}\right|_{\mathrm{C3}}
=
-\frac{1}{N}
\Big[
\Lambda+(M_+-1)\Xi(\mathcal U_q^{+})+(W_--1)\Xi(\mathcal U_q^{-})
\Big].
\]
When the proxy sets align with their target quadrants and capture the dominant contributions, we have
\[
\left.\frac{d\bar H}{d\eta}\right|_{\mathrm{C3}}
=
-\frac{1}{N}
\Big[
\Lambda-(1-M_+)\Gamma_{U^+}-(W_--1)\Gamma_{U^-}
\Big].
\]

\paragraph{(4) Combination C4: attenuating $\mathcal U_q^{+}$ and attenuating $\mathcal L_q^{-}$.}

C4 attenuates the two dominant entropy-decreasing token categories.

Define
\[
\omega_{i,t}^{(\mathrm{C4})}
=
\begin{cases}
M_+, & (i,t)\in\mathcal U_q^{+},\\
M_-, & (i,t)\in\mathcal L_q^{-},\\
1, & \text{otherwise},
\end{cases}
\qquad 0<M_+<1,\quad 0<M_-<1.
\]
By Proposition~\ref{prop:app-two-subset-reweighting},
\[
\left.\frac{d\bar H}{d\eta}\right|_{\mathrm{C4}}
=
-\frac{1}{N}
\Big[
\Lambda+(M_+-1)\Xi(\mathcal U_q^{+})+(M_--1)\Xi(\mathcal L_q^{-})
\Big].
\]
When the two proxy sets capture the dominant mass of their corresponding entropy-decreasing quadrants, we obtain
\[
\left.\frac{d\bar H}{d\eta}\right|_{\mathrm{C4}}
=
-\frac{1}{N}
\Big[
\Lambda-(1-M_+)\Gamma_{U^+}-(1-M_-)\Gamma_{L^-}
\Big].
\]

\begin{table}[ht]
\centering
\caption{Summary of O1 to O4 and C1 to C4. The Target subset column specifies the entropy quadrant or pair of quadrants selected by each operation.}
\label{tab:o1-o4-c1-c4-summary}
\small
\setlength{\tabcolsep}{4pt}
\scalebox{0.9}{
\begin{tabularx}{\textwidth}{@{} c c c p{3.2cm} X @{}}
\toprule
\textbf{Notation} & \textbf{Type} & \textbf{Selected Subset} & \textbf{Weight Adjustment} & \textbf{Mechanistic  Interpretation} \\
\midrule
\addlinespace[2pt]
O1 & Single-polarity & $\mathcal{L}_q^{+}$ & Amplification, $W_+ > 1$ & Strengthens high-surprisal tokens in positive-advantage samples \\
\addlinespace[4pt]
O2 & Single-polarity & $\mathcal{U}_q^{+}$ & Attenuation, $0 < M_+ < 1$ & Weakens low-surprisal tokens in positive-advantage samples \\
\addlinespace[4pt]
O3 & Single-polarity & $\mathcal{U}_q^{-}$ & Amplification, $W_- > 1$ & Strengthens the suppression of low-surprisal tokens under negative advantages \\
\addlinespace[4pt]
O4 & Single-polarity & $\mathcal{L}_q^{-}$ & Attenuation, $0 < M_- < 1$ & Weakens the suppression of high-surprisal tokens under negative advantages \\
\midrule
\addlinespace[2pt]
C1 & Combined & $\mathcal{L}_q^{+} + \mathcal{U}_q^{-}$ & Joint amplification & Strengthens entropy-increasing token categories on both advantage sides \\
\addlinespace[4pt]
C2 & Combined & $\mathcal{L}_q^{+} + \mathcal{L}_q^{-}$ & Amplifies the former and attenuates the latter & Strengthens positive-advantage high-surprisal tokens while weakening negative-advantage high-surprisal entropy-decreasing tokens \\
\addlinespace[4pt]
C3 & Combined & $\mathcal{U}_q^{+} + \mathcal{U}_q^{-}$ & Attenuates the former and amplifies the latter & Jointly regulates the low-surprisal region on both advantage sides \\
\addlinespace[4pt]
C4 & Combined & $\mathcal{U}_q^{+} + \mathcal{L}_q^{-}$ & Joint attenuation & Suppresses entropy-decreasing token categories on both advantage sides \\
\addlinespace[2pt]
\bottomrule
\end{tabularx}
}
\end{table}

\subsection{Summary of Single-Polarity and Combined Operations}\label{app:operations-summary}

For clarity, we summarize the single-polarity operations O1 to O4 and the combined operations C1 to C4 below.

Single-polarity operations act on one entropy quadrant:
\begin{itemize}
    \item O1 amplifies $\mathcal L_q^{+}$, the positive-advantage high-surprisal set. It strengthens the entropy-increasing minority under positive advantages.
    \item O2 attenuates $\mathcal U_q^{+}$, the positive-advantage low-surprisal set. It weakens the dominant entropy-decreasing signal under positive advantages.
    \item O3 amplifies $\mathcal U_q^{-}$, the negative-advantage low-surprisal set. It strengthens the entropy-increasing effect induced by suppressing low-surprisal tokens.
    \item O4 attenuates $\mathcal L_q^{-}$, the negative-advantage high-surprisal set. It weakens the entropy-decreasing effect induced by suppressing high-surprisal tokens.
\end{itemize}

Under the four-quadrant analysis in Section~\ref{sec:four-quadrant}, O1 and O3 amplify entropy-increasing signals, whereas O2 and O4 attenuate entropy-decreasing signals. Variant I in the main text is a direct instance of O1, since it only amplifies token weights in $\mathcal L_q^{+}$.

Combined operations act on two entropy quadrants:
\begin{itemize}
    \item C1 jointly amplifies $\mathcal L_q^{+}$ and $\mathcal U_q^{-}$. It strengthens entropy-increasing signals on both the positive- and negative-advantage sides.
    \item C2 amplifies $\mathcal L_q^{+}$ and attenuates $\mathcal L_q^{-}$. It simultaneously strengthens positive-advantage high-surprisal entropy-increasing tokens and weakens negative-advantage high-surprisal entropy-decreasing tokens.
    \item C3 attenuates $\mathcal U_q^{+}$ and amplifies $\mathcal U_q^{-}$. It regulates the low-surprisal region by weakening positive-advantage entropy-decreasing tokens and strengthening negative-advantage entropy-increasing tokens.
    \item C4 jointly attenuates $\mathcal U_q^{+}$ and $\mathcal L_q^{-}$. It suppresses the two dominant entropy-decreasing token categories.
\end{itemize}

Variant II in the main text is a representative instance of C2. It combines amplification of $\mathcal L_q^{+}$ with attenuation of $\mathcal L_q^{-}$. Since all weights remain positive, the sign of each original GRPO policy-gradient update is preserved. The method only rescales token-level contribution magnitudes, thereby regulating batch-level entropy dynamics through two complementary mechanisms: amplifying entropy-increasing contributions and attenuating entropy-decreasing contributions.

Single-polarity operations implement targeted interventions on individual entropy quadrants, while combined operations perform joint regulation across two quadrants.

Here, $\mathcal L_q^{+}$ and $\mathcal L_q^{-}$ denote high-surprisal proxy subsets within positive- and negative-advantage tokens, respectively, while $\mathcal U_q^{+}$ and $\mathcal U_q^{-}$ denote low-surprisal proxy subsets within positive- and negative-advantage tokens. Amplification refers to assigning weights larger than one, whereas attenuation refers to assigning weights in $(0,1)$.

\section{Limitations and Broader Impacts}\label{app:limitations}

Our empirical evaluation may not exhaustively cover all possible task distributions and optimization settings. Moreover, similar to other LLMs, our trained model could also generate potentially unethical or misleading information sometimes. We hope our work contributes positively to the development of more reliable post-training paradigms for large language models.

\stopcontents[appendices]
\clearpage

\end{document}